\documentclass[11pt]{article}
\usepackage[margin=1in]{geometry}
\usepackage{amsmath, amssymb, graphicx, booktabs}
\usepackage{natbib}
\usepackage{multirow}
\usepackage{float}
\usepackage{subcaption}
\usepackage{pgffor}
\usepackage[T1]{fontenc} 
\usepackage{bbm}
\usepackage{amsthm}
\usepackage{newtxtext, newtxmath} 
\usepackage{indentfirst}
\usepackage{hyperref}
\usepackage{authblk}

\title{Orthogonal machine learning for conditional odds and risk ratios}
\author[1]{Jiacheng Ge\footnote{Correspondence to: jiacheng.ge@nyulangone.org}}
\author[1]{Iv\'an D\'iaz}

\affil[1]{Division of Biostatistics, Department of Population Health, 
New York University Grossman School of Medicine, New York, NY, USA}
\date{\today}
\theoremstyle{definition}
\newtheorem{theorem}{Theorem}
\newtheorem{proposition}{Proposition}

\begin{document}
\maketitle

\begin{abstract}
Conditional effects are commonly used measures for understanding how treatment effects vary across different groups, and are often used to target treatments/interventions to groups who benefit most. In this work we review existing methods and propose novel ones, focusing on the odds ratio (OR) and the risk ratio (RR). While estimation of the conditional average treatment effect (ATE) has been widely studied, estimators for the OR and RR lag behind, and cutting edge estimators such as those based on doubly robust transformations or orthogonal risk functions have not been generalized to these parameters. We propose such a generalization here, focusing on the DR-learner and the R-learner. We derive orthogonal risk functions for the OR and RR and show that the associated pseudo-outcomes satisfy second-order conditional-mean remainder properties analogous to the ATE case. We also evaluate estimators for the conditional ATE, OR, and RR in a comprehensive nonparametric Monte Carlo simulation study to compare them with common alternatives under hundreds of different data-generating distributions. Our numerical studies provide empirical guidance for choosing an estimator. For instance, they show that while parametric models are useful in very simple settings, the proposed nonparametric estimators significantly reduce bias and mean squared error in the more complex settings expected in the real world. We illustrate the methods in the analysis of physical activity and sleep trouble in U.S. adults using data from the National Health and Nutrition Examination Survey (NHANES) \citep{NHANES}. The results demonstrate that our estimators uncover substantial treatment effect heterogeneity that is obscured by traditional regression approaches and lead to improved treatment decision rules, highlighting the importance of data-adaptive methods for advancing precision health research.
\end{abstract}

\section{Introduction}
\label{sec:intro}
Conditional effect metrics such as the risk difference, the odds ratio, and risk ratio measure how treatment effects vary across levels of covariates, enabling better mechanistic understanding and more tailored interventions. Compared to more common effect heterogeneity measures such as risk difference, the conditional odds ratios and risk ratios are sometimes more desirable. For example the conditional odds ratio can be recovered from studies with selection bias, such as case-control studies, even in settings where the ATE cannot \citep{pmlr-v22-bareinboim12}, and it plays a fundamental role in identifiability of joint distributions under certain missing data models \citep{malinsky2022semiparametric}.

Traditional statistics education and practice prescribe estimation of conditional effect metrics using simple parametric models. For instance, logistic regression is often promoted as the go-to tool for estimating odds ratios due to its purported advantages \citep{hosmer2013applied,harrell2015rms, doi2022odds, mittinty2023reflection}. However, recent literature shows that reliance on inflexible parametric models can induce bias under misspecification and motivates nonparametric approaches that leverage modern machine learning \citep[e.g.,][]{rudolph2023comparing}. 
Many of these recent nonparametric approaches are based on orthogonal pseudo-outcomes, including the doubly robust unbiased transformations that appear in the ATE literature: transformations of the data that depend on two or more nuisance parameters and whose conditional mean equals the target parameter up to a remainder that is second order in the nuisance errors. These constructions are closely related to and are often derived from so-called \textit{efficient influence functions} \citep[EIFs,][]{bickel1993efficient}. They were first proposed as a functional estimation tool by \citet{RubinVanderLaan2007} in censored time-to-event settings to enable consistent nonparametric estimation of risk differences, and have been widely used in multiple statistical problems ranging from estimation of dose-response curves \citep{KennedyMaMcHughSmall2017}, estimation of causal effects in longitudinal settings \citep{luedtke2017sequential, diaz2023nonparametric}, and conditional effects in survival settings \citep{DiazSavenkovBallman2018}. For conditional effects, \citet{Kennedy2023EJS} reviews multiple estimators for the conditional average treatment effect, including the so-called DR-learner, which had been discussed in \citep{van2006statistical, luedtke2016super}, among others. The R-learner \citep{nie2021quasi} is another popular estimator for conditional average treatment effects based on such an orthogonal pseudo-outcome. Recently, \citet{morzywolek2025weighted} presented theoretical developments unifying these two estimators in a general class of weighted orthogonal estimators. The estimators we use are also closely related to the literature on orthogonal statistical learning, where one studies empirical risk minimization problems involving nuisance functions and exploits so-called ``Neyman orthogonality'' to ensure that first-stage nuisance estimation error enters the target risk only at second order \citep{chernozhukov2018double, foster2023orthogonal}. From this perspective, DR- and R-learners can be viewed as instances of a broader strategy that combines orthogonalization, cross-fitting, and flexible regression learners to obtain robust estimation of heterogeneous effects without requiring correct parametric specification \citep{hines2022demystifying, morzywolek2025weighted}.

Existing DR- and R-learners allow for estimation of heterogeneous treatment effects, but the literature has almost exclusively focused on additive-scale heterogeneity (e.g., average treatment effect). By contrast, ratio-based parameters such as conditional odds and risk ratios have received less attention. A notable recent exception is the work of \citet{vanderLaanCaroneLuedtke2024}, who introduced EP-learning, a framework that combines the stability of plug-in estimators with the efficiency of doubly robust approaches for a general class of causal contrasts which includes conditional risk ratios. EP-learning can be viewed as an advanced plug-in learner. It enforces oracle efficiency by updating preliminary outcome regressions on the link scale using sieve fluctuations targeted to obtain risk estimators that are non-parametrically optimal. This combination guarantees bounded predictions and oracle efficiency, but at the cost of the further tuning necessary (choosing a sieve basis and dimension) and additional computation. Meanwhile, conditional OR and RR estimators based directly on the canonical recipe to create suitable pseudo-outcomes and then regress these on covariates using any machine learning methods have not been developed. 

This gap motivates our two main contributions. First, we propose novel estimators for conditional odds and risk ratios based on efficient-influence-function-derived orthogonal pseudo-outcomes. In the DR-learner, the efficient influence function of the marginal ATE is used to define the pseudo-outcome used in estimation of the conditional ATE. The estimators we propose extend the same principle to multiplicative contrasts, replacing the ATE efficient influence function with efficient influence functions for the average conditional OR and RR. We also develop R-learner variants for these ratio-scale parameters, which differ from the DR-learner only in the weighting of the second-stage regression. This extends results by \citet{morzywolek2025weighted}, who showed that in the CATE setting, the DR-learner and R-learner share the same pseudo-outcome and differ only by a propensity-score-based coefficient in the loss function. Second, we conduct a comprehensive simulation study to compare the DR- and R-learners for the conditional ATE, OR, and RR with plug-in estimators (i.e., the T-Learner and S-Learner \citep{kunzel2019metalearners}) based on traditional logistic regression models and the SuperLearner \citet{vanderLaanPolleyHubbard2007}, under a wide range of realistic scenarios. Our simulation framework adopts the approach of \citet{rudolph2023comparing}, who compared parametric and nonparametric estimators for the average treatment effect using a universal Monte Carlo simulation approach that allows the comparison of estimators under hundreds of data-generating mechanisms generated at random, providing a more complete picture of estimator performance than standard simulation studies based on bespoke data-generating mechanisms. 

The remainder of this paper is organized as follows. Section \ref{sec:setup} defines the target parameters, measurements, and existing estimators. Section \ref{sec:estimators} proposes the new EIF estimators. Section \ref{sec:numstu} describes the data-generating process, details the Monte Carlo simulation study design. Section \ref{sec:results} presents the main results. Section \ref{sec:illustrative} provides a real data illustration. Section \ref{sec:discussion} provides a discussion of practical implications. Additional plots and detailed proofs are included in the Supplementary Materials (Section \ref{sec:supp_materials}).

\section{Preliminaries and existing estimators} \label{sec:setup}
In what follows we will assume access to an i.i.d. sample $Z_1,\ldots, Z_n$ from a random variable $Z=(X,T,Y)\sim P$, where $X$ is a vector of covariates, $T$ is a binary treatment, and $Y$ is a binary outcome of interest. We assume the distribution $P$ is in a nonparametric model. We focus on estimation of the conditional average treatment effect, defined as
\[ATE(X) = P(Y=1|T=1,X) - P(Y=1|T=0,X),\]
the conditional odds ratio, defined as 
    \[
    OR(X) = \frac{P(Y=1|T=1,X)/(1-P(Y=1|T=1,X))}{P(Y=1|T=0,X)/(1-P(Y=1|T=0,X))},
    \]
and the conditional risk ratio, defined as
    \[
    RR(X) = \frac{P(Y=1|T=1,X)}{P(Y=1|T=0,X)}.
    \]

Under standard assumptions in causal inference (namely, there are no unmeasured confounders and there is enough overlap) these quantities represent conditional causal effects. In what follows we assume the relevant assumptions hold, and refer to the conditional odds ratios, risk ratios, and treatment effects as "effects", but we emphasize that such language is only warranted in practice within the context of a rigorous causal model and well-justified identification assumptions. We refer the interested reader to standard textbooks on causality \cite[e.g.,][]{pearl2009causality} for more discussion of these assumptions.

In what follows, to simplify notation, we will use
\[
p_1(X) = P(Y=1 \mid T=1, X), \quad \text{ and }\quad p_0(X) = P(Y=1 \mid T=0, X).
\]
Straightforward estimators of $ATE(X)$, $OR(X)$, and $RR(X)$ may be constructed by estimating the probabilities $p_t(X)$ and plugging them into the definitions of $ATE(X)$, $OR(X)$ and $RR(X)$. These probabilities may be estimated stratified by $T$, or pooling data in $T$. In the case of the $ATE(X)$, these estimators have been referred to as the T- and S-Learners, where the T and S stand for \textit{two} and \textit{single}, respectively. 

\subsection{Plug-in estimators using the S- and T-Learners for treatment-specific probabilities}

The simplest option to estimate $ATE(X)$, $OR(X)$, or $RR(X)$ is to estimate $p_1(X)$ and $p_0(X)$ first and then plug them into the definition of the corresponding parameter. Estimators of that type are naturally referred to as \textit{plug-in} estimators. To implement plug-in estimators based on the S-Learner, the easiest and perhaps the most common in practice option in some applied fields (e.g., clinical research) is to impose the logistic regression parametrization without interactions, i.e., to assume that 
    \[
    \text{logit}\{p_t(X)\} = \beta_0 + \beta_T t + \sum_j \beta_j X_j.
    \]
It is straightforward to verify that under this model we have $OR(X) = \exp(\beta_T)$ is not a function of $X$, so that there is no treatment effect heterogeneity in the odds ratio scale, although the same is not true in other scales such as the risk ratio or the additive scale. This fact is sometimes incorrectly described as a virtue of the odds ratio \cite[e.g.,][]{doi2022odds}, but in reality it arises from the modeling assumption rather than from any intrinsic property of the odds ratio as a measure. We refer to this estimator as \texttt{LR} throughout, clarifying when necessary if it is used to estimate ATE, OR, or RR. This estimator is an example of an S-Learner, where the data are pooled and a single regression is used to estimate outcome expectations under treatment and control.

Alternatively, it is also sometimes common to add an interaction term between treatment and covariates. This allows treatment effects to vary by covariate levels, capturing heterogeneous effects through terms like $T X_j$:
    \[
     \text{logit}\{p_t(X)\} = \beta_0 + \beta_T t + \sum_j \beta_j X_j + \sum_j \gamma_j t X_j
    \]
This is equivalent to fitting two separate logistic regression models---one for the treated ($T=1$) and one for the control ($T=0$) group---and is therefore an example of a  T-learner. We refer to estimators based on the logistic regression model as \texttt{LR-T} throughout, clarifying whenever necessary if it refers to the ATE, OR, or RR.

The \texttt{SuperLearner} (\texttt{SL}) by \citet{vanderLaanPolleyHubbard2007} is an ensemble machine learning method that combines multiple base learners to minimize prediction error. 
\texttt{SuperLearner} uses V-fold cross-validation to evaluate each learner's performance, assigning weights to create a weighted combination of predictions. For each $X$, it estimates $\hat{p}_1(X)$ and $\hat{p}_0(X)$. As with logistic regression, the \texttt{SuperLearner} can be trained using data stratified by treatment $T$, or pooling treatment arms, leading to T- and S-learners. We will refer to such estimators as \texttt{SL-T} and \texttt{SL}, respectively.

\subsection{Plug-in estimators using the DR-learner for treatment-specific probabilities}

Plug-in estimators for $ATE(X)$, $OR(X)$, or $RR(X)$ may be constructed by first estimating the conditional probabilities $p_t(X)$ using the DR-learner \citep{luedtke2016super, Kennedy2023EJS}, and then plugging them into the definition of the effect. We will refer to this estimator as \texttt{DR-P} below.

The \texttt{DR-P} approach estimates the conditional probabilities in two stages before calculating the contrasts (ATE, OR, or RR). In the first stage, we use \texttt{SuperLearner} to estimate the nuisance functions---outcome probabilities $\hat{p}_1(X) = \hat{P}(Y=1|T=1,X)$, $\hat{p}_0(X) = \hat{P}(Y=1|T=0,X)$, and propensity score $\hat{e}(X) = \hat{P}(T=1|X)$. In our implementation, we use two-fold cross-fitting; see Section~\ref{sec:mc_details} for details.

In the second stage, these cross-fitted nuisance estimates are used to construct pseudo-outcomes for $P(Y=1|T=1,X)$ and $P(Y=1|T=0,X)$ as:
\[
\hat\varphi_1(Z) = \frac{I(T=1)}{\hat{e}(X)} (Y - \hat{p}_1(X)) + \hat{p}_1(X),
\]
\[
\hat\varphi_0(Z) = \frac{I(T=0)}{1-\hat{e}(X)} (Y - \hat{p}_0(X)) + \hat{p}_0(X).
\]
These pseudo-outcomes are then used as the response variables in \texttt{SuperLearner} regression models with covariates $X$, and the resulting fold-specific predictions are averaged to obtain the refined probabilities $\hat{p}_1^*(X)$ and $\hat{p}_0^*(X)$ on the target distribution. The final conditional ATE, OR, and RR are still obtained by plugging refined probability estimates into the usual effect formulas. This construction is attractive because the arm-specific pseudo-outcomes above are efficient influence-function-based scores for the treatment-specific outcome regressions, and related DR-learner theory suggests that cross-fitted two-stage procedures based on such orthogonal scores can enjoy reduced sensitivity to first-stage nuisance estimation error together with oracle-type error bounds \citet{Kennedy2023EJS}.

\section{Novel estimators}
\label{sec:estimators}

For the case of the conditional ATE, it has been argued that separately estimating $p_1(X)$ and $p_0(X)$ can lead to suboptimal estimators, particularly if the underlying conditional ATE is smooth but the outcome probabilities are not \citep{Kennedy2023EJS}. This motivates the DR-learner, which regresses $\hat\varphi_1(Z)-\hat\varphi_0(Z)$ on $X$ rather than fitting S- or T-learners or regressing the pseudo-outcomes separately. The optimality of this approach (which will be referred to as \texttt{DR-CATE}) has been widely discussed, and oracle inequalities showing double robustness have been proved \citep{Kennedy2023EJS, DiazSavenkovBallman2018}. In this section, we develop DR-learner estimators for the conditional OR and RR. The motivation is the same as for the conditional ATE: if the OR or RR is smooth in the covariates, but the treatment-specific outcome probabilities $p_t(X)$ are not, then the plug-in estimator \texttt{DR-P} from the previous section can have inflated variance.

Our approach can be understood from two complementary angles. First, we replace unstable plug-in estimators of odds ratios or risk ratios with orthogonal pseudo-outcomes derived from the efficient influence function of related pathwise differentiable parameters \citep{bickel1993efficient,hines2022demystifying,Kennedy2023EJS,jun2023average}. In particular, we use the uncentered EIFs for $E[OR(X)]$ and $E[\log OR(X)]$ as pseudo-outcomes to construct estimators of the conditional odds ratio, and the uncentered EIFs for $E[RR(X)]$ and $E[\log RR(X)]$ as pseudo-outcomes to construct estimators of the conditional risk ratio. These EIF-type pseudo-outcomes satisfy conditional-mean identities with second-order remainder terms in the nuisance errors. Second, these pseudo-outcomes fit naturally into the orthogonal statistical learning framework of \citet{foster2023orthogonal}: after defining the corresponding squared losses, we verify the orthogonality condition in Assumption~1 of \citet{foster2023orthogonal}, the first-order optimality condition in Assumption~2 of \citet{foster2023orthogonal}, and the smoothness and strong-convexity conditions in Assumptions~3--4 of \citet{foster2023orthogonal}, and then invoke the plug-in empirical risk minimization (ERM) bound in Theorem~3 of \citet{foster2023orthogonal}.

\subsection{DR-learner for the conditional OR and RR}
We now give the formal definitions of the four DR-learners. For the OR case, our derivations are closely related to work by \citet{jun2023average}, who derived the EIF of $E[\log OR(X)]$ with the goal of estimating it directly as a scalar summary of marginal treatment effectiveness.

To keep the notation consistent throughout Theorems~\ref{thm:dr_transforms} and~\ref{thm:or_orthogonal}, let $\eta=(q_1,q_0,\pi)$ denote a generic nuisance triple, let $\eta_0=(p_1,p_0,e)$ denote the true nuisance tuple, and let $\hat\eta=(\hat p_1,\hat p_0,\hat e)$ denote the first-stage estimate. Define
\[
OR_\eta(X)=\frac{q_1(X)\{1-q_0(X)\}}{q_0(X)\{1-q_1(X)\}},
\qquad
\log OR_\eta(X)=\log\{OR_\eta(X)\},
\]
\[
RR_\eta(X)=\frac{q_1(X)}{q_0(X)},
\qquad
\log RR_\eta(X)=\log\{RR_\eta(X)\}.
\]
The pseudo-outcomes corresponding to each estimator are as follows:
\begin{itemize}
  \item For estimating $OR(X)$, define
    \[
    \varphi_{OR}(Z;\eta)=OR_\eta(X)
    +\frac{OR_\eta(X)}{q_1(X)\{1-q_1(X)\}}\frac{I(T=1)}{\pi(X)}\{Y-q_1(X)\}
    -\frac{OR_\eta(X)}{q_0(X)\{1-q_0(X)\}}\frac{I(T=0)}{1-\pi(X)}\{Y-q_0(X)\}.
    \]
    The resulting estimator is referred to as \texttt{DR-OR}.
  \item For estimating $\log(OR(X))$, define
    \[
    \varphi_{\log OR}(Z;\eta)=\log OR_\eta(X)
    +\frac{I(T=1)}{\pi(X)}\frac{Y-q_1(X)}{q_1(X)\{1-q_1(X)\}}
    -\frac{I(T=0)}{1-\pi(X)}\frac{Y-q_0(X)}{q_0(X)\{1-q_0(X)\}}.
    \]
    The resulting estimator is referred to as \texttt{DR-LOR}.
  \item For estimating $RR(X)$, define
    \[
    \varphi_{RR}(Z;\eta)=RR_\eta(X)
    +\frac{1}{q_0(X)}\frac{I(T=1)}{\pi(X)}\{Y-q_1(X)\}
    -\frac{q_1(X)}{q_0(X)^2}\frac{I(T=0)}{1-\pi(X)}\{Y-q_0(X)\}.
    \]
    The resulting estimator is referred to as \texttt{DR-RR}.
  \item Lastly, for estimating $\log(RR(X))$, define
    \[
    \varphi_{\log RR}(Z;\eta)=\log RR_\eta(X)
    +\frac{1}{q_1(X)}\frac{I(T=1)}{\pi(X)}\{Y-q_1(X)\}
    -\frac{1}{q_0(X)}\frac{I(T=0)}{1-\pi(X)}\{Y-q_0(X)\}.
    \]
    The resulting estimator is referred to as \texttt{DR-LRR}.
\end{itemize}

Theorem~\ref{thm:dr_transforms} and Theorem~\ref{thm:or_orthogonal} will use three related versions of the same pseudo-outcome:
\[
\varphi_\theta(Z):=\varphi_\theta(Z;\eta_0),
\qquad
\hat\varphi_\theta(Z):=\varphi_\theta(Z;\hat\eta),
\qquad
\theta\in\{OR,\log OR,RR,\log RR\}.
\]
Here $\varphi_\theta(Z)$ is the population pseudo-outcome evaluated at the true nuisance $\eta_0$, while $\hat\varphi_\theta(Z)$ is the computed pseudo-outcome evaluated at the first-stage estimate $\hat\eta$.
For each target parameter, the second stage fits
\[
\hat\theta\in\arg\min_{f\in\mathcal{F}}\frac{1}{n}\sum_{i=1}^{n}\left(\hat\varphi_\theta(Z_i)-f(X_i)\right)^2.
\]
For theoretical analysis we also define the corresponding population risk
\[
\mathcal{L}_{\theta}(f,\eta)=E\!\left[\left\{\varphi_{\theta}(Z;\eta)-f(X)\right\}^2\right].
\]

We now establish that the pseudo-outcomes defined above satisfy a second-order remainder property: their conditional expectation equals the target parameter up to a remainder term that is second order in the nuisance estimation errors. This property underlies both the quality of the resulting conditional effect estimator and the Neyman orthogonality of the associated squared losses.

\begin{theorem}[Conditional-mean identities with second-order remainders]
\label{thm:dr_transforms}
Let $\mathcal{G}$ be a class of nuisance triples. Assume the true nuisance tuple $\eta_0=(p_1,p_0,e)$ and every $\eta=(q_1,q_0,\pi)\in\mathcal{G}$ satisfy
\[
0<e(X)<1,
\quad
0<p_t(X)<1,
\quad
0<\pi(X)<1,
\quad
0<q_t(X)<1
\]
for $t\in\{0,1\}$. For each $\theta\in\{OR,\log OR,RR,\log RR\}$ and each $\eta\in\mathcal{G}$, let $\varphi_\theta(Z;\eta)$ be defined as above. Then
\begin{align*}
E[\varphi_{OR}(Z;\eta)\mid X] &= OR(X) + R_{OR}(X;\eta),\\
E[\varphi_{\log OR}(Z;\eta)\mid X] &= \log OR(X) + R_{\log OR}(X;\eta),\\
E[\varphi_{RR}(Z;\eta)\mid X] &= RR(X) + R_{RR}(X;\eta),\\
E[\varphi_{\log RR}(Z;\eta)\mid X] &= \log RR(X) + R_{\log RR}(X;\eta),
\end{align*}
where each $R_\theta(X;\eta)$ is a second-order remainder. Each remainder consists of (a) cross terms that are products of propensity-score errors $(e-\pi)$ and outcome-regression errors $(p_t-q_t)$, and (b) a second-order Taylor remainder term that is quadratic in the outcome-regression errors, given explicitly in the appendix. In particular, $R_\theta(X;\eta)=0$ whenever $q_t(X)=p_t(X)$ for $t\in\{0,1\}$. For the nonlinear parameters considered here, this second-order Taylor remainder term generally remains even when the propensity score is correct but the outcome regressions are not.
\end{theorem}

In particular, after specializing the generic identity in Theorem~\ref{thm:dr_transforms} to $\eta=\eta_0$ and $\eta=\hat\eta$, we obtain the exact population identity
\[
E[\varphi_\theta(Z)\mid X]=\theta(X)
\]
for $\varphi_\theta(Z)=\varphi_\theta(Z;\eta_0)$, and the plug-in identity
\[
E[\hat\varphi_\theta(Z)\mid X]=\theta(X)+R_\theta(X;\hat\eta)
\]
for $\hat\varphi_\theta(Z)=\varphi_\theta(Z;\hat\eta)$. To rigorously justify that the second-stage empirical risk minimizer is a good estimator of $\theta(X)$, we next show that the population loss $\mathcal{L}_\theta(f,\eta)$ is a Neyman-orthogonal risk in the sense of \citet{foster2023orthogonal}. Proofs of Theorem~\ref{thm:dr_transforms}, including the OR case and the analogous arguments for $\log OR$, $RR$, and $\log RR$, are deferred to Supplementary Materials Section~\ref{sec:supp_thm1_proofs}.

We now use the explicit remainder structure from Theorem~\ref{thm:dr_transforms} to establish the orthogonal learning guarantee. The following result applies to all four DR-learners. Below, $D_f$ and $D_\eta$ denote Gateaux derivatives of $\mathcal{L}_\theta(f,\eta)$ with respect to the target function $f$ and nuisance argument $\eta$, respectively, and $h$ and $k$ denote arbitrary directions in the corresponding spaces.

\begin{theorem}[Orthogonal loss and oracle excess risk bound]
\label{thm:or_orthogonal}
For each $\theta\in\{OR,\log OR,RR,\log RR\}$, define the population risk
\[
\mathcal{L}_{\theta}(f,\eta)=E\!\left[\left\{\varphi_{\theta}(Z;\eta)-f(X)\right\}^2\right].
\]
Assume the true nuisance tuple $\eta_0=(p_1,p_0,e)$ and the nuisance class $\mathcal{G}$ satisfy the boundary condition that there exists $c_0\in(0,1/2)$ with
\[
c_0\le e(X)\le 1-c_0,
\quad
c_0\le p_t(X)\le 1-c_0,
\quad
c_0\le \pi(X)\le 1-c_0
\quad\text{and}\quad
c_0\le q_t(X)\le 1-c_0
\]
for all $\eta=(q_1,q_0,\pi)\in\mathcal{G}$ and $t\in\{0,1\}$. Let $\mathcal{F}$ be a bounded class of candidate functions containing $\theta(\cdot)$ with $M:=\max\!\left\{\sup_{f\in\mathcal{F}}\|f\|_\infty,1\right\}$. Then:
\begin{enumerate}
\item[\textbf{(i)}] \textbf{(Population minimizer.)} For every nuisance value $\eta$, the minimizer of $\mathcal{L}_{\theta}(f,\eta)$ over all measurable functions $f$ is $m_{\theta,\eta}(X)=E[\varphi_{\theta}(Z;\eta)\mid X]$. In particular, since $R_\theta(X;\eta_0)=0$ by Theorem~\ref{thm:dr_transforms},
\[
\theta(X) = \arg\min_{f\;\text{measurable}}\, \mathcal{L}_{\theta}(f,\eta_0),
\]
and the oracle excess risk satisfies $\mathcal{L}_{\theta}(f,\eta_0)-\mathcal{L}_{\theta}(\theta,\eta_0)=E\!\left[(f(X)-\theta(X))^2\right]$.

\item[\textbf{(ii)}] \textbf{(Neyman orthogonality.)} The population risk $\mathcal{L}_{\theta}(f,\eta)$ satisfies the Neyman orthogonality condition in Assumption~1 of \citet{foster2023orthogonal}:
\[
D_\eta\, D_f\, \mathcal{L}_{\theta}(\theta,\eta_0)[h,\,k] = 0 \qquad \forall\;\text{target direction}\;h,\;\forall\;\text{nuisance direction}\;k.
\]
This follows from the fact that $R_\theta(X;\eta)$ is second-order in the nuisance errors for each $\theta$ (Theorem~\ref{thm:dr_transforms}), so its pathwise derivative at $\eta_0$ vanishes.

\item[\textbf{(iii)}] \textbf{(Oracle excess risk bound.)} Let $\hat\eta=(\hat p_1,\hat p_0,\hat e)$ be the first-stage nuisance estimate and define the second-stage estimator
\[
\hat\theta\in\arg\min_{f\in\mathcal{F}}\frac{1}{n}\sum_{i=1}^n\left\{\hat\varphi_{\theta}(Z_i)-f(X_i)\right\}^2.
\]
Then Theorem~3 of \citet{foster2023orthogonal} implies that, with probability at least $1-\delta$,
\[
E\!\left[\left\{\hat\theta(X)-\theta(X)\right\}^2\right]
\le
\underbrace{C_1\!\left(\frac{\delta_n^2}{M^2}+\frac{\log(\delta^{-1})}{n}\right)}_{\text{oracle term}}
+\underbrace{C_2\,\|\hat\eta-\eta_0\|_\mathcal{G}^4}_{\text{nuisance term}},
\]
where $\delta_n$ is the critical radius appearing in Theorem~3 of \citet{foster2023orthogonal}, $\|\cdot\|_\mathcal{G}:=\|\cdot\|_{L_4(\ell_2)}$ is the nuisance pre-norm, and $C_1,C_2>0$ depend on the boundary constant $c_0$, the function-class bound $M$, and the parameter $\theta$.
\end{enumerate}
\end{theorem}

The proof is deferred to Supplementary Materials Section~\ref{sec:supp_thm2_proof}.

\medskip
In words, Theorem~\ref{thm:or_orthogonal} shows that the squared loss constructed from each orthogonal pseudo-outcome $\varphi_\theta$ is a Neyman-orthogonal loss in the sense of \citet{foster2023orthogonal}. The final excess risk of the two-stage learner therefore decomposes into two pieces: (a)~the oracle term, which corresponds to the second-stage rate one would obtain if the true nuisance $\eta_0$ were known; and (b)~the nuisance term, which captures first-stage nuisance estimation error through $\|\hat\eta-\eta_0\|_\mathcal{G}^4$. This fourth-order dependence is the $r=0$ specialization of the orthogonal-loss bounds in \citet{foster2023orthogonal} and, in our plug-in empirical risk minimization setting, is instantiated through Theorem~3 of \citet{foster2023orthogonal}. Thus the two-stage learner inherits the second-stage oracle rate whenever the nuisance term is of smaller order than the oracle term.

\subsection{R-learner variants}
In addition to the DR-learner approach described above, we also propose R-learner variants for all four target parameters. The R-learner differs from the DR-learner only through the second-stage loss, where observations are reweighted by the propensity-overlap weights \citep{nie2021quasi,morzywolek2025weighted}
\[
w(X)=\hat e(X)\{1-\hat e(X)\}.
\]

We implement four R-learner estimators:
\begin{itemize}
  \item \texttt{R-OR}: R-learner for $OR(X)$
  \item \texttt{R-LOR}: R-learner for $\log(OR(X))$
  \item \texttt{R-RR}: R-learner for $RR(X)$
  \item \texttt{R-LRR}: R-learner for $\log(RR(X))$
\end{itemize}

Because the R-learner uses the same pseudo-outcomes $\hat\varphi_\theta(Z)$ as the DR-learner, the conditional-mean identities in Theorem~\ref{thm:dr_transforms} carry over unchanged. The distinction arises only in the second-stage empirical risk: Theorem~\ref{thm:or_orthogonal} studies the unweighted squared loss used by the DR-learner, whereas the R-learner corresponds to a propensity-overlap-weighted loss. In this sense, the R-learner may be viewed as a weighted DR-learner, consistent with the weighted orthogonal learner framework of \citet{morzywolek2025weighted}. The corresponding R-learner extension is analogous, with the unweighted squared loss in Theorem~\ref{thm:or_orthogonal} replaced by a propensity-overlap-weighted loss.

\section{Numerical Studies} \label{sec:numstu}
We assess the proposed estimators through a comprehensive Monte Carlo simulation study designed to evaluate performance across a broad range of data-generating mechanisms rather than under a single bespoke setting. The rationale for this approach is that any particular DGP may favor certain estimators over others, and that a meaningful comparison should span a representative portion of the nonparametric model. This allows us to make statements about typical estimator performance rather than best-case or worst-case behavior under a hand-picked scenario.

\subsection{Data-Generating Process} 
We adopt a data-generating process (DGP) by \citet{rudolph2023comparing}. This DGP generates synthetic datasets that mimic the complexity characterized by nonlinear relationships, covariate interactions, and treatment effect heterogeneity. It employs a Bayesian-inspired approach, sampling from a nonparametric model using minimally informative priors to reflect limited prior knowledge of the data-generating mechanism. It generates the joint distribution $P$ sequentially: first sampling the covariate distribution $P(X)$, then the treatment mechanism $P(T=1|X)$, and finally the outcome mechanism $P(Y=1|T,X)$. The joint distribution is given by $P =  P(X)\cdot P(T|X) \cdot P(Y|T,X)$. Because the algorithms used to generate the data-generating processes are computationally demanding, we focus on a simple setting with a few covariates, which is sufficient to illustrate the properties of the estimators.

The DGP generates data for a binary treatment $T$, binary outcome $Y$, five binary covariates ($X_{\text{bin}}$), and one numerical covariate ($X_{\text{num}}$) with cardinality 100. It is controlled by the following hyperparameters:

\begin{itemize}
  \item \texttt{n\_bin = 5}: Number of binary covariates, capturing discrete confounding factors.
  \item \texttt{n\_num = 1}: Number of numerical covariates, allowing continuous confounding effects.
  \item \texttt{inter\_order} $\in \{1, 2, 3\}$: Order of interactions among binary covariates, controlling complexity of covariate relationships.
  \item \texttt{tx\_inter} $\in \{TRUE, FALSE\}$: Controls whether treatment-covariate interactions are included in the outcome mechanism. When \texttt{TRUE}, the coefficient of the treatment depends on the covariates, allowing the conditional effects to vary across covariate levels.
  \item \texttt{conf\_bias}: Confounding bias, sampled from $U(-0.5, 0.5)$, to simulate realistic confounding scenarios.
  \item \texttt{eta}: Controls nonlinearity, sampled from $U(0.1, 30)$, allowing flexible functional forms.
  \item \texttt{rho}: Controls smoothness, sampled from $U(0.1, 30)$, ensuring realistic covariate effects.
  \item \texttt{pos\_bound = 1000}: Ensures treatment probabilities are $\geq 0.001$, preventing positivity violations.
  \item \texttt{npoints = 100}: Number of points for the numerical covariate, providing sufficient granularity.
  \item \texttt{tol = 0.01}: Tolerance for confounding bias constraints, ensuring feasible distributions.
\end{itemize}

The DGP operates as follows:

\begin{enumerate}
  \item \textbf{Covariate Distribution}: The joint distribution $P(X)$ is sampled using a Dirichlet distribution over the covariate space ($\{0,1\}^5 \times [0,1]$). This approach, common for categorical data \citep{dunson2009nonparametric}, generates diverse covariate patterns, mimicking real-world heterogeneity and correlations.
  \item \textbf{Treatment Mechanism}: The propensity score $P(T=1|X)$ is modeled as:
    \[
    P(T=1|X) = \sum_{l \in L} \left\{ \alpha_{0,l} + \alpha_{1,l} f_l(X_{\text{num}}) \right\} \prod_{j=1}^5 X_{\text{bin},j}^{l_j},
    \]
    where $L$ indexes interactions up to order \texttt{inter\_order}, and $f_l(X_{\text{num}})$ are Gaussian processes with covariance $K(x_i, x_j) = \eta \exp(-\rho \lVert x_i - x_j \rVert^2)$. Coefficients $\alpha_{0,l}, \alpha_{1,l}$ are sampled uniformly from a polytope defined by constraints ensuring $0 < P(T=1|X) < 1$ and positivity ($\max_{t,x} \frac{P(T=t)}{P(T=t|X=x)} \leq 1000$). This structure allows nonlinear and interactive treatment assignment, critical for testing estimator robustness to complex confounding.
  \item \textbf{Outcome Mechanism}: The outcome probability $P(Y=1|T,X)$ is modeled as:
\[
P(Y=1|T,X) = T \sum_{l \in L} \big\{ \lambda_{0,l} + \lambda_{1,l} \, h_l(X_{\text{num}}) \big\} \tilde{X}_l
+ \sum_{l \in L} \big\{ \beta_{0,l} + \beta_{1,l} \, w_l(X_{\text{num}}) \big\} \tilde{X}_l,
\]
when the treatment-covariate interactions are present (\texttt{tx\_inter} = TRUE), or without the $T$-dependent interaction term otherwise. Here,
\[
\tilde{X}_l = \prod_{j=1}^{5} X_{\text{bin},j}^{l_j}
\quad \text{is the interaction term of binary covariates.}
\]
Functions $h_l, w_l$ are Gaussian processes, and coefficients are sampled to satisfy confounding bias constraints. Outcome probabilities are constrained to $[0.05, 0.95]$ to prevent extreme values that could destabilize OR and RR estimates.
  \item \textbf{Confounding Bias}: The DGP ensures the confounding bias $C$, defined as:
    \[
    C = \sum_t (2t-1) \sum_x \frac{P(X=x)}{P(T=t)} \{P(T=t|X=x) - P(T=t)\} P(Y=1|T=t,X=x),
    \]
    lies within feasible bounds ($C_{\text{low}} \leq \text{conf\_bias} \leq C_{\text{high}}$). If not, the draw is rejected, ensuring realistic confounding scenarios.
  \item \textbf{Output}: Each distribution is saved as a CSV file containing covariates, treatment, outcome, and probabilities ($P(X)$, $P(T|X)$, $P(Y|T,X)$, $OR(X)$, $RR(X)$).
\end{enumerate}

We generate 100 distributions for each combination of \texttt{inter\_order} $\in \{1, 2, 3\}$ and \texttt{tx\_inter} $\in \{\text{True}, \text{False}\}$, yielding 600 distributions. This large number of distributions ensures robust evaluation across diverse DGPs, capturing variations in complexity critical for comparing estimators of conditional ATE, OR, and RR. Some covariate distributions may result in no feasible solution for the desired confounding bias, leading to frequent failures of the DGP. Specifically, if the confounding bias constraints cannot be satisfied after 1,000 iterations, the draw is rejected, indicating infeasible initial conditions.

For each target parameter separately, we compute the CV of the true conditional effect within each generated DGP. Since each DGP is generated under a specific $(\texttt{inter\_order}, \texttt{tx\_inter})$ combination, the HTE label is assigned separately at the DGP level. Within each target parameter, DGPs with CV $\geq 0.2$ are labeled HTE High, and those with CV $< 0.2$ are labeled HTE Low.

\subsection{Monte Carlo Simulation Details}
\label{sec:mc_details}
We conduct large-scale Monte Carlo simulations to compare estimators, following the framework of \citet{rudolph2023comparing}. For each of the 600 generated distributions:
\begin{itemize}
  \item We sampled $B=100$ datasets of sizes $n\in\{200, 500, 1000, 2000\}$ with replacement, using probability weights $P(X)$. If a sampled dataset contained only one observed level of $T$ or $Y$, we discarded it and resampled.
  \item We analyzed results for each hyperparameter combination of \texttt{inter\_order} $\in \{1, 2, 3\}$, HTE level $\in\{\text{High},\text{Low}\}$ , and sample size $n\in\{200, 500, 1000, 2000\}$.
\end{itemize}

\subsubsection{SuperLearner Configuration}
Our library for all \texttt{SuperLearner} fits includes:
\begin{itemize}
  \item Generalized Linear Models (GLM): Logistic regression with main effects.
  \item Light Gradient Boosting Machine (LightGBM) \citep{ke2017lightgbm}: A tree-based method for nonlinear relationships.
  \item Bayesian Additive Regression Trees (BART) \citep{chipman2010bart}: A Bayesian tree ensemble for flexible modeling.
  \item Multivariate Adaptive Regression Splines (Earth) \citep{friedman1991mars}: A spline-based method for capturing nonlinearities.
\end{itemize}
\texttt{SuperLearner} uses 5-fold cross-validation internally to select weights for combining these candidate learners; this cross-validation is entirely within the \texttt{SuperLearner} fitting procedure and is distinct from the sample-splitting cross-fitting used by the DR- and R-learner estimators. In our implementation, each sampled dataset is split into two complementary halves. For each split, nuisance models are fit on one half, the other half is used to construct the pseudo-outcomes and fit the second-stage regression, and the final predictions are obtained by averaging the two swapped-split fits.

\subsubsection{Truncation of estimates}
In simulation, we truncate estimated outcome probabilities and propensity scores using fixed bounds chosen with reference to the known support and overlap restrictions of the DGP in order to avoid numerical instability. Specifically:
\begin{itemize}
  \item \textbf{Outcome probabilities}: All estimated probabilities $\hat p_t(x)$ are truncated to $[0.05, 0.95]$, matching the outcome-probability support used in the DGP. This applies uniformly to all estimators that produce probability estimates, including \texttt{LR}, \texttt{LR-T}, \texttt{SL}, \texttt{SL-T}, and the first-stage nuisance estimates in all DR- and R-learner variants.
  \item \textbf{Propensity scores}: All estimated propensity scores $\hat e(x)$ are truncated to $[0.01, 0.99]$. This is a conservative overlap truncation used for numerical stability in the simulation study; it does not require knowledge of the true propensity-score function. This applies to all estimators that use propensity scores: \texttt{DR-P}, \texttt{DR-CATE}, \texttt{DR-OR}, \texttt{DR-LOR}, \texttt{DR-RR}, \texttt{DR-LRR}, \texttt{R-CATE}, \texttt{R-OR}, \texttt{R-LOR}, \texttt{R-RR}, \texttt{R-LRR}.  
  \item \textbf{Two-stage models}: For all two-stage estimators, truncation is first applied to the nuisance outcome and propensity estimates used to construct the second-stage pseudo-outcomes. For \texttt{DR-P}, the second stage refines treatment-arm probabilities, so the resulting estimates $\hat p_t^*(x)$ are truncated again to $[0.05, 0.95]$ before forming the target contrast. For the direct-effect DR- and R-learner estimators, the second stage targets the effect itself rather than treatment-arm probabilities, so the final second-stage predictions are truncated to the parameter-specific feasible ranges implied by the probability clamps.

\end{itemize}

In preliminary experiments, the ratio-scale estimators (\texttt{DR-OR}, \texttt{DR-RR}, \texttt{R-OR}, \texttt{R-RR}) produced similar but modestly inferior results compared to their log-scale counterparts. We therefore focus exclusively on the log-scale variants (\texttt{DR-LOR}, \texttt{DR-LRR}, \texttt{R-LOR}, \texttt{R-LRR}) in all subsequent results.

\subsection{Metrics}
To compare estimators fairly across different effect scales, we adopt scale-appropriate transformations. For CATE, we compare on the original probability difference scale, where $\theta(x) = p_1(x) - p_0(x)$. For the conditional OR and RR, we compare on the log scale, where $\theta(x) = \log \text{OR}(x)$ and $\theta(x) = \log \text{RR}(x)$, respectively. The log transformation symmetrizes the ratio metrics and provides a more meaningful comparison as ratios are naturally multiplicative. Letting $\hat \theta_B(x)=\frac{1}{B}\sum_{b=1}^{B}\hat\theta_b(x)$ denote the average of the estimator across the $B=100$ simulated datasets, the metrics are:
\begin{itemize}
  \item \text{Integrated Bias$^2$ (iBias$^2$)}: $\sum_x P(X=x) \cdot \bigl(\hat \theta_B(x) - \theta(x)\bigr)^2$.
  \item \text{Integrated Variance (iVariance)}: $\sum_x P(X=x) \cdot \frac{1}{B}\sum_{b=1}^{B}\bigl(\hat\theta_b(x) - \hat \theta_B(x)\bigr)^2$.
  \item \text{Integrated Mean Squared Error (iMSE)}: $\sum_x P(X=x) \cdot \frac{1}{B}\sum_{b=1}^B\bigl(\hat\theta_b(x)-\theta(x)\bigr)^2$.
\end{itemize}

\section{Results}
\label{sec:results}

The estimators compared are: S-Learner based on logistic regression (\texttt{LR}), T-Learner based on logistic regression (\texttt{LR-T}), S-Learner based on SuperLearner (\texttt{SL}), T-Learner based on SuperLearner (\texttt{SL-T}), doubly robust plug-in (\texttt{DR-P}), DR-learner variants (\texttt{DR-CATE}, \texttt{DR-LOR}, \texttt{DR-LRR}), and R-learner variants (\texttt{R-CATE}, \texttt{R-LOR}, \texttt{R-LRR}). 

Every numerical value reported in the text is then the median integrated mean squared error (iMSE), integrated squared bias (iBias$^2$), and integrated variance (iVariance) across the simulated distributions in the indicated interaction-order, sample-size, and HTE stratum, computed on the original scale for ATE and on the log scale for OR and RR. The figures show the full distribution of these DGP-specific summaries, while the corresponding tables report the medians and interquartile ranges from which the numerical comparisons in the text are taken. More detailed results are in Supplementary Materials Section~\ref{sec:supp_materials}. We organize our discussion by the joint configuration of HTE level and sample size.

Figures ~\ref{fig:COR_inter3_2000} - ~\ref{fig:COR_inter1_200} and ~\ref{fig:COR_inter3_1000} - ~\ref{fig:CATE_inter1_200} display the empirical reliability function (complementary CDF) of the metric across the generated DGPs, stratified by HTE level. For a given estimator, the curve plots $P(\text{metric} > t)$ on the y-axis against the metric value $t$ on the x-axis. A curve that drops steeply to zero at low values of $t$ indicates an estimator that performs well across most DGPs. Conversely, a curve with a long right tail indicates that the estimator has poor worst-case performance in some DGPs. An estimator stochastically dominates another when its error distribution curve lies entirely to the left, indicating a consistently smaller metric across all quantiles.

\subsection{Conditional OR Performance}

\subsubsection{HTE high, large samples}
At interaction order 3 with $n=2000$, Table~\ref{tab:cor_3_2000_hte_h} reports that \texttt{DR-LOR} has median integrated mean squared error 0.393 across simulated distributions, compared with 0.744 for \texttt{LR}. This is one of the clearest settings in which the proposed DR-learner estimators outperform the parametric alternatives. The decomposition of the error is informative: for \texttt{LR}, the median integrated squared bias is 0.735, so nearly all of the median integrated mean squared error is explained by bias, which is consistent with model misspecification; for \texttt{DR-LOR}, the median integrated mean squared error is more evenly split between integrated squared bias (0.216) and integrated variance (0.162). \texttt{DR-P} is also close in median integrated mean squared error (0.451) and attains the smallest median integrated squared bias (0.140), although its median integrated variance is larger than that of \texttt{DR-LOR}; Figure~\ref{fig:COR_inter3_2000} also suggests better right-tail behavior for \texttt{DR-P}. \texttt{SL} and \texttt{SL-T} are also competitive: among the two, \texttt{SL} has smaller integrated variance, whereas \texttt{SL-T} has smaller integrated squared bias.

At interaction order 1 with $n=2000$ (Table~\ref{tab:cor_1_2000_hte_h}; Figure~\ref{fig:COR_inter1_2000}), several estimators perform similarly in terms of median integrated mean squared error: \texttt{SL}, \texttt{LR-T}, and \texttt{SL-T} attain 0.152, 0.174, and 0.187, while \texttt{LR}, \texttt{DR-P}, and \texttt{DR-LOR} are all close to 0.24. Their error decompositions differ, however. \texttt{LR} reaches that median mainly through extremely small integrated variance (0.015) despite a much larger integrated squared bias (0.215), whereas \texttt{DR-P} and \texttt{DR-LOR} trade somewhat higher variance for substantially less bias. \texttt{R-LOR} attains low median integrated squared bias (0.039) but very large integrated variance (0.564), which leaves it clearly behind overall. This pattern is consistent with the R-learner weighting $w(X)=\hat{e}(X)\{1-\hat{e}(X)\}$, which gives the largest weight to observations with estimated propensity scores near 0.5 and downweights observations with estimated propensities near 0 or 1. When the strongest heterogeneity lies in those lower-weight regions, that reweighting can hurt overall performance.

\subsubsection{HTE low, large samples}
At interaction order 3 with $n=2000$, low heterogeneity favors the simpler learners. Table~\ref{tab:cor_3_2000_hte_l} shows that \texttt{SL} has the smallest median integrated mean squared error (0.046), followed by \texttt{LR} (0.085) and \texttt{DR-LOR} (0.139). The advantage of \texttt{SL} comes from keeping both median integrated squared bias (0.013) and integrated variance (0.030) small. \texttt{LR} does attain a smaller first quartile of the integrated mean squared error distribution than \texttt{SL} (0.023 versus 0.030), suggesting stronger performance on the easiest DGPs in this stratum. However, its median integrated mean squared error is still clearly larger than that of \texttt{SL}, and Figure~\ref{fig:COR_inter3_2000} shows a heavier right tail, indicating less stable performance on the more difficult DGPs within this stratum.

At interaction order 1 with $n=2000$ (Table~\ref{tab:cor_1_2000_hte_l}; Figure~\ref{fig:COR_inter1_2000}), \texttt{LR} and \texttt{SL} are nearly tied in median integrated mean squared error (0.024 versus 0.027). This near-tie is also unsurprising given our implementation: the \texttt{SuperLearner} library includes a main-effects logistic regression, so in relatively simple settings the ensemble can place substantial weight on a model very close to \texttt{LR}, while still retaining flexibility when the underlying signal is more nonlinear or interaction-heavy. Both estimators have median integrated squared bias equal to 0.003, so the small advantage of \texttt{LR} is mainly a variance story. However, \texttt{LR} again shows a somewhat heavier right tail than \texttt{SL}.

\subsubsection{HTE high, small samples}
In settings with high treatment effect heterogeneity and limited sample sizes, integrated variance is the dominant component of error for all nonparametric estimators. This is clearly seen at interaction order 3 with $n=200$ (Table~\ref{tab:cor_3_200_hte_h}; Figure~\ref{fig:COR_inter3_200}), where the ranking by median integrated mean squared error is \texttt{SL} (0.691), \texttt{LR} (0.884), \texttt{SL-T} (1.060), \texttt{LR-T} (1.446), \texttt{DR-LOR} (1.801), \texttt{DR-P} (2.011), and \texttt{R-LOR} (2.395). That ordering closely matches the ranking by median integrated variance. A similar pattern emerges at interaction order 1 with $n=200$ (Table~\ref{tab:cor_1_200_hte_h}; Figure~\ref{fig:COR_inter1_200}), except that \texttt{DR-P} slightly outperforms \texttt{DR-LOR}. These results show that in data-sparse regimes, the stability of an estimator can outweigh its asymptotic bias-reduction properties, therefore masking any potential gains from orthogonal/debiased learners.

\subsubsection{HTE low, small samples}
The overall pattern is similar to the HTE-high, small-sample case, but all methods improve because the target is easier when heterogeneity is weaker. At interaction order 3 with $n=200$, Table~\ref{tab:cor_3_200_hte_l} shows that \texttt{SL} and \texttt{LR} again have the smallest median integrated mean squared errors (0.147 and 0.202), well ahead of \texttt{DR-LOR} (1.455) and \texttt{DR-P} (1.743). At interaction order 1 with $n=200$ (Table~\ref{tab:cor_1_200_hte_l}; Figure~\ref{fig:COR_inter1_200}), the ranking is similar: \texttt{SL} (0.134) and \texttt{LR} (0.176) remain best, followed by \texttt{SL-T}, \texttt{LR-T}, \texttt{DR-P}, \texttt{DR-LOR}, and \texttt{R-LOR}. Relative to the HTE-high case, the main difference is that the median errors for \texttt{SL} and \texttt{LR} drop substantially, reflecting the weaker heterogeneity, but the main message is unchanged: with only $n=200$, variance still dominates and the orthogonal learners do not yet recover their large-sample advantage.

\subsection{Conditional RR Performance}
The pattern for the conditional RR broadly mirrors that for the conditional OR, but the gains from the debiased ratio learners are somewhat more modest. In the most complex large-sample setting, interaction order 3 with $n=2000$ and HTE High, Table~\ref{tab:crr_3_2000_hte_h} shows median integrated mean squared errors of 0.144 for \texttt{SL}, 0.148 for \texttt{SL-T}, 0.153 for \texttt{DR-LRR}, and 0.174 for \texttt{DR-P}, compared with 0.213 for \texttt{LR}. Within the DR-family estimators, directly targeting the log risk ratio lowers both median iMSE and integrated variance relative to \texttt{DR-P} (0.153 versus 0.174 and 0.070 versus 0.115), although \texttt{DR-P} retains the smaller median integrated squared bias (0.047 versus 0.080). At the same time, \texttt{SL} and \texttt{SL-T} remain slightly ahead of \texttt{DR-LRR} in this stratum, suggesting that conditional RR may be somewhat easier to learn than conditional OR in these DGPs, so the incremental benefit of the extra debiasing step is smaller.

That advantage disappears as the setting becomes easier or the sample size becomes smaller. At interaction order 1 with $n=2000$ and HTE High, Table~\ref{tab:crr_1_2000_hte_h} shows that \texttt{DR-P} slightly outperforms \texttt{DR-LRR} in median iMSE (0.095 versus 0.113). With weak heterogeneity, the simpler learners are clearly preferable: at interaction order 1 with $n=2000$ and HTE Low, Table~\ref{tab:crr_1_2000_hte_l} reports median iMSEs of 0.010 for \texttt{LR} and 0.011 for \texttt{SL}, compared with 0.081 for \texttt{DR-P} and 0.094 for \texttt{DR-LRR}. With only $n=200$ and interaction order 3, the variance penalty for the debiased estimators becomes much larger; Table~\ref{tab:crr_3_200_hte_h} gives median iMSEs of 0.214 for \texttt{SL} and 0.252 for \texttt{LR}, versus 0.647 for \texttt{DR-LRR} and 0.725 for \texttt{DR-P}. Overall, the RR results reinforce the same practical pattern as the OR results: the debiased ratio learners are most attractive in the genuinely complex, data-rich regimes, whereas \texttt{SL} and \texttt{LR} remain preferable once the problem becomes easier.

\subsection{Conditional ATE Performance}
The conditional ATE results mainly serve as a benchmark against the existing DR-learner literature, and they show the same broad pattern with smaller absolute differences across estimators. In complex settings with substantial heterogeneity and adequate sample size, \texttt{DR-CATE} is often the best or among the best estimators. At interaction order 3 with $n=2000$ and HTE High, Table~\ref{tab:cate_3_2000_hte_h} reports median iMSEs of 0.012 for \texttt{DR-CATE}, 0.014 for \texttt{DR-P}, 0.017 for \texttt{R-CATE}, and 0.013 for \texttt{SL}. At interaction order 2 with $n=2000$ and HTE High, Table~\ref{tab:cate_2_2000_hte_h} shows the same ranking within the doubly robust family, with 0.009 for \texttt{DR-CATE} and 0.011 for \texttt{DR-P}. These results are consistent with the usual DR-learner intuition that directly learning the target contrast can help when treatment effects are genuinely heterogeneous, although the improvement over other flexible learners is modest in absolute terms.

As with RR and OR, that advantage disappears in easier or more data-limited regimes. At interaction order 3 with $n=2000$ and HTE Low, Table~\ref{tab:cate_3_2000_hte_l} shows that \texttt{SL} has the smallest median iMSE (0.001), while \texttt{LR} and \texttt{DR-CATE} are tied at 0.004 and \texttt{DR-P} reaches 0.010. With only $n=200$ and interaction order 3, the simpler learners again dominate: Table~\ref{tab:cate_3_200_hte_l} gives 0.006 for \texttt{SL}, 0.008 for \texttt{LR}, 0.045 for \texttt{DR-CATE}, and 0.058 for \texttt{DR-P}; in the corresponding HTE-high stratum, Table~\ref{tab:cate_3_200_hte_h} shows 0.026 for \texttt{SL} versus 0.060 for \texttt{DR-CATE}. Thus the CATE benchmark leads to the same practical takeaway as the OR and RR analyses: direct-targeting orthogonalization helps most in complex, data-rich settings, while \texttt{SL} or \texttt{LR} are often enough in easier regimes.

\subsection{Practical Recommendations}
\label{sec:practical_recs}
Based on our simulation results, we offer the following guidance for practitioners:

\begin{itemize}
  \item \textbf{Complex settings} (high HTE, complex outcome model, large sample size): The DR-learner estimators and \texttt{SL} clearly improve on logistic regression. Within the DR-family estimators, the main pattern is a bias-variance tradeoff: the direct-targeting estimators (\texttt{DR-LOR}/\texttt{DR-LRR}/\texttt{DR-CATE}) often reduce iVariance, and in the more complex large-sample settings this often translates into lower median iMSE, whereas \texttt{DR-P} often achieves lower iBias$^2$ and sometimes better right-tail behavior. The R-learner variants are not recommended because of their consistently high variance. If one wants an estimator built from the orthogonal pseudo-outcomes developed here, prefer a DR-learner estimator over \texttt{SL}. Otherwise, \texttt{SL} is usually a strong default.
  \item \textbf{Simpler or data-limited settings} (low HTE, low DGP complexity, or small sample size): \texttt{SL} and \texttt{LR} usually achieve the smallest iMSE. \texttt{LR} offers interpretability and lower computational cost, whereas \texttt{SL} tends to provide better worst-case performance across the full range of DGPs.
  \item \textbf{Intermediate or uncertain settings}: \texttt{SL} provides a robust default.
\end{itemize}

\begin{figure}[H]
\centering
\includegraphics[width=1\textwidth]{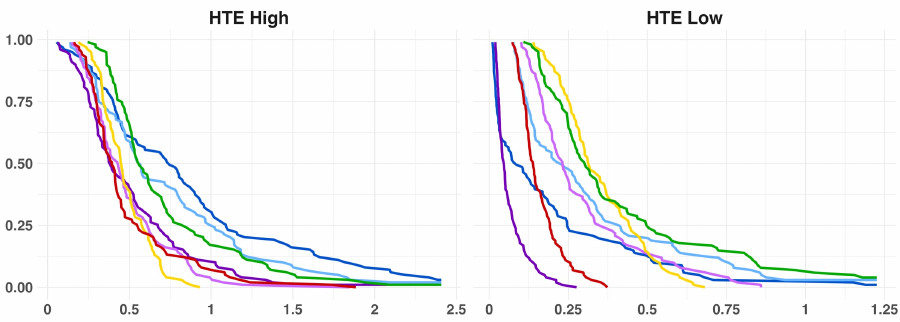}
\caption*{\footnotesize (a) iMSE}

\vspace{0.5em}
\includegraphics[width=1\textwidth]{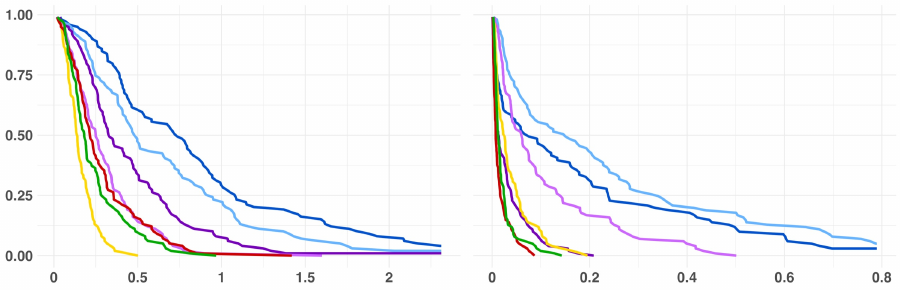}
\caption*{\footnotesize (b) iBias$^2$}

\vspace{0.5em}
\includegraphics[width=1\textwidth]{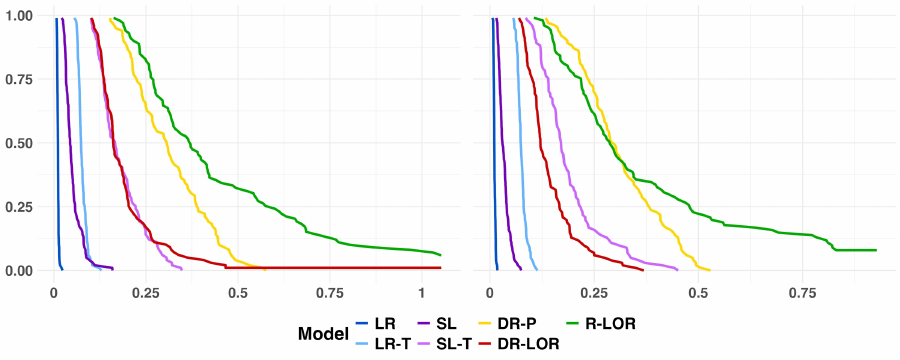}
\caption*{\footnotesize (c) iVariance}

\caption{Conditional OR: interaction order 3, sample size 2000}
\label{fig:COR_inter3_2000}
\end{figure}

\begin{figure}[H]
\centering
\includegraphics[width=1\textwidth]{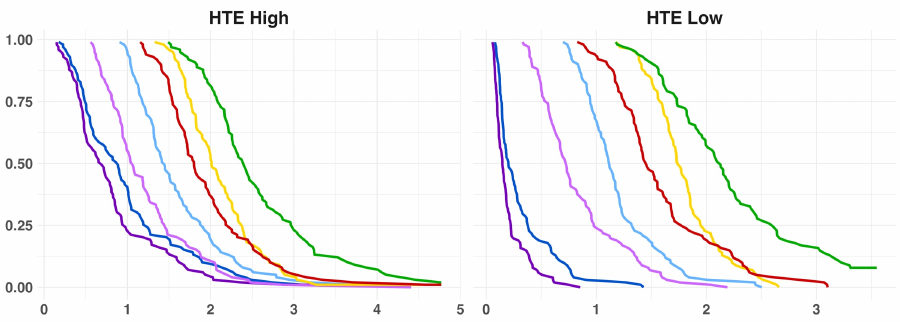}
\caption*{\footnotesize (a) iMSE}

\vspace{0.5em}
\includegraphics[width=1\textwidth]{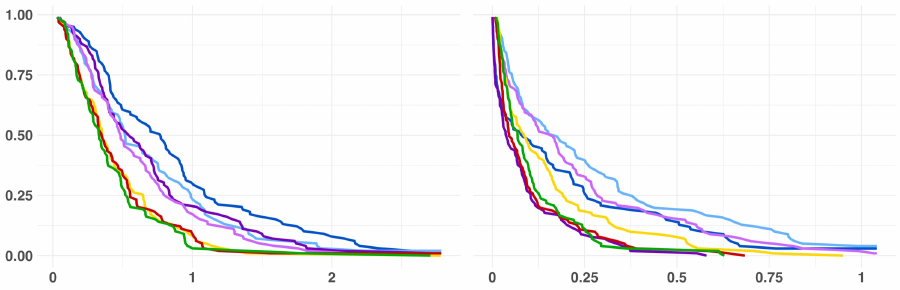}
\caption*{\footnotesize (b) iBias$^2$}

\vspace{0.5em}
\includegraphics[width=1\textwidth]{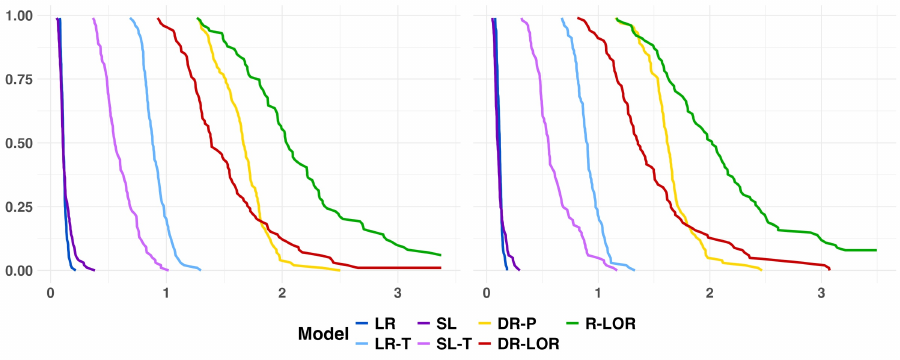}
\caption*{\footnotesize (c) iVariance}

\caption{Conditional OR: interaction order 3, sample size 200}
\label{fig:COR_inter3_200}
\end{figure}

\begin{figure}[H]
\centering
\includegraphics[width=1\textwidth]{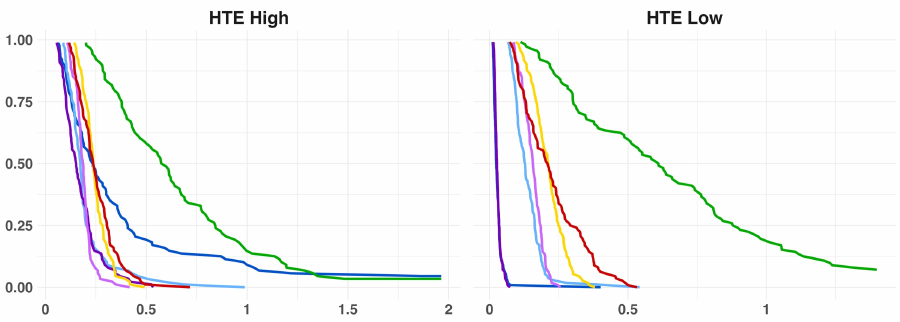}
\caption*{\footnotesize (a) iMSE}

\vspace{0.5em}
\includegraphics[width=1\textwidth]{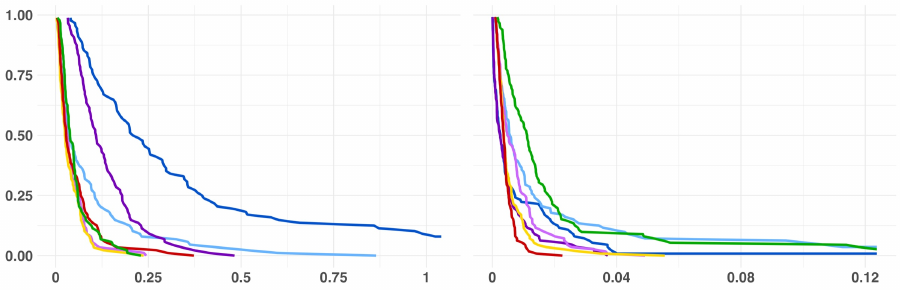}
\caption*{\footnotesize (b) iBias$^2$}

\vspace{0.5em}
\includegraphics[width=1\textwidth]{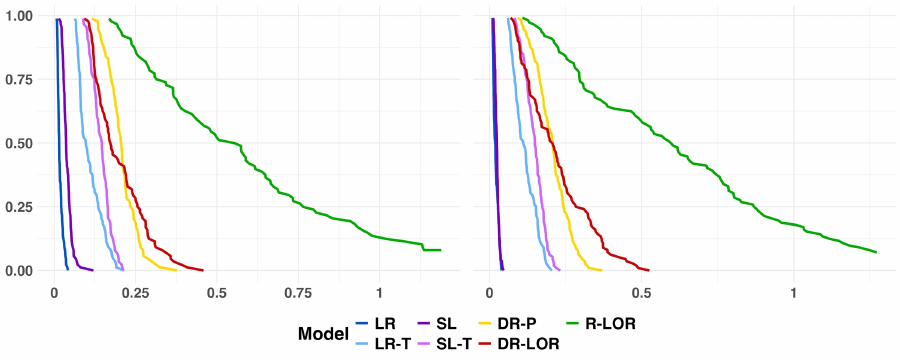}
\caption*{\footnotesize (c) iVariance}

\caption{Conditional OR: interaction order 1, sample size 2000}
\label{fig:COR_inter1_2000}
\end{figure}

\begin{figure}[H]
\centering
\includegraphics[width=1\textwidth]{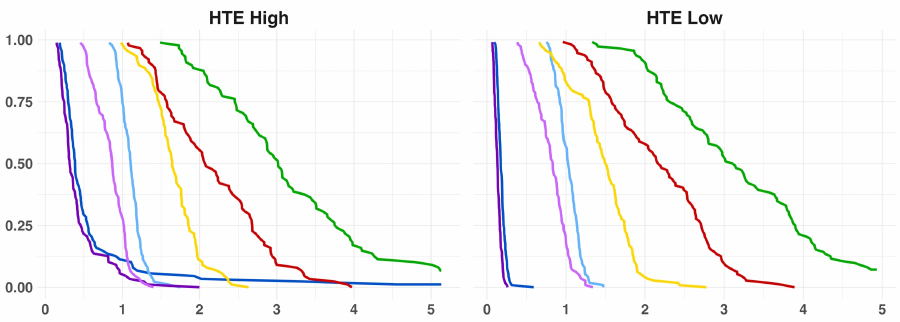}
\caption*{\footnotesize (a) iMSE}

\vspace{0.5em}
\includegraphics[width=1\textwidth]{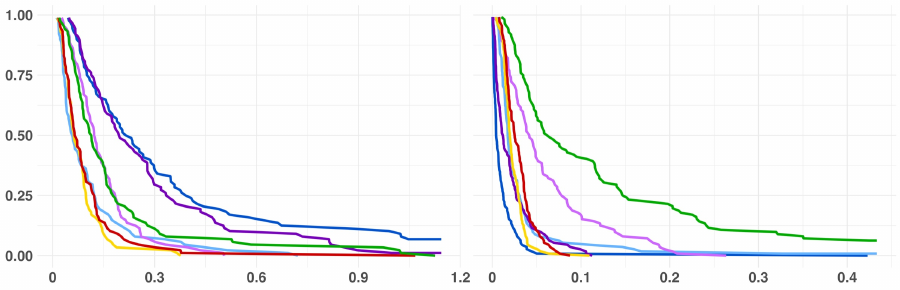}
\caption*{\footnotesize (b) iBias$^2$}

\vspace{0.5em}
\includegraphics[width=1\textwidth]{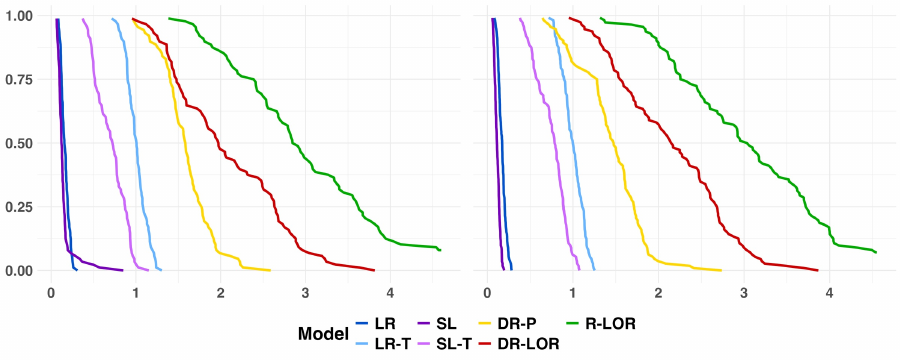}
\caption*{\footnotesize (c) iVariance}

\caption{Conditional OR: interaction order 1, sample size 200}
\label{fig:COR_inter1_200}
\end{figure}

\section{Illustrative Application}
\label{sec:illustrative}

We apply our estimation methods to real-world data from the National Health and Nutrition Examination Survey (NHANES) \citep{NHANES} to demonstrate their practical utility for detecting heterogeneous treatment effects in observational health data.

\subsection{Data and Study Design}

The NHANES is a nationally representative survey of the U.S. civilian population, collecting comprehensive health, nutrition, and demographic information. We focus on assessing the effect of physical activity on sleep trouble while allowing the conditional effect measure to vary across demographic, socioeconomic, and health characteristics.

Our analytic sample includes 6,632 adults (age $\geq$ 18 years) with complete covariate and outcome data. The treatment variable is self-reported participation in moderate or vigorous physical activity during a typical week (53.5\% prevalence). The outcome variable is self-reported trouble sleeping, defined as reporting difficulty falling asleep, staying asleep, or sleeping too much (26.6\% prevalence).

Baseline covariates include: age (continuous), gender (binary), body mass index (BMI, continuous), education level (5 categories), race/ethnicity (5 categories), poverty-income ratio (continuous), diabetes status (binary), and smoking history (binary). In implementation, education level and race/ethnicity are represented by indicator variables, and age, BMI, and poverty-income ratio are standardized. These covariates capture demographic, socioeconomic, and health factors that may confound the relationship between physical activity and sleep trouble and may also modify the treatment effect.

For the empirical analysis, we randomly split the analytic sample into an 80\% training sample and a 20\% validation sample. All estimators are trained on the training sample, and both the heterogeneity plots and the treatment-rule comparisons reported below are evaluated on the held-out validation sample.

For binary outcomes, OR- and RR-based treatment directions agree when they are derived from the same pair of conditional risks, since $OR(x)<1$ if and only if $RR(x)<1$, equivalently $p_1(x)<p_0(x)$. For this reason, we focus on the OR-based analysis below.

To connect these heterogeneity estimates to treatment-rule comparison, we define an OR-based rule for each estimator:
\[
d_m(x)=I\{\widehat{\log OR}_m(x)<0\}.
\]
Because $Y=1$ indicates sleep trouble, this rule recommends treatment when the estimated odds of sleep trouble are lower under treatment than under control.

We compare these rules through the mean counterfactual outcome $E[Y(d_m)]$, estimated on the validation sample using Targeted Maximum Likelihood Estimation (TMLE) \citep{vanderLaanRubin2006TMLE}. Under the standard assumptions of consistency, conditional exchangeability, and positivity, this quantity represents the mean outcome that would be observed if treatment were assigned according to $d_m$. Smaller values are therefore better. In this analysis, all rules are evaluated using a common set of risk predictions and propensity scores estimated by \texttt{SL}.

\subsection{Results}

Figure~\ref{fig:nhanes_application} summarizes the held-out validation-sample results. The left panel displays violin plots of the estimated conditional log odds ratios. Standard logistic regression (\texttt{LR}) yields a constant conditional OR (about 0.05 on the log scale) for all individuals, reflecting its implicit assumption of homogeneous treatment effects, since the log-odds ratio is constrained to depend only on the main effect of treatment. In contrast, the flexible learners and the orthogonal learners produce much broader distributions, suggesting substantial heterogeneity in the association between physical activity and sleep trouble. The corresponding log risk ratio plots were qualitatively similar and are omitted from the main text.

The right panel of Figure~\ref{fig:nhanes_application} shows that this rule-based analysis also separates the estimators. The lowest estimated mean outcome is attained by \texttt{DR-LOR} (0.151), followed by \texttt{DR-P} (0.157), \texttt{R-LOR} (0.162), and \texttt{SL-T} (0.184). These four estimators all perform substantially better than \texttt{LR} (0.259), and \texttt{Treat All} performs worst (0.291). These findings suggest that estimators capturing richer effect heterogeneity can induce materially better treatment rules than simple parametric models in this application. In this application, rather than recommending physical activity uniformly to all adults with sleep trouble, clinicians could target individuals more likely to benefit. The ability to estimate conditional treatment effects, rather than population averages, enables more targeted and efficient health interventions.

\begin{figure}[H]
\centering
\begin{subfigure}{.5\textwidth}
  \centering
  \includegraphics[width=\linewidth]{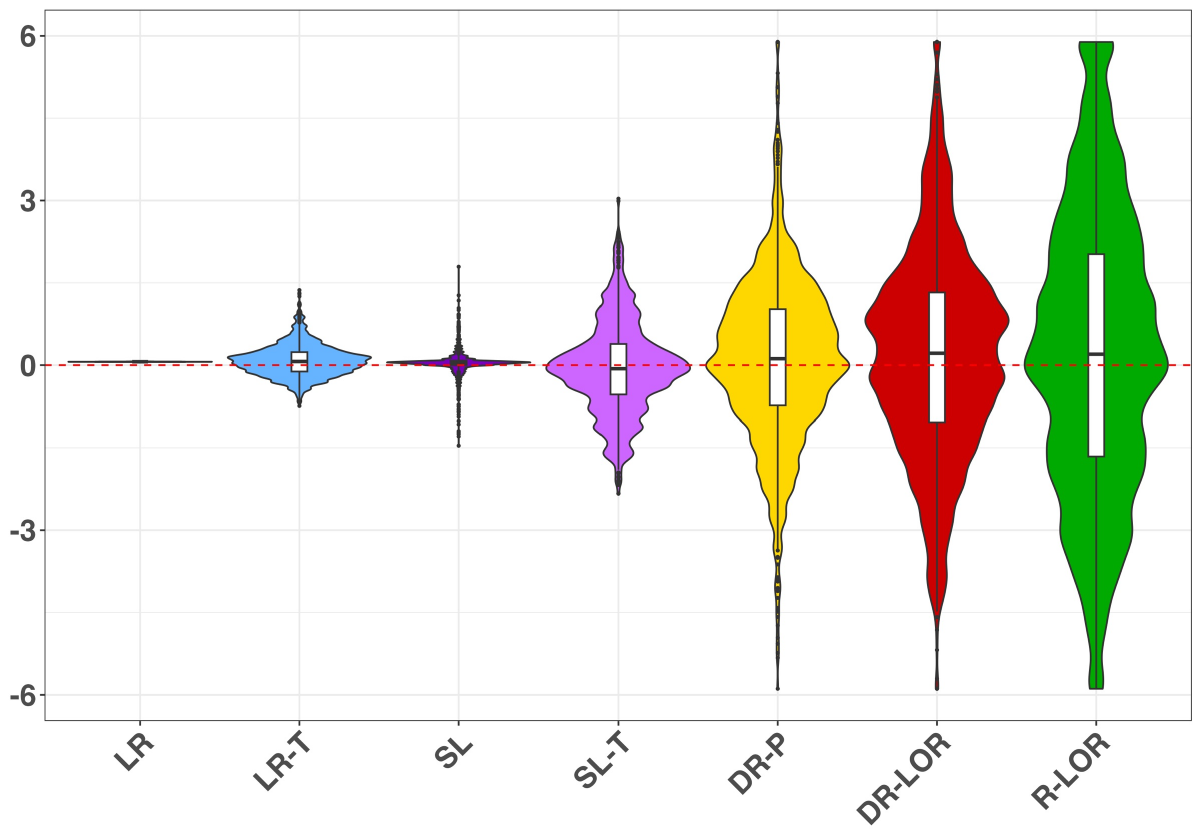}
  \caption{Conditional Odds Ratio (Log Scale)}
\end{subfigure}%
\begin{subfigure}{.5\textwidth}
  \centering
  \includegraphics[width=\linewidth]{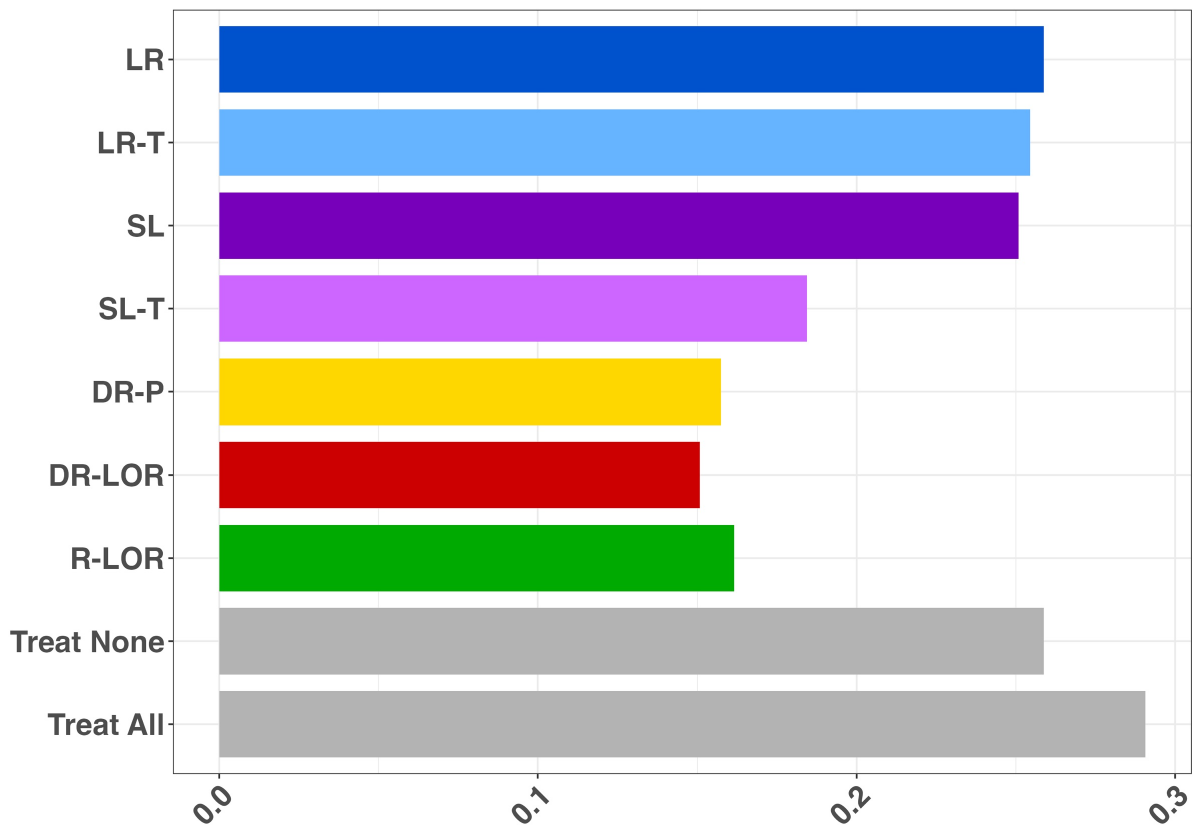}
  \caption{Estimated Mean Outcome under the Induced Rule}
\end{subfigure}
\caption{NHANES application for the effect of physical activity on sleep trouble. The left panel shows the validation-sample distribution of estimated conditional log odds ratios. The right panel compares the estimated mean outcome under OR-based treatment rules induced by each estimator, using TMLE. Because sleep trouble is the undesirable outcome, lower values in the right panel are better.}
\label{fig:nhanes_application}
\end{figure}

\section{Discussion}
\label{sec:discussion}

This study provides a comprehensive comparison of methods for estimating conditional treatment effects on three scales: the risk difference (ATE), odds ratio, and risk ratio. Our theoretical development of orthogonal pseudo-outcomes with second-order remainder properties, combined with extensive simulations and a real-world application, yields several important insights for applied researchers.

We make three primary contributions. First, we propose DR-learner and R-learner estimators for the conditional odds ratio and risk ratio based on efficient-influence-function-derived orthogonal pseudo-outcomes. These pseudo-outcomes satisfy conditional-mean identities in which the target parameter is recovered up to a second-order remainder term. Second, we conduct a rigorous simulation study comparing these DR-learner methods against model (logistic regression) and flexible machine learning (SuperLearner) S- and T-Learners across diverse randomly generated data-generating processes. Third, we demonstrate the practical utility of these methods using NHANES data to estimate heterogeneous effects of physical activity on sleep trouble and to compare the treatment rules induced by different estimators.

Based on our findings (see also the Practical Recommendations in Section~\ref{sec:practical_recs}), we offer the following guidance. In complex scenarios, DR-learner estimators and SuperLearner clearly outperform logistic regression. Within the DR-learner family, a bias--variance tradeoff emerges: direct-targeting estimators (e.g., DR-LOR, DR-LRR, DR-CATE) tend to reduce variance and often achieve lower median iMSE in large-sample complex settings, whereas DR-P often attains lower squared bias and can exhibit more favorable right-tail behavior. R-learner variants are not recommended due to consistently high variance. If one wants an estimator built from the orthogonal pseudo-outcomes introduced here, DR-learner estimators are preferred; otherwise, SuperLearner is typically a strong default. In simpler or data-limited settings, logistic regression and SuperLearner often achieve the lowest iMSE, with logistic regression offering interpretability and computational efficiency, and SuperLearner providing more robust worst-case performance. For intermediate or uncertain settings, SuperLearner remains a reliable default choice.

We note that even in the most complex scenarios considered here, the data-generating mechanisms remain relatively simple compared to real-world data. Nevertheless, the proposed DR-family estimators already match or surpass the best competing methods in these settings. In more complex real-world applications, where data are higher-dimensional and exhibit greater heterogeneity, we expect the theoretical advantages of orthogonal and debiased estimation to become even more pronounced. The illustrative application in Section~\ref{sec:illustrative} provides further empirical support for this perspective.

Several limitations of our study warrant discussion. First, our simulation DGP, while flexible, uses a relatively low-dimensional covariate space (five binary and one continuous covariate) due to computational resource limitations. Real-world applications often involve higher-dimensional confounders. Second, we focus on binary outcomes and treatments; extensions to continuous or survival outcomes require additional methodological development and constitute a natural direction for future work.


\clearpage
\section{Supplementary Materials} \label{sec:supp_materials}

\subsection{Supplementary proofs for the novel estimators}
\subsubsection{Proofs of Theorem~\ref{thm:dr_transforms}}
\label{sec:supp_thm1_proofs}

We first prove the OR case, then treat $\log OR$, $RR$, and $\log RR$ analogously.

\begin{proof}[Proof of Theorem~\ref{thm:dr_transforms} for the OR case]
Let
\[
\psi(u,v)=\frac{u(1-v)}{v(1-u)},
\qquad
\Delta(X)=
\begin{pmatrix}
p_1(X)-q_1(X)\\[2pt]
p_0(X)-q_0(X)
\end{pmatrix}.
\]
Then $OR(X)=\psi(p_1(X),p_0(X))$ and $OR_\eta(X)=\psi(q_1(X),q_0(X))$. Also,
\[
\partial_1\psi(u,v)=\frac{1-v}{v(1-u)^2},
\qquad
\partial_0\psi(u,v)=-\frac{u}{(1-u)v^2}.
\]
Taking the conditional expectation of $\varphi_{OR}(Z;\eta)$ given $X$,
\begin{align}
E[\varphi_{OR}(Z;\eta)\mid X]
&=OR_\eta(X)+\frac{e(X)}{\pi(X)}\partial_1\psi(q_1(X),q_0(X))\{p_1(X)-q_1(X)\}\nonumber\\
&\quad+\frac{1-e(X)}{1-\pi(X)}\partial_0\psi(q_1(X),q_0(X))\{p_0(X)-q_0(X)\}.
\label{eq:dror_conditional_mean_generic}
\end{align}
Adding and subtracting the first-order Taylor expansion of $\psi(p_1(X),p_0(X))$ around $(q_1(X),q_0(X))$, and applying the mean-value form of Taylor's theorem to the second-order remainder (using that $\psi$ is twice continuously differentiable on every compact subset of $(0,1)^2$), we obtain
\begin{align}
E[\varphi_{OR}(Z;\eta)\mid X] = OR(X) + R_{OR}(X;\eta),
\label{eq:dror_second_order_identity}
\end{align}
with
\begin{align}
R_{OR}(X;\eta)
&=
\frac{e(X)-\pi(X)}{\pi(X)}\partial_1\psi(q_1(X),q_0(X))\{p_1(X)-q_1(X)\}\nonumber\\
&\quad+
\frac{\pi(X)-e(X)}{1-\pi(X)}\partial_0\psi(q_1(X),q_0(X))\{p_0(X)-q_0(X)\}\nonumber\\
&\quad-\frac{1}{2}\Delta(X)^\top H_\psi(\bar q(X))\Delta(X),
\label{eq:dror_remainder_explicit}
\end{align}
where $\bar{q}(X)$ lies on the segment joining $(q_1(X),q_0(X))$ and $(p_1(X),p_0(X))$, and
\[
H_\psi(u,v)=
\begin{pmatrix}
\dfrac{2(1-v)}{v(1-u)^3} & -\dfrac{1}{(1-u)^2v^2}\\[10pt]
-\dfrac{1}{(1-u)^2v^2} & \dfrac{2u}{(1-u)v^3}
\end{pmatrix}.
\]
The remainder is second order: the first two terms are products of propensity and outcome regression errors, and the last term is quadratic in the outcome regression errors. If, in addition, the nuisance class satisfies the boundary condition in Theorem~\ref{thm:or_orthogonal}, then there exists $C<\infty$ such that
\[
|R_{OR}(X;\eta)|
\le
C\Big(
|e(X)-\pi(X)|\,|p_1(X)-q_1(X)|
+|e(X)-\pi(X)|\,|p_0(X)-q_0(X)|
+\|\Delta(X)\|^2
\Big).
\]
Setting $\eta=\hat\eta$ gives $E[\hat\varphi_{OR}(Z)\mid X]=OR(X)+R_{OR}(X;\hat\eta)$.
\end{proof}

\begin{proof}[Proof of Theorem~\ref{thm:dr_transforms} for $\log OR$]
Let
\[
\psi(u,v)=\log\!\left(\frac{u(1-v)}{v(1-u)}\right) = \log u - \log(1-u) + \log(1-v) - \log v,
\qquad
\Delta(X)=
\begin{pmatrix}
p_1(X)-q_1(X)\\[2pt]
p_0(X)-q_0(X)
\end{pmatrix}.
\]
Then $\log OR(X)=\psi(p_1(X),p_0(X))$ and $\log OR_\eta(X)=\psi(q_1(X),q_0(X))$. The partial derivatives are
\[
\partial_1\psi(u,v)=\frac{1}{u(1-u)},
\qquad
\partial_0\psi(u,v)=-\frac{1}{v(1-v)}.
\]
Taking the conditional expectation of $\varphi_{\log OR}(Z;\eta)$ given $X$,
\begin{align*}
E[\varphi_{\log OR}(Z;\eta)\mid X]
&=\log OR_\eta(X)+\frac{e(X)}{\pi(X)}\,\partial_1\psi(q_1(X),q_0(X))\{p_1(X)-q_1(X)\}\\
&\quad+\frac{1-e(X)}{1-\pi(X)}\,\partial_0\psi(q_1(X),q_0(X))\{p_0(X)-q_0(X)\}.
\end{align*}
Adding and subtracting the first-order Taylor expansion of $\psi(p_1(X),p_0(X))$ around $(q_1(X),q_0(X))$, and applying the mean-value form of Taylor's theorem to the second-order remainder (using that $\psi$ is twice continuously differentiable on every compact subset of $(0,1)^2$), we obtain
\[
E[\varphi_{\log OR}(Z;\eta)\mid X] = \log OR(X) + R_{\log OR}(X;\eta),
\]
with
\begin{align*}
R_{\log OR}(X;\eta)
&=\frac{e(X)-\pi(X)}{\pi(X)}\,\partial_1\psi(q_1(X),q_0(X))\{p_1(X)-q_1(X)\}\\
&\quad+\frac{\pi(X)-e(X)}{1-\pi(X)}\,\partial_0\psi(q_1(X),q_0(X))\{p_0(X)-q_0(X)\}\\
&\quad-\frac{1}{2}\Delta(X)^\top H_\psi(\bar q(X))\,\Delta(X),
\end{align*}
where $\bar{q}(X)$ lies on the segment joining $(q_1(X),q_0(X))$ and $(p_1(X),p_0(X))$, and
\[
H_\psi(u,v)=
\begin{pmatrix}
\dfrac{2u-1}{u^2(1-u)^2} & 0\\[10pt]
0 & \dfrac{1-2v}{v^2(1-v)^2}
\end{pmatrix}.
\]
The Hessian is diagonal because $\psi$ is additively separable in $u$ and $v$. The remainder is second order: the first two terms are products of propensity and outcome regression errors, and the last term is quadratic in the outcome regression errors. If, in addition, the nuisance class satisfies the boundary condition in Theorem~\ref{thm:or_orthogonal}, then there exists $C<\infty$ such that
\[
|R_{\log OR}(X;\eta)|
\le
C\Big(
|e(X)-\pi(X)|\,|p_1(X)-q_1(X)|
+|e(X)-\pi(X)|\,|p_0(X)-q_0(X)|
+\|\Delta(X)\|^2
\Big).
\]
Setting $\eta=\hat\eta$ gives $E[\hat\varphi_{\log OR}(Z)\mid X]=\log OR(X)+R_{\log OR}(X;\hat\eta)$.
\end{proof}

\begin{proof}[Proof of Theorem~\ref{thm:dr_transforms} for $RR$]
Let
\[
\psi(u,v)=\frac{u}{v},
\qquad
\Delta(X)=
\begin{pmatrix}
p_1(X)-q_1(X)\\[2pt]
p_0(X)-q_0(X)
\end{pmatrix}.
\]
Then $RR(X)=\psi(p_1(X),p_0(X))$ and $RR_\eta(X)=\psi(q_1(X),q_0(X))$. The partial derivatives are
\[
\partial_1\psi(u,v)=\frac{1}{v},
\qquad
\partial_0\psi(u,v)=-\frac{u}{v^2}.
\]
Taking the conditional expectation of $\varphi_{RR}(Z;\eta)$ given $X$, we obtain
\begin{align*}
E[\varphi_{RR}(Z;\eta)\mid X]
&=RR_\eta(X)+\frac{e(X)}{\pi(X)}\,\partial_1\psi(q_1(X),q_0(X))\{p_1(X)-q_1(X)\}\\
&\quad+\frac{1-e(X)}{1-\pi(X)}\,\partial_0\psi(q_1(X),q_0(X))\{p_0(X)-q_0(X)\}.
\end{align*}
Adding and subtracting the first-order Taylor expansion of $\psi(p_1(X),p_0(X))$ around $(q_1(X),q_0(X))$, and applying the mean-value form of Taylor's theorem to the second-order remainder (using that $\psi$ is twice continuously differentiable on every compact subset of $(0,1)^2$), we obtain
\[
E[\varphi_{RR}(Z;\eta)\mid X] = RR(X) + R_{RR}(X;\eta),
\]
with
\begin{align*}
R_{RR}(X;\eta)
&=\frac{e(X)-\pi(X)}{\pi(X)}\,\partial_1\psi(q_1(X),q_0(X))\{p_1(X)-q_1(X)\}\\
&\quad+\frac{\pi(X)-e(X)}{1-\pi(X)}\,\partial_0\psi(q_1(X),q_0(X))\{p_0(X)-q_0(X)\}\\
&\quad-\frac{1}{2}\Delta(X)^\top H_\psi(\bar q(X))\,\Delta(X),
\end{align*}
where $\bar q(X)$ lies on the segment joining $(q_1(X),q_0(X))$ and $(p_1(X),p_0(X))$, and
\[
H_\psi(u,v)=
\begin{pmatrix}
0 & -\dfrac{1}{v^2}\\[10pt]
-\dfrac{1}{v^2} & \dfrac{2u}{v^3}
\end{pmatrix}.
\]
The remainder is second order: the first two terms are products of propensity and outcome regression errors, and the last term is quadratic in the outcome regression errors. If, in addition, the nuisance class satisfies the boundary condition in Theorem~\ref{thm:or_orthogonal}, then there exists $C<\infty$ such that
\[
|R_{RR}(X;\eta)|
\le
C\Big(
|e(X)-\pi(X)|\,|p_1(X)-q_1(X)|
+|e(X)-\pi(X)|\,|p_0(X)-q_0(X)|
+\|\Delta(X)\|^2
\Big).
\]
Setting $\eta=\hat\eta$ gives $E[\hat\varphi_{RR}(Z)\mid X]=RR(X)+R_{RR}(X;\hat\eta)$.
\end{proof}

\begin{proof}[Proof of Theorem~\ref{thm:dr_transforms} for $\log RR$]
Let
\[
\psi(u,v)=\log u - \log v,
\qquad
\Delta(X)=
\begin{pmatrix}
p_1(X)-q_1(X)\\[2pt]
p_0(X)-q_0(X)
\end{pmatrix}.
\]
Then $\log RR(X)=\psi(p_1(X),p_0(X))$ and $\log RR_\eta(X)=\psi(q_1(X),q_0(X))$. The partial derivatives are
\[
\partial_1\psi(u,v)=\frac{1}{u},
\qquad
\partial_0\psi(u,v)=-\frac{1}{v}.
\]
Taking the conditional expectation of $\varphi_{\log RR}(Z;\eta)$ given $X$, we obtain
\begin{align*}
E[\varphi_{\log RR}(Z;\eta)\mid X]
&=\log RR_\eta(X)+\frac{e(X)}{\pi(X)}\,\partial_1\psi(q_1(X),q_0(X))\{p_1(X)-q_1(X)\}\\
&\quad+\frac{1-e(X)}{1-\pi(X)}\,\partial_0\psi(q_1(X),q_0(X))\{p_0(X)-q_0(X)\}.
\end{align*}
Adding and subtracting the first-order Taylor expansion of $\psi(p_1(X),p_0(X))$ around $(q_1(X),q_0(X))$, and applying the mean-value form of Taylor's theorem to the second-order remainder (using that $\psi$ is twice continuously differentiable on every compact subset of $(0,1)^2$), we obtain
\[
E[\varphi_{\log RR}(Z;\eta)\mid X] = \log RR(X) + R_{\log RR}(X;\eta),
\]
with
\begin{align*}
R_{\log RR}(X;\eta)
&=\frac{e(X)-\pi(X)}{\pi(X)}\,\partial_1\psi(q_1(X),q_0(X))\{p_1(X)-q_1(X)\}\\
&\quad+\frac{\pi(X)-e(X)}{1-\pi(X)}\,\partial_0\psi(q_1(X),q_0(X))\{p_0(X)-q_0(X)\}\\
&\quad-\frac{1}{2}\Delta(X)^\top H_\psi(\bar q(X))\,\Delta(X),
\end{align*}
where $\bar q(X)$ lies on the segment joining $(q_1(X),q_0(X))$ and $(p_1(X),p_0(X))$, and
\[
H_\psi(u,v)=
\begin{pmatrix}
-\dfrac{1}{u^2} & 0\\[10pt]
0 & \dfrac{1}{v^2}
\end{pmatrix}.
\]
As with $\log OR$, the Hessian is diagonal because $\psi$ is additively separable. The remainder is second order: the first two terms are products of propensity and outcome regression errors, and the last term is quadratic in the outcome regression errors. If, in addition, the nuisance class satisfies the boundary condition in Theorem~\ref{thm:or_orthogonal}, then there exists $C<\infty$ such that
\[
|R_{\log RR}(X;\eta)|
\le
C\Big(
|e(X)-\pi(X)|\,|p_1(X)-q_1(X)|
+|e(X)-\pi(X)|\,|p_0(X)-q_0(X)|
+\|\Delta(X)\|^2
\Big).
\]
Setting $\eta=\hat\eta$ gives $E[\hat\varphi_{\log RR}(Z)\mid X]=\log RR(X)+R_{\log RR}(X;\hat\eta)$.
\end{proof}

\subsubsection{Proof of Theorem~\ref{thm:or_orthogonal}}
\label{sec:supp_thm2_proof}

\begin{proof}
We prove parts~(i) and (ii) for an arbitrary fixed $\theta\in\{OR,\log OR,RR,\log RR\}$. For part~(iii), we verify the conditions in detail for the OR loss and defer the corresponding verifications for $\log OR$, $RR$, and $\log RR$ to Section~\ref{sec:supp_thm2_partiii}.

\medskip
\noindent\textbf{Proof of (i).}
Fix $\eta$ and write $m_{\theta,\eta}(X)=E[\varphi_{\theta}(Z;\eta)\mid X]$. By the bias--variance decomposition,
\begin{align*}
\mathcal{L}_{\theta}(f,\eta)
&= E\!\left[\left\{\varphi_{\theta}(Z;\eta)-m_{\theta,\eta}(X)+m_{\theta,\eta}(X)-f(X)\right\}^2\right]\\
&= E\!\left[\left\{\varphi_{\theta}(Z;\eta)-m_{\theta,\eta}(X)\right\}^2\right]
+ E\!\left[\left\{m_{\theta,\eta}(X)-f(X)\right\}^2\right],
\end{align*}
since the cross term vanishes:
$2\,E\!\left[\{m_{\theta,\eta}(X)-f(X)\}\,E[\varphi_{\theta}(Z;\eta)-m_{\theta,\eta}(X)\mid X]\right]=0$.
The first term does not depend on $f$, so $m_{\theta,\eta}$ is the unique minimizer over all measurable functions. Taking $\eta=\eta_0$ and using Theorem~\ref{thm:dr_transforms} gives $R_{\theta}(X;\eta_0)=0$ and hence $m_{\theta,\eta_0}(X)=\theta(X)$. Therefore
\[
\mathcal{L}_{\theta}(f,\eta_0)-\mathcal{L}_{\theta}(\theta,\eta_0) = E\!\left[(f(X)-\theta(X))^2\right].
\]

\medskip
\noindent\textbf{Proof of (ii).}
We verify the Neyman orthogonality condition in Assumption~1 of \citet{foster2023orthogonal}.

\noindent\textbf{Step 1 (compute $D_f\mathcal{L}_{\theta}$).}
By direct differentiation of the squared loss,
\begin{align}
D_f\mathcal{L}_{\theta}(f,\eta)[h]
&= -2\,E\!\left[(\varphi_{\theta}(Z;\eta)-f(X))\,h(X)\right].
\label{eq:DfL}
\end{align}

\noindent\textbf{Step 2 (evaluate at $(\theta,\eta_0)$).}
Conditioning on $X$ and using $E[\varphi_{\theta}(Z;\eta_0)\mid X]=\theta(X)$ (from Theorem~\ref{thm:dr_transforms} with $R_{\theta}=0$),
\begin{align}
D_f\mathcal{L}_{\theta}(\theta,\eta_0)[h]
&= -2\,E\!\left[E[\varphi_{\theta}(Z;\eta_0)-\theta(X)\mid X]\,h(X)\right] = 0.
\label{eq:DfL_at_true}
\end{align}
This confirms the first-order optimality condition in Assumption~2 of \citet{foster2023orthogonal}.

\noindent\textbf{Step 3 (compute the cross derivative $D_\eta D_f\mathcal{L}_{\theta}$).}
Taking the pathwise derivative of \eqref{eq:DfL} with respect to $\eta$ at $\eta_0$ in direction $k$:
\begin{align}
D_\eta D_f\mathcal{L}_{\theta}(\theta,\eta_0)[h,k]
&= -2\,E\!\left[D_\eta\varphi_{\theta}(Z;\eta_0)[k]\cdot h(X)\right].
\label{eq:crossderiv_raw}
\end{align}
Conditioning on $X$ and exchanging the derivative with the conditional expectation,
\begin{align}
D_\eta D_f\mathcal{L}_{\theta}(\theta,\eta_0)[h,k]
&= -2\,E\!\left[D_\eta\,E[\varphi_{\theta}(Z;\eta)\mid X]\Big|_{\eta=\eta_0}[k]\cdot h(X)\right].
\label{eq:crossderiv}
\end{align}
By Theorem~\ref{thm:dr_transforms}, $E[\varphi_{\theta}(Z;\eta)\mid X]=\theta(X)+R_{\theta}(X;\eta)$, so
\[
D_\eta E[\varphi_{\theta}(Z;\eta)\mid X]\Big|_{\eta=\eta_0}[k] = D_\eta R_{\theta}(X;\eta)\Big|_{\eta=\eta_0}[k].
\]
By Theorem~\ref{thm:dr_transforms} and Section~\ref{sec:supp_thm1_proofs}, each $R_{\theta}(X;\eta)$ is the sum of cross terms in the nuisance errors and a Taylor remainder that is quadratic in the outcome-regression errors. Hence every term in $R_{\theta}(X;\eta)$ contains at least two factors that vanish at $\eta=\eta_0$, so
\[
D_\eta R_{\theta}(X;\eta)\Big|_{\eta=\eta_0}[k] = 0,
\]
and substituting back into \eqref{eq:crossderiv},
\[
D_\eta D_f\mathcal{L}_{\theta}(\theta,\eta_0)[h,k] = 0.
\]
This establishes Neyman orthogonality for every $\theta\in\{OR,\log OR,RR,\log RR\}$.

\medskip
\noindent\textbf{Proof of (iii).}
We verify the conditions required by Theorem~3 of \citet{foster2023orthogonal} for the OR loss. The corresponding verifications for $\log OR$, $RR$, and $\log RR$ are given in Section~\ref{sec:supp_thm2_partiii}.

\noindent\textbf{Strong convexity (Assumption~4 of \citet{foster2023orthogonal}).}
For any fixed $\eta$, the Hessian of $\mathcal{L}_{OR}(f,\eta)$ with respect to $f$ is
\[
D_f^2\, \mathcal{L}_{OR}(\bar f,\eta)[h,h]
= 2\,E\!\left[(h(X))^2\right]
= 2\,\|h\|_{L_2}^2.
\]
With $\|f\|_\mathcal{F}:=\|f\|_{L_2}$, this gives Assumption~4 of \citet{foster2023orthogonal} with $\lambda=2$ and $\kappa=0$ (the strong convexity holds uniformly in $\eta$, so there is no correction term).

\noindent\textbf{Second-order smoothness (Assumption~3 of \citet{foster2023orthogonal}).}
The second derivative in $f$ satisfies $D_f^2\mathcal{L}_{OR}(\bar f,\eta)[h,h]\le 2\|h\|_{L_2}^2$, giving $\beta_1=2$. For the higher-order smoothness condition on the mixed derivative,
\[
D_\eta^2 D_f\mathcal{L}_{OR}(OR,\bar\eta)[h,k,k]
= -2\,E\!\left[D_\eta^2\varphi_{OR}(Z;\bar\eta)[k,k]\cdot h(X)\right].
\]
We bound this by computing the second derivatives of $\varphi_{OR}(Z;\eta)$ with respect to the nuisance coordinates $(q_1,q_0,\pi)$. Since $\varphi_{OR}(Z;\eta)$ is a rational function of $(q_1,q_0,\pi)$ with each denominator bounded below by $c_0$, the second partial derivatives are uniformly bounded. Specifically, denoting the nuisance direction $k=(\delta_{q_1},\delta_{q_0},\delta_\pi)$ and applying H\"older's inequality:
\[
\left|D_\eta^2 D_f\mathcal{L}_{OR}(OR,\bar\eta)[h,k,k]\right|
\le \beta_{2,OR}\,\|h\|_{L_2}\,\|k\|_\mathcal{G}^2,
\]
for some finite constant $\beta_{2,OR}=\beta_{2,OR}(c_0)$.

\noindent\textbf{Lipschitz condition.}
Under the boundary condition in Theorem~\ref{thm:or_orthogonal}, the pseudo-outcome is uniformly bounded. The OR plug-in satisfies $|OR_\eta(X)|\le (1-c_0)^2/c_0^2$, and the two AIPW corrections satisfy
\[
\frac{OR_\eta(X)}{q_1(X)\{1-q_1(X)\}}\frac{1}{\pi(X)}
\le \frac{1-c_0}{c_0^4},
\qquad
\frac{OR_\eta(X)}{q_0(X)\{1-q_0(X)\}}\frac{1}{1-\pi(X)}
\le \frac{1-c_0}{c_0^4}.
\]
Therefore
\[
B_{\varphi}:=\sup_{\eta\in\mathcal{G},\,Z}|\varphi_{OR}(Z;\eta)|
\le \frac{(1-c_0)^2}{c_0^2}+\frac{2(1-c_0)}{c_0^4}<\infty.
\]
For the pointwise squared loss underlying $\mathcal{L}_{OR}(f,\eta)$, namely $\{\varphi_{OR}(Z;\eta)-f(X)\}^2$, we have for any $f_1,f_2\in\mathcal{F}$,
\begin{align*}
\left|\{\varphi_{OR}(Z;\eta)-f_1(X)\}^2-\{\varphi_{OR}(Z;\eta)-f_2(X)\}^2\right|
&=|f_1(X)-f_2(X)|\,|2\varphi_{OR}(Z;\eta)-f_1(X)-f_2(X)|\\
&\le 2(B_\varphi+M)\,|f_1(X)-f_2(X)|,
\end{align*}
so the pointwise squared loss is Lipschitz in $f$ with constant $L=2(B_\varphi+M)$.

\noindent\textbf{Extension to $\log OR$, $RR$, and $\log RR$.}
Section~\ref{sec:supp_thm2_partiii} verifies the remaining Assumption~3(b) condition of \citet{foster2023orthogonal} and the corresponding Lipschitz conditions for the pointwise squared losses underlying $\mathcal{L}_{\log OR}$, $\mathcal{L}_{RR}$, and $\mathcal{L}_{\log RR}$, with parameter-specific constants $\beta_{2,\theta}=\beta_{2,\theta}(c_0)$ and $B_{\varphi,\theta}=B_{\varphi,\theta}(c_0)$.

\noindent\textbf{Applying Theorem~3 of \citet{foster2023orthogonal}.}
Parts~(i) and (ii) verify Assumptions~1 and~2 of \citet{foster2023orthogonal} (orthogonality and first-order optimality). The strong convexity and smoothness bounds verify Assumptions~3 and~4 of \citet{foster2023orthogonal}. Therefore Theorem~3 of \citet{foster2023orthogonal} applies with $r=0$, yielding: with probability at least $1-\delta$,
\begin{align*}
\mathcal{L}_{OR}(\widehat{OR},\eta_0) - \mathcal{L}_{OR}(OR,\eta_0)
&\le \underbrace{C_1\!\left(\frac{\delta_n^2}{M^2}+\frac{\log(\delta^{-1})}{n}\right)}_{\text{oracle term}}
+\underbrace{C_2\,\|\hat\eta-\eta_0\|_\mathcal{G}^4}_{\text{nuisance term}},
\end{align*}
and by part~(i), the left-hand side equals $E[(\widehat{OR}(X)-OR(X))^2]$. Replacing the OR-specific constants by their $\theta$-specific analogues yields the displayed bound for each $\theta\in\{OR,\log OR,RR,\log RR\}$.
\end{proof}

\subsubsection{Verification of Theorem~\ref{thm:or_orthogonal}(iii) for $\log OR$, $RR$, and $\log RR$}
\label{sec:supp_thm2_partiii}

This subsection verifies the parameter-specific boundedness and higher-order smoothness conditions used in part~(iii) of Theorem~\ref{thm:or_orthogonal}. As in the OR verification in Section~\ref{sec:supp_thm2_proof}, the squared-loss structure alone implies Assumption~4 of \citet{foster2023orthogonal} with $(\lambda,\kappa)=(2,0)$ and Assumption~3(a) of \citet{foster2023orthogonal} with $\beta_1=2$. Thus it remains only to verify Assumption~3(b) of \citet{foster2023orthogonal} and the Lipschitz condition required by Theorem~3 of \citet{foster2023orthogonal}.

\begin{proposition}[Parameter-specific verification for Theorem~\ref{thm:or_orthogonal}(iii)]
\label{prop:supp_thm2_partiii}
Under the assumptions of Theorem~\ref{thm:or_orthogonal}, for each $\theta\in\{\log OR,RR,\log RR\}$ there exist finite constants $\beta_{2,\theta}=\beta_{2,\theta}(c_0)$ and $B_{\varphi,\theta}=B_{\varphi,\theta}(c_0)$ such that
\[
\left|D_\eta^2D_f\mathcal{L}_\theta(\theta,\bar\eta)[h,k,k]\right|
\le
\beta_{2,\theta}\,\|h\|_{L_2}\,\|k\|_{\mathcal{G}}^2
\]
for all nuisance values $\bar\eta=(\bar q_1,\bar q_0,\bar\pi)$ with coordinates in $[c_0,1-c_0]^3$, and the pointwise squared loss underlying $\mathcal{L}_\theta(f,\eta)$, namely $\{\varphi_\theta(Z;\eta)-f(X)\}^2$, is Lipschitz in $f$ with constant $L_\theta=2(B_{\varphi,\theta}+M)$.
\end{proposition}

\begin{proof}
We treat the three remaining parameters in turn.

\noindent\textbf{Case $\log OR$.}
Under the boundary condition in Theorem~\ref{thm:or_orthogonal},
\[
|\log OR_\eta(X)|
\le
2\log\!\left(\frac{1-c_0}{c_0}\right),
\qquad
\frac{1}{q_t(X)\{1-q_t(X)\}}
\le \frac{1}{c_0^2},
\qquad
\frac{1}{\pi(X)},\frac{1}{1-\pi(X)}
\le \frac{1}{c_0}.
\]
Since $|Y-q_t(X)|\le 1$, it follows that
\[
|\varphi_{\log OR}(Z;\eta)|
\le
2\log\!\left(\frac{1-c_0}{c_0}\right)+\frac{2}{c_0^3}
=:B_{\varphi,\log OR}<\infty.
\]
Moreover, $\varphi_{\log OR}(Z;\eta)$ is twice continuously differentiable in $(q_1,q_0,\pi)$ on the compact domain $[c_0,1-c_0]^3$. Therefore all second partial derivatives are uniformly bounded, and there exists $\beta_{2,\log OR}=\beta_{2,\log OR}(c_0)$ such that
\[
\left|D_\eta^2D_f\mathcal{L}_{\log OR}(\log OR,\bar\eta)[h,k,k]\right|
\le
\beta_{2,\log OR}\,\|h\|_{L_2}\,\|k\|_{\mathcal{G}}^2.
\]
The corresponding pointwise loss is Lipschitz in $f$ with constant
\[
L_{\log OR}=2(B_{\varphi,\log OR}+M).
\]

\noindent\textbf{Case $RR$.}
Again under the boundary condition in Theorem~\ref{thm:or_orthogonal},
\[
|RR_\eta(X)|\le \frac{1-c_0}{c_0},
\qquad
\frac{1}{q_0(X)\pi(X)}\le \frac{1}{c_0^2},
\qquad
\frac{q_1(X)}{q_0(X)^2\{1-\pi(X)\}}
\le \frac{1-c_0}{c_0^3}.
\]
Hence
\[
|\varphi_{RR}(Z;\eta)|
\le
\frac{1-c_0}{c_0}+\frac{1}{c_0^2}+\frac{1-c_0}{c_0^3}
=:B_{\varphi,RR}<\infty.
\]
Since $\varphi_{RR}(Z;\eta)$ is a rational function of $(q_1,q_0,\pi)$ with denominators bounded below by $c_0$, all second partial derivatives are uniformly bounded on $[c_0,1-c_0]^3$. Therefore there exists $\beta_{2,RR}=\beta_{2,RR}(c_0)$ such that
\[
\left|D_\eta^2D_f\mathcal{L}_{RR}(RR,\bar\eta)[h,k,k]\right|
\le
\beta_{2,RR}\,\|h\|_{L_2}\,\|k\|_{\mathcal{G}}^2.
\]
The pointwise loss is Lipschitz in $f$ with constant
\[
L_{RR}=2(B_{\varphi,RR}+M).
\]

\noindent\textbf{Case $\log RR$.}
Under the same boundary condition,
\[
|\log RR_\eta(X)|
\le
\log\!\left(\frac{1-c_0}{c_0}\right),
\qquad
\frac{1}{q_1(X)\pi(X)},\frac{1}{q_0(X)\{1-\pi(X)\}}
\le \frac{1}{c_0^2}.
\]
Hence
\[
|\varphi_{\log RR}(Z;\eta)|
\le
\log\!\left(\frac{1-c_0}{c_0}\right)+\frac{2}{c_0^2}
=:B_{\varphi,\log RR}<\infty.
\]
Because $\varphi_{\log RR}(Z;\eta)$ is twice continuously differentiable on $[c_0,1-c_0]^3$, all second partial derivatives are uniformly bounded there. Thus there exists $\beta_{2,\log RR}=\beta_{2,\log RR}(c_0)$ such that
\[
\left|D_\eta^2D_f\mathcal{L}_{\log RR}(\log RR,\bar\eta)[h,k,k]\right|
\le
\beta_{2,\log RR}\,\|h\|_{L_2}\,\|k\|_{\mathcal{G}}^2.
\]
The pointwise loss is Lipschitz in $f$ with constant
\[
L_{\log RR}=2(B_{\varphi,\log RR}+M).
\]

Combining these verifications with parts~(i) and~(ii) of Theorem~\ref{thm:or_orthogonal} and with Theorem~3 of \citet{foster2023orthogonal} yields part~(iii) for $\log OR$, $RR$, and $\log RR$.
\end{proof}

\subsubsection{Remark on the weighted-loss extension to the R-learner}

Because the R-learner uses the same pseudo-outcomes as the DR-learner, Theorem~\ref{thm:dr_transforms} applies directly to the R-learner pseudo-outcomes. A weighted analogue of Theorem~\ref{thm:or_orthogonal} would replace the unweighted loss
\[
\mathcal{L}_\theta(f,\eta)=E\!\left[(\varphi_\theta(Z;\eta)-f(X))^2\right]
\]
by the propensity-overlap-weighted loss
\[
\mathcal{L}^{R}_\theta(f,\eta)=E\!\left[\pi(X)\{1-\pi(X)\}(\varphi_\theta(Z;\eta)-f(X))^2\right],
\]
and then verify the corresponding weighted orthogonality and regularity conditions, as in the weighted orthogonal learner framework of \citet{morzywolek2025weighted}. This gives the natural weighted-loss analogue of Theorem~\ref{thm:or_orthogonal} for the R-learner.

\subsection{Plots}






\begin{figure}[H]
\centering
\includegraphics[width=1\textwidth]{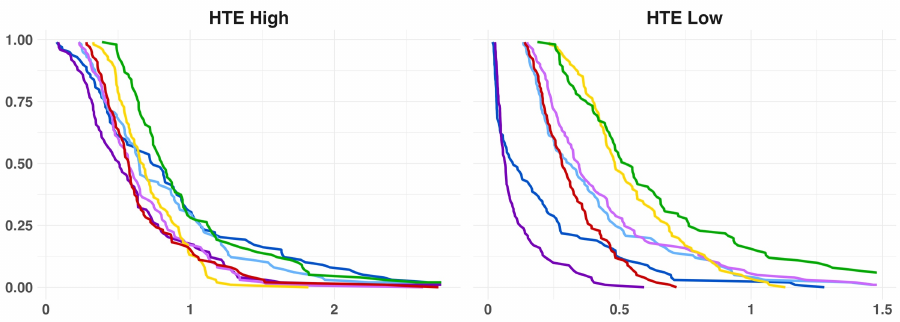}
\caption*{\footnotesize (a) iMSE}

\includegraphics[width=1\textwidth]{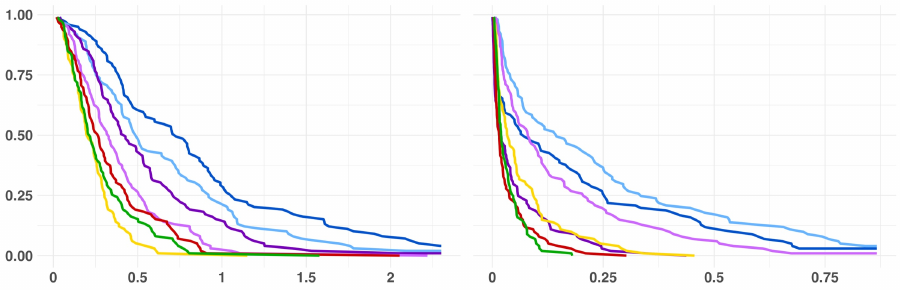}
\caption*{\footnotesize (b) iBias$^2$}

\includegraphics[width=1\textwidth]{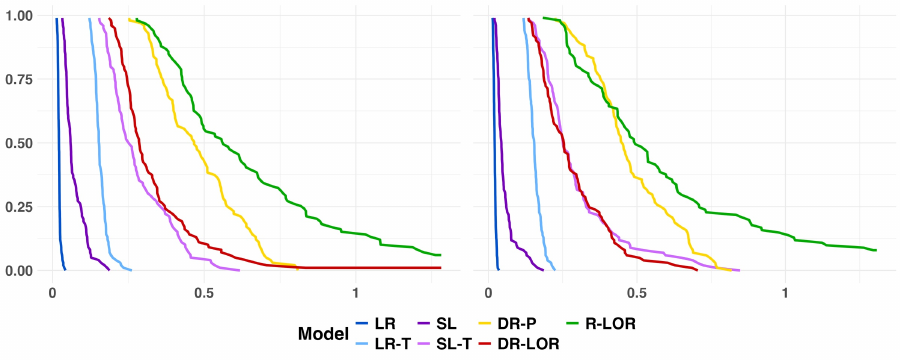}
\caption*{\footnotesize (c) iVariance}

\caption{Conditional OR: interaction order 3, sample size 1000}
\label{fig:COR_inter3_1000}
\end{figure}

\begin{figure}[H]
\centering
\includegraphics[width=1\textwidth]{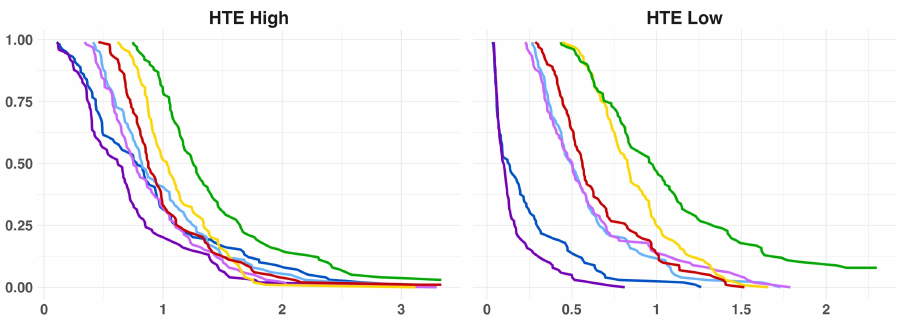}
\caption*{\footnotesize (a) iMSE}

\includegraphics[width=1\textwidth]{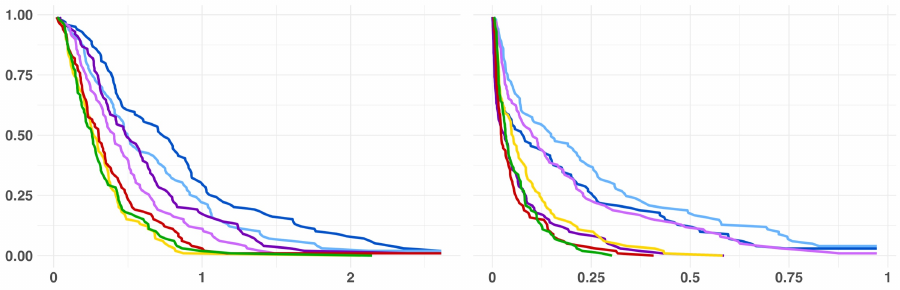}
\caption*{\footnotesize (b) iBias$^2$}

\includegraphics[width=1\textwidth]{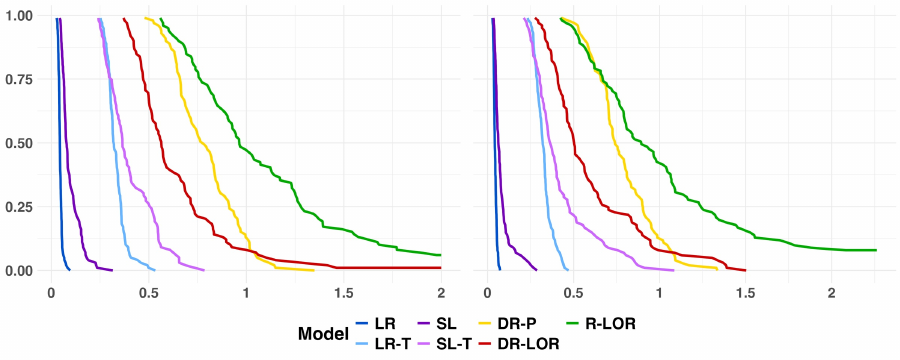}
\caption*{\footnotesize (c) iVariance}

\caption{Conditional OR: interaction order 3, sample size 500}
\label{fig:COR_inter3_500}
\end{figure}





\begin{figure}[H]
\centering
\includegraphics[width=1\textwidth]{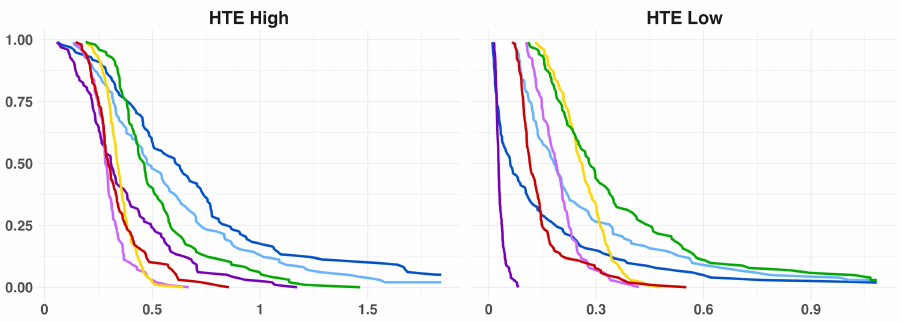}
\caption*{\footnotesize (a) iMSE}

\includegraphics[width=1\textwidth]{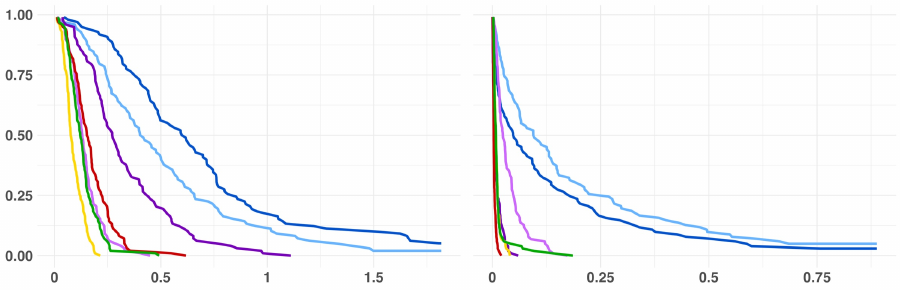}
\caption*{\footnotesize (b) iBias$^2$}

\includegraphics[width=1\textwidth]{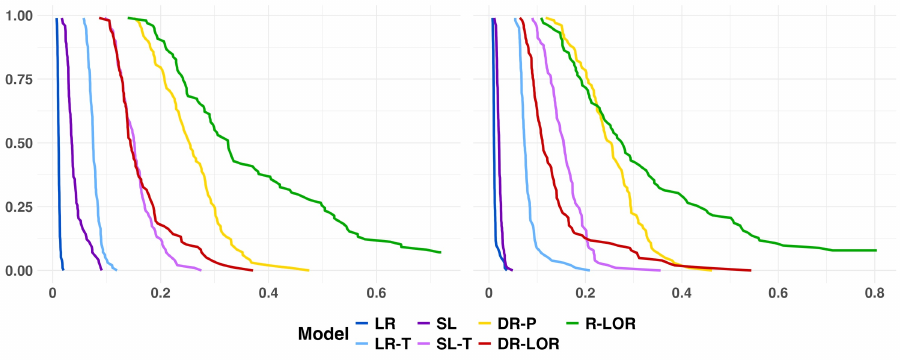}
\caption*{\footnotesize (c) iVariance}

\caption{Conditional OR: interaction order 2, sample size 2000}
\label{fig:COR_inter2_2000}
\end{figure}

\begin{figure}[H]
\centering
\includegraphics[width=1\textwidth]{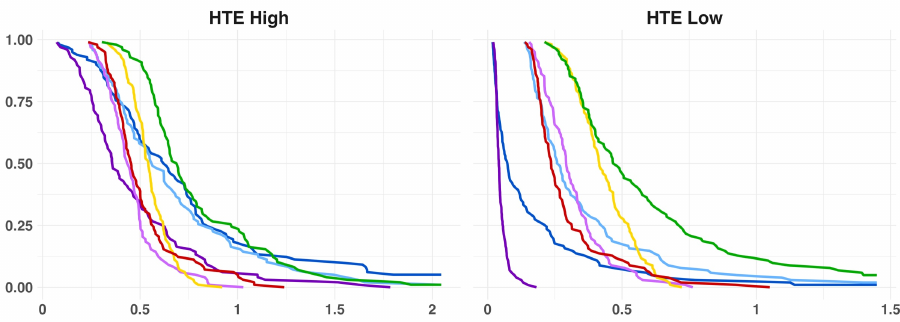}
\caption*{\footnotesize (a) iMSE}

\includegraphics[width=1\textwidth]{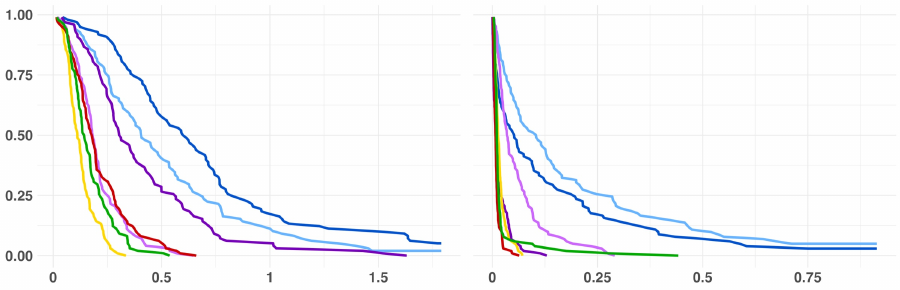}
\caption*{\footnotesize (b) iBias$^2$}

\includegraphics[width=1\textwidth]{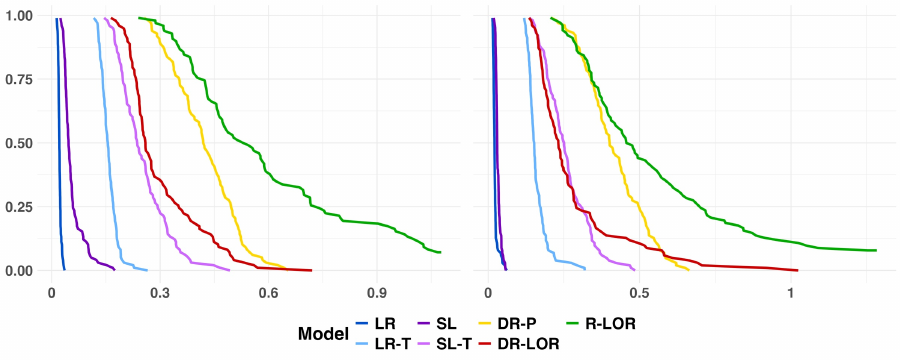}
\caption*{\footnotesize (c) iVariance}

\caption{Conditional OR: interaction order 2, sample size 1000}
\label{fig:COR_inter2_1000}
\end{figure}

\begin{figure}[H]
\centering
\includegraphics[width=1\textwidth]{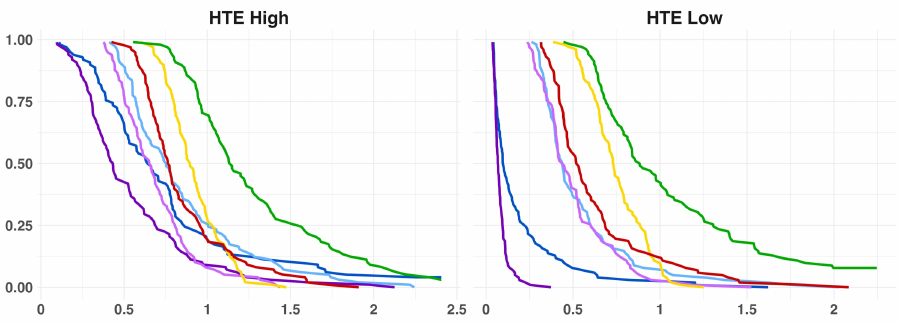}
\caption*{\footnotesize (a) iMSE}

\includegraphics[width=1\textwidth]{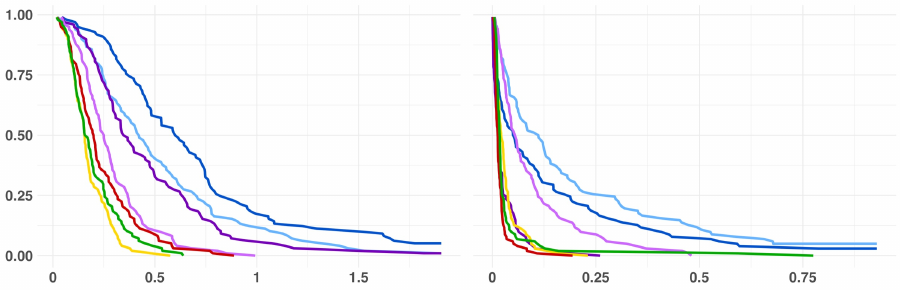}
\caption*{\footnotesize (b) iBias$^2$}

\includegraphics[width=1\textwidth]{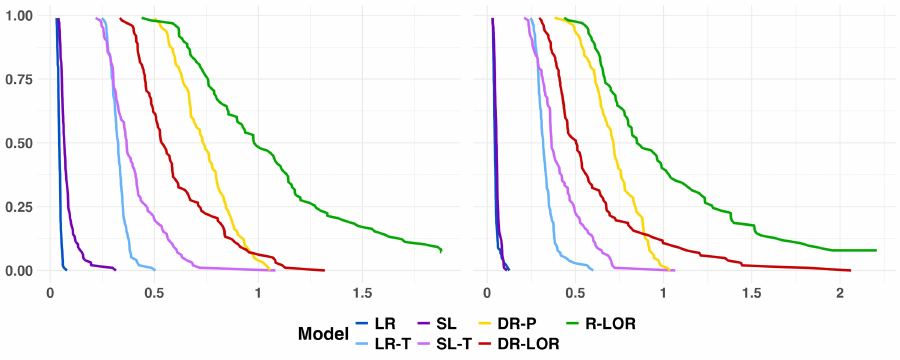}
\caption*{\footnotesize (c) iVariance}

\caption{Conditional OR: interaction order 2, sample size 500}
\label{fig:COR_inter2_500}
\end{figure}

\begin{figure}[H]
\centering
\includegraphics[width=1\textwidth]{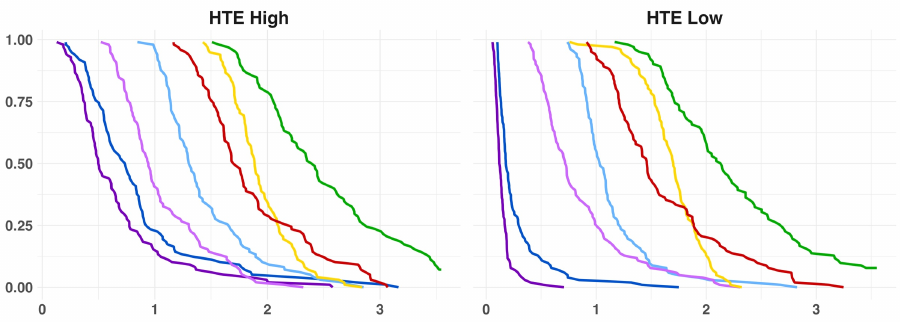}
\caption*{\footnotesize (a) iMSE}

\includegraphics[width=1\textwidth]{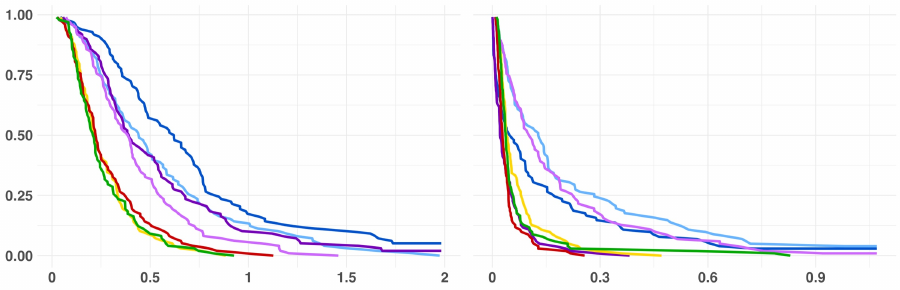}
\caption*{\footnotesize (b) iBias$^2$}

\includegraphics[width=1\textwidth]{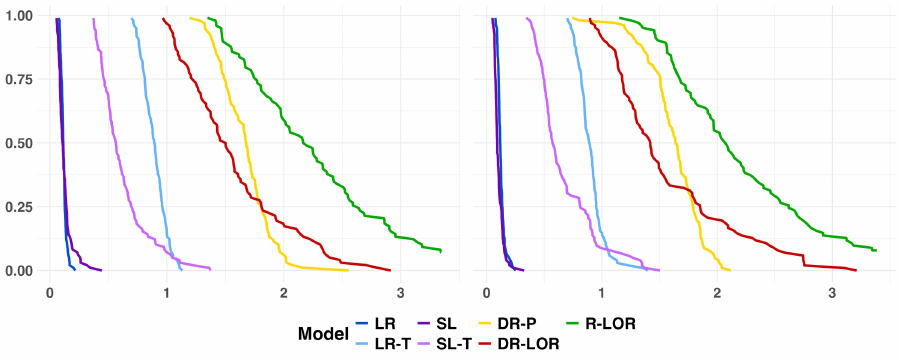}
\caption*{\footnotesize (c) iVariance}

\caption{Conditional OR: interaction order 2, sample size 200}
\label{fig:COR_inter2_200}
\end{figure}





\begin{figure}[H]
\centering
\includegraphics[width=1\textwidth]{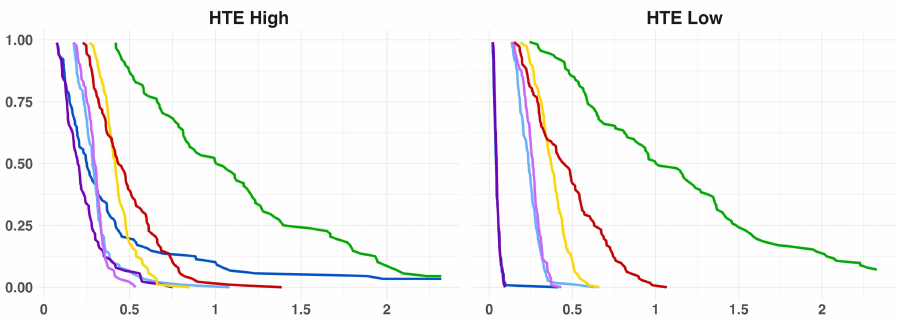}
\caption*{\footnotesize (a) iMSE}

\includegraphics[width=1\textwidth]{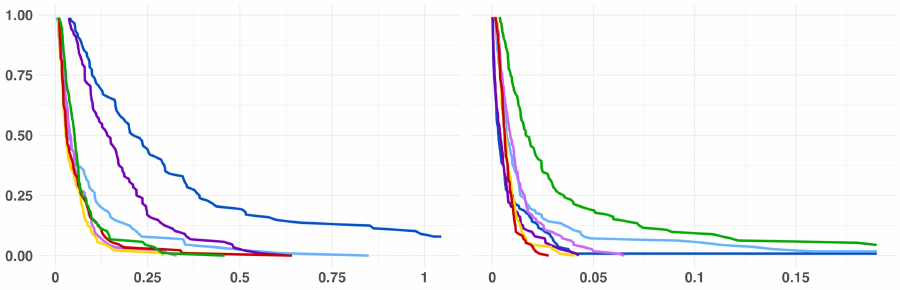}
\caption*{\footnotesize (b) iBias$^2$}

\includegraphics[width=1\textwidth]{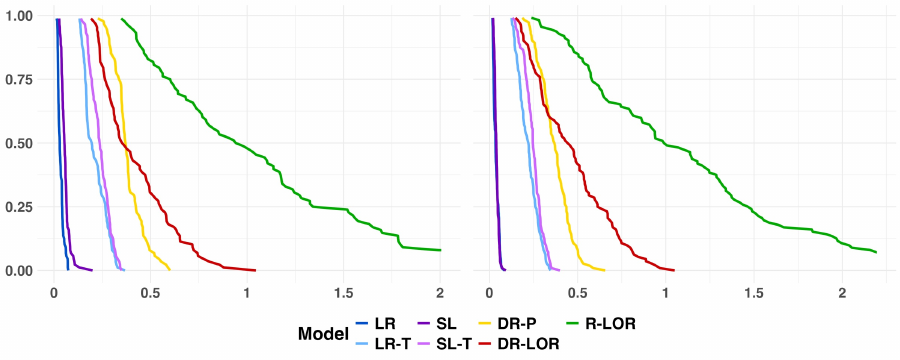}
\caption*{\footnotesize (c) iVariance}

\caption{Conditional OR: interaction order 1, sample size 1000}
\label{fig:COR_inter1_1000}
\end{figure}

\begin{figure}[H]
\centering
\includegraphics[width=1\textwidth]{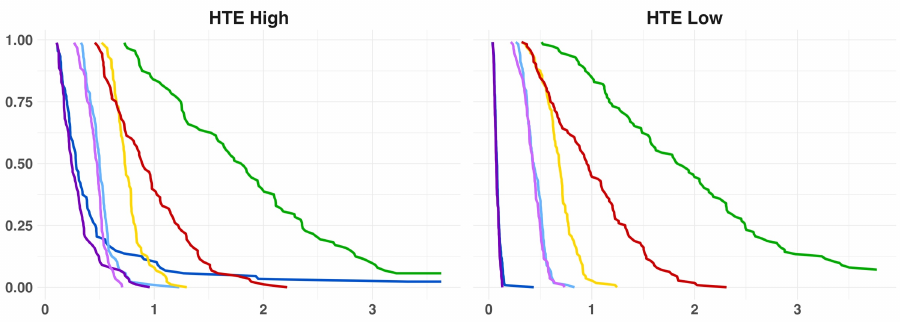}
\caption*{\footnotesize (a) iMSE}

\includegraphics[width=1\textwidth]{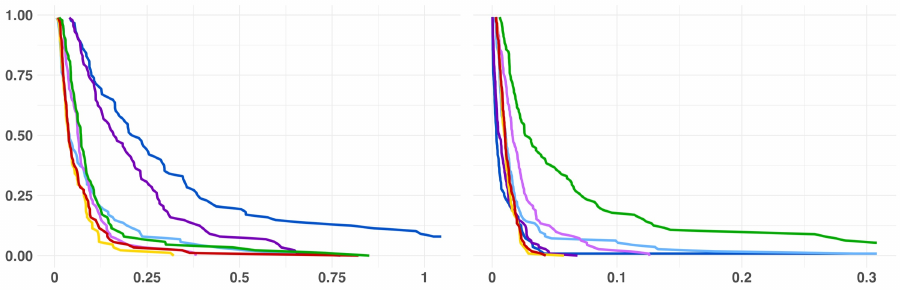}
\caption*{\footnotesize (b) iBias$^2$}

\includegraphics[width=1\textwidth]{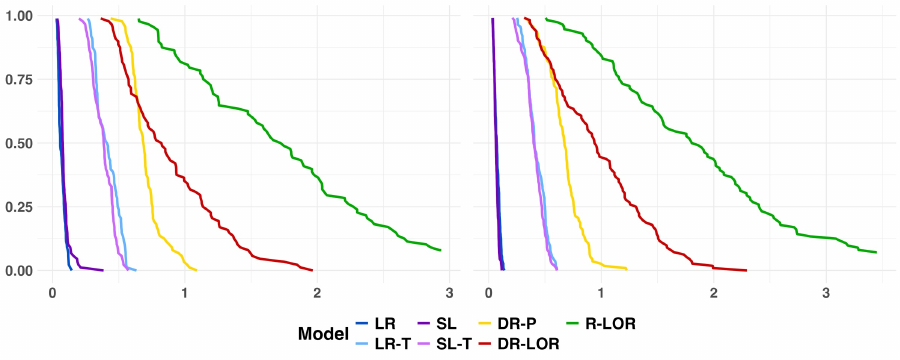}
\caption*{\footnotesize (c) iVariance}

\caption{Conditional OR: interaction order 1, sample size 500}
\label{fig:COR_inter1_500}
\end{figure}






\begin{figure}[H]
\centering
\includegraphics[width=1\textwidth]{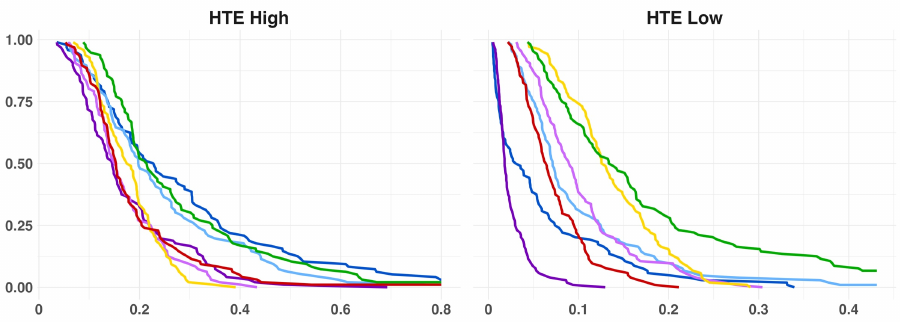}
\caption*{\footnotesize (a) iMSE}

\includegraphics[width=1\textwidth]{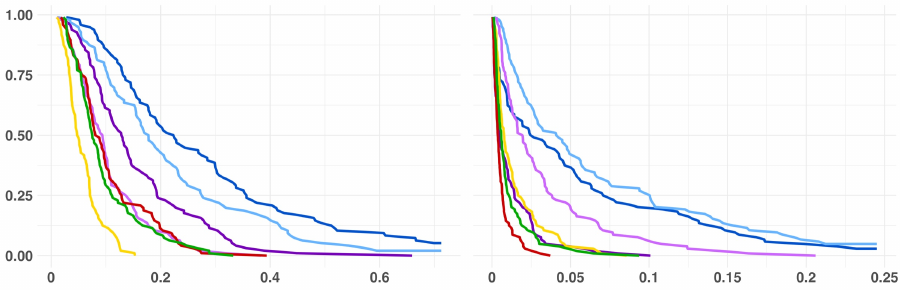}
\caption*{\footnotesize (b) iBias$^2$}

\includegraphics[width=1\textwidth]{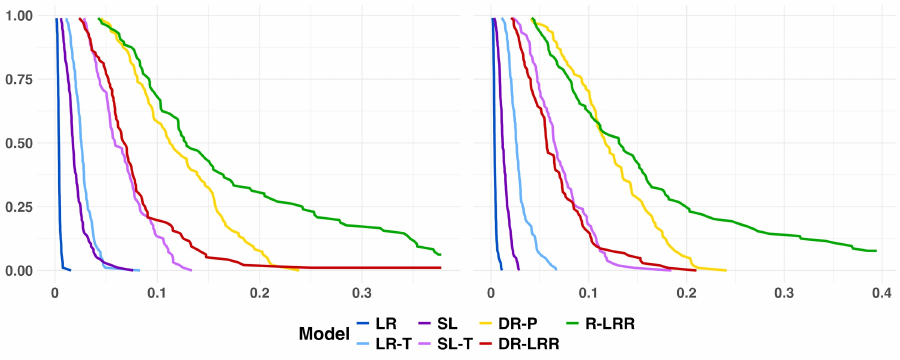}
\caption*{\footnotesize (c) iVariance}

\caption{Conditional RR: interaction order 3, sample size 2000}
\label{fig:CRR_inter3_2000}
\end{figure}

\begin{figure}[H]
\centering
\includegraphics[width=1\textwidth]{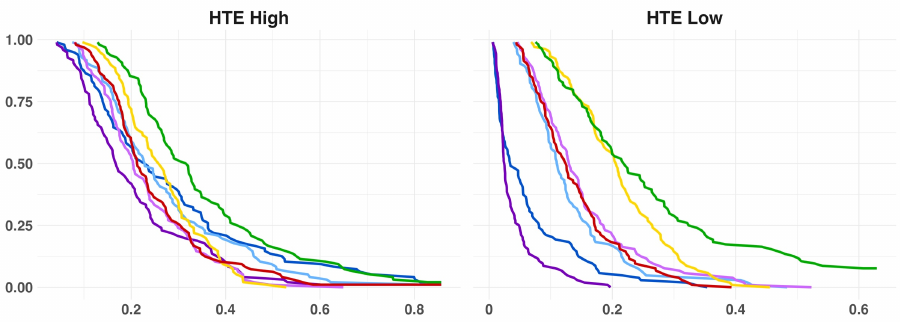}
\caption*{\footnotesize (a) iMSE}

\includegraphics[width=1\textwidth]{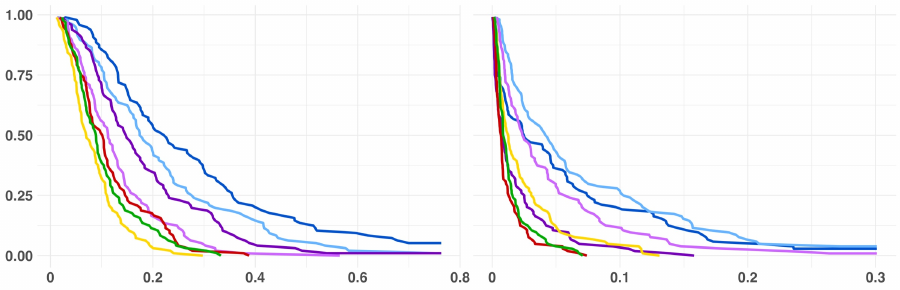}
\caption*{\footnotesize (b) iBias$^2$}

\includegraphics[width=1\textwidth]{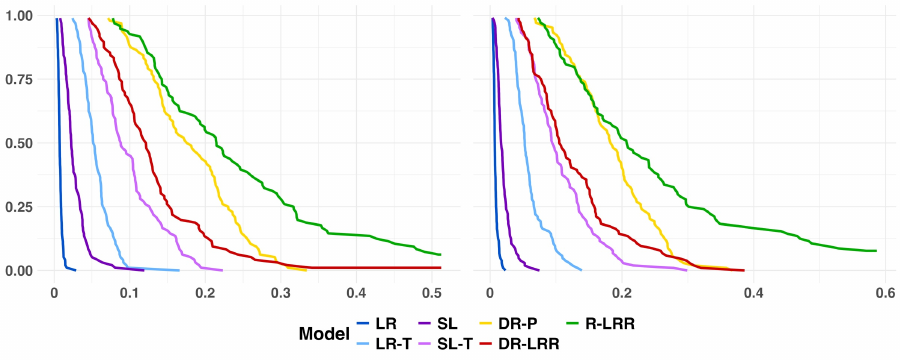}
\caption*{\footnotesize (c) iVariance}

\caption{Conditional RR: interaction order 3, sample size 1000}
\label{fig:CRR_inter3_1000}
\end{figure}

\begin{figure}[H]
\centering
\includegraphics[width=1\textwidth]{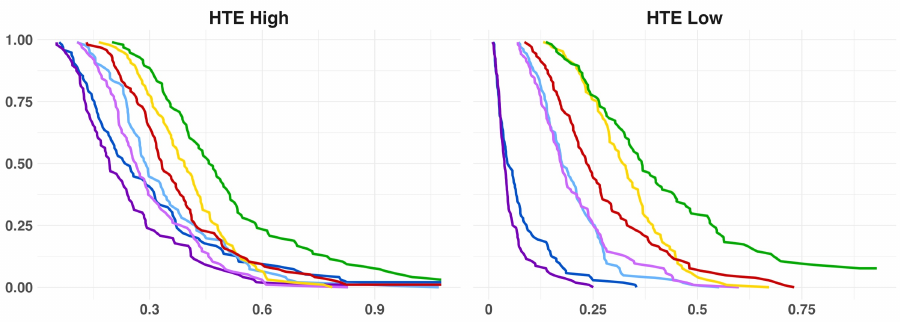}
\caption*{\footnotesize (a) iMSE}

\includegraphics[width=1\textwidth]{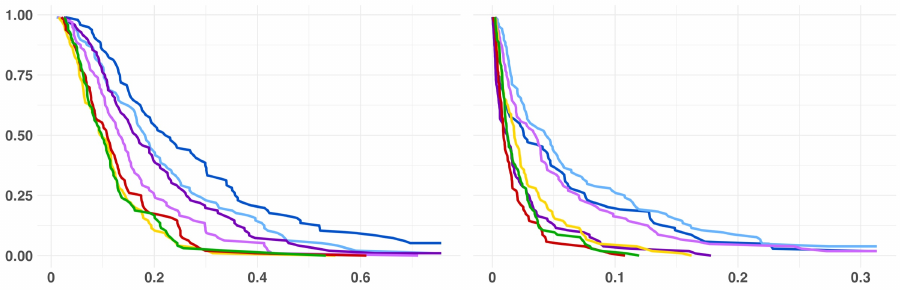}
\caption*{\footnotesize (b) iBias$^2$}

\includegraphics[width=1\textwidth]{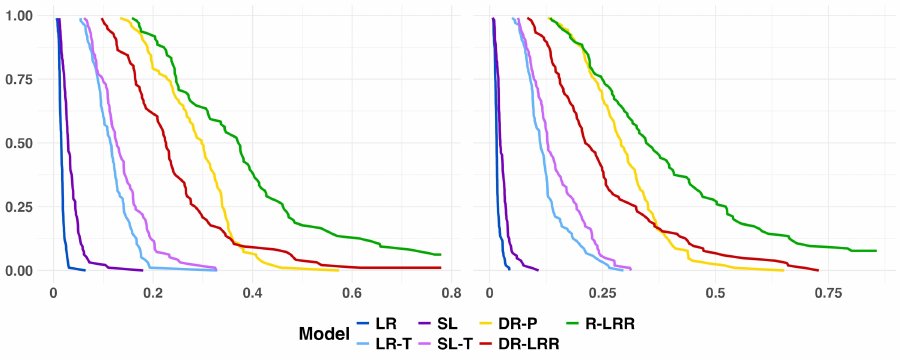}
\caption*{\footnotesize (c) iVariance}

\caption{Conditional RR: interaction order 3, sample size 500}
\label{fig:CRR_inter3_500}
\end{figure}

\begin{figure}[H]
\centering
\includegraphics[width=1\textwidth]{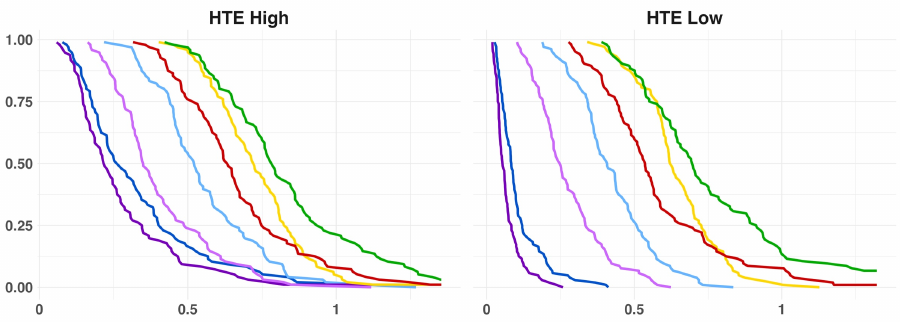}
\caption*{\footnotesize (a) iMSE}

\includegraphics[width=1\textwidth]{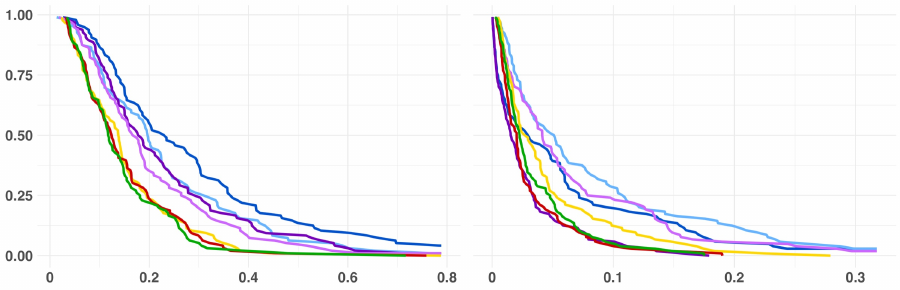}
\caption*{\footnotesize (b) iBias$^2$}

\includegraphics[width=1\textwidth]{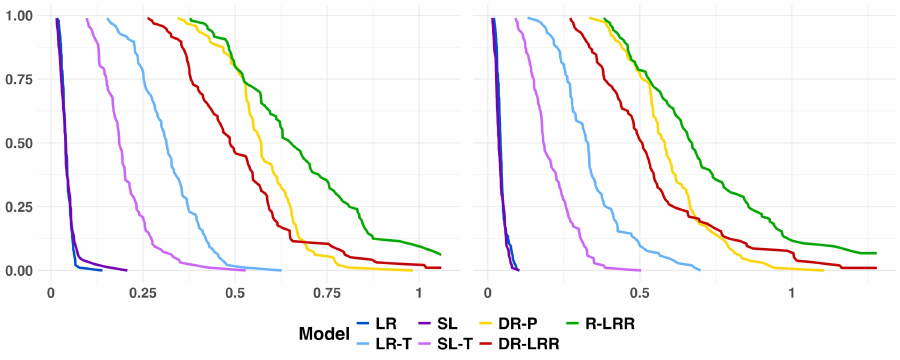}
\caption*{\footnotesize (c) iVariance}

\caption{Conditional RR: interaction order 3, sample size 200}
\label{fig:CRR_inter3_200}
\end{figure}

\begin{figure}[H]
\centering
\includegraphics[width=1\textwidth]{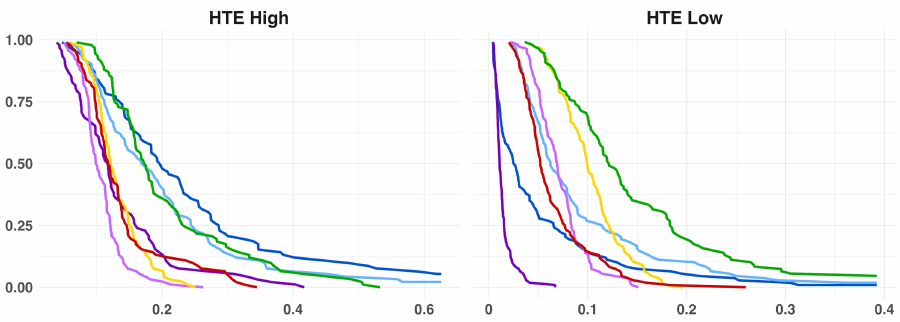}
\caption*{\footnotesize (a) iMSE}

\includegraphics[width=1\textwidth]{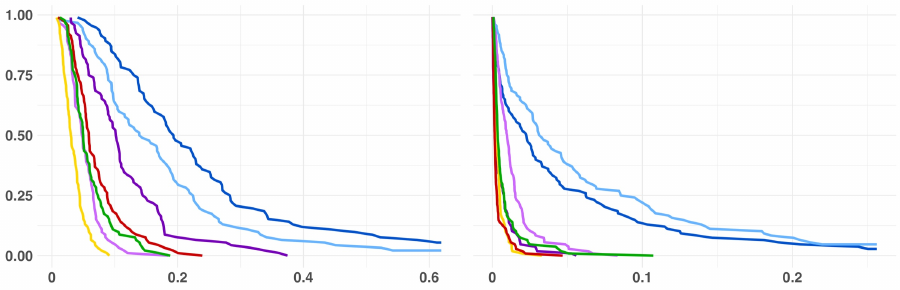}
\caption*{\footnotesize (b) iBias$^2$}

\includegraphics[width=1\textwidth]{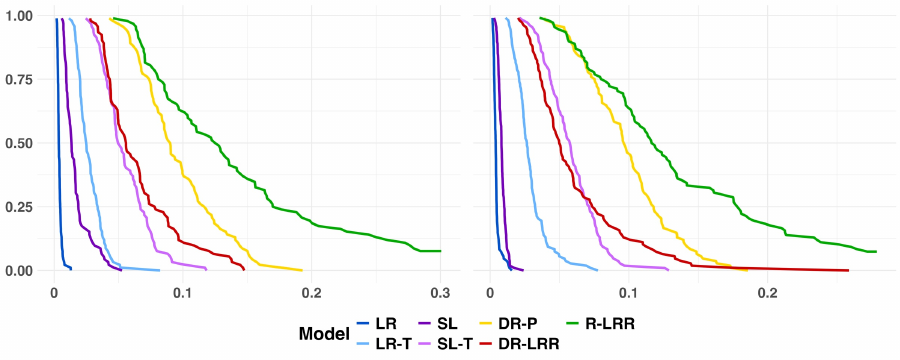}
\caption*{\footnotesize (c) iVariance}

\caption{Conditional RR: interaction order 2, sample size 2000}
\label{fig:CRR_inter2_2000}
\end{figure}

\begin{figure}[H]
\centering
\includegraphics[width=1\textwidth]{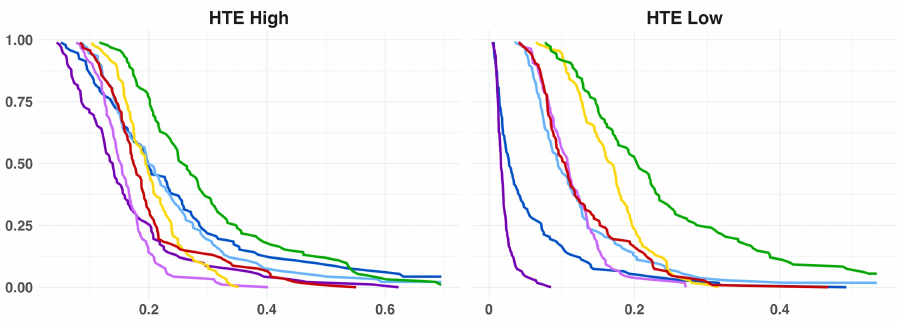}
\caption*{\footnotesize (a) iMSE}

\includegraphics[width=1\textwidth]{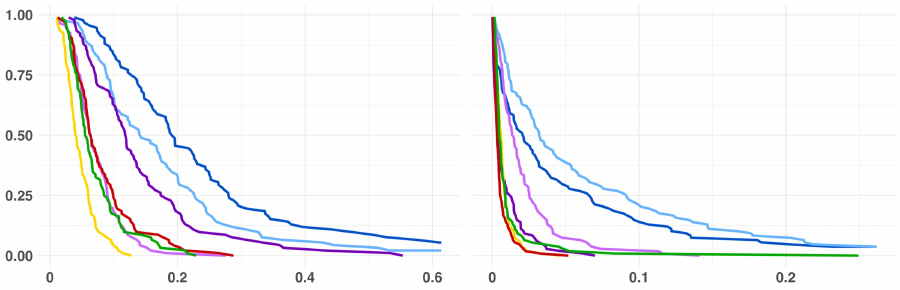}
\caption*{\footnotesize (b) iBias$^2$}

\includegraphics[width=1\textwidth]{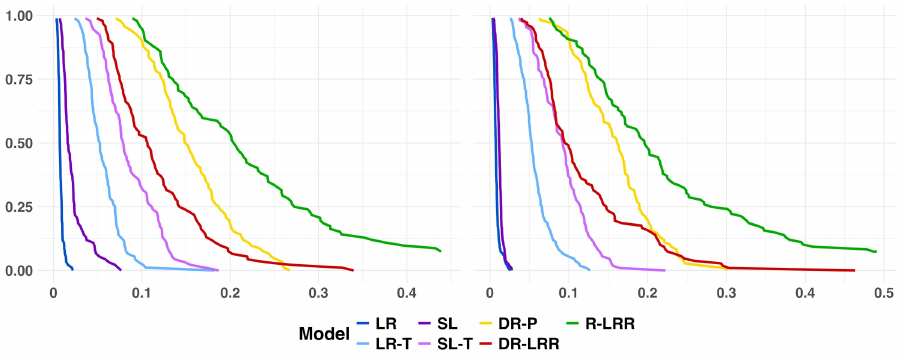}
\caption*{\footnotesize (c) iVariance}

\caption{Conditional RR: interaction order 2, sample size 1000}
\label{fig:CRR_inter2_1000}
\end{figure}

\begin{figure}[H]
\centering
\includegraphics[width=1\textwidth]{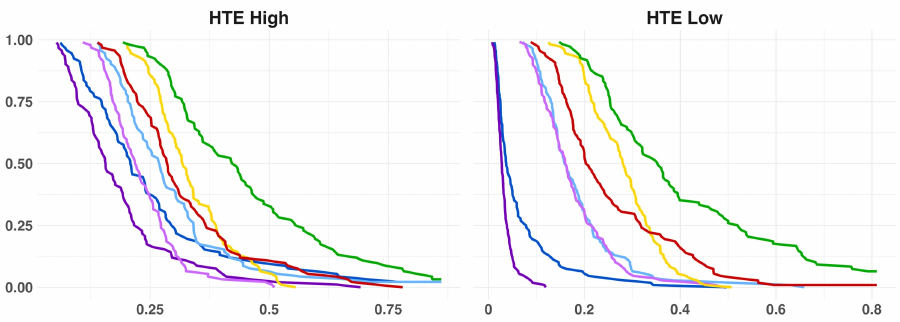}
\caption*{\footnotesize (a) iMSE}

\includegraphics[width=1\textwidth]{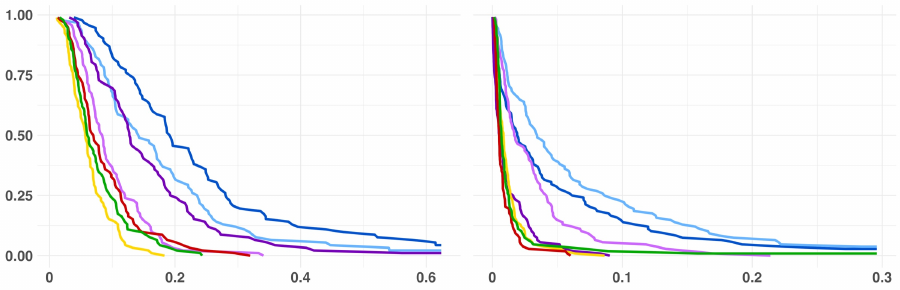}
\caption*{\footnotesize (b) iBias$^2$}

\includegraphics[width=1\textwidth]{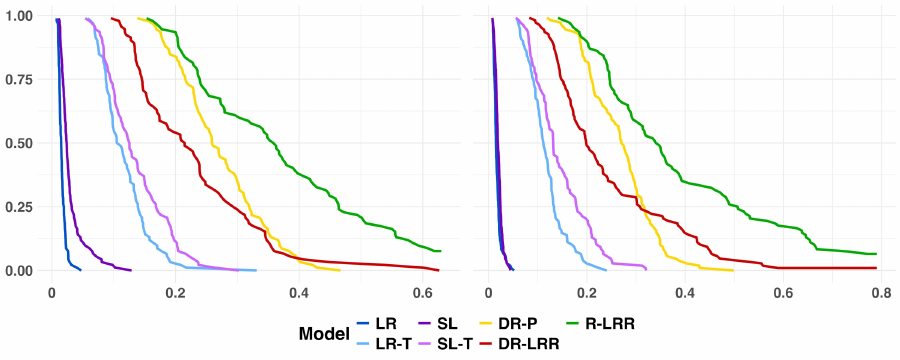}
\caption*{\footnotesize (c) iVariance}

\caption{Conditional RR: interaction order 2, sample size 500}
\label{fig:CRR_inter2_500}
\end{figure}

\begin{figure}[H]
\centering
\includegraphics[width=1\textwidth]{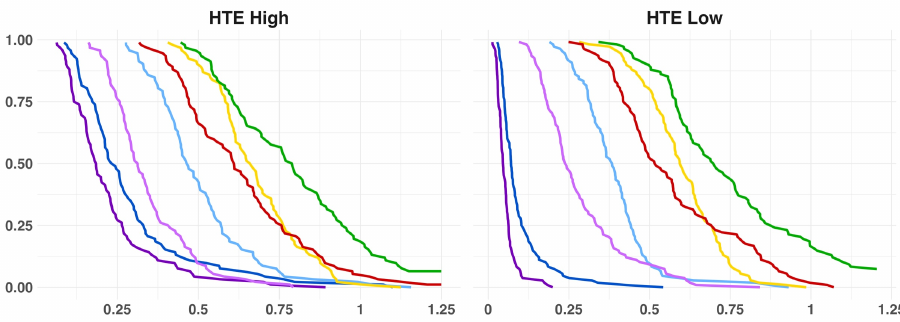}
\caption*{\footnotesize (a) iMSE}

\includegraphics[width=1\textwidth]{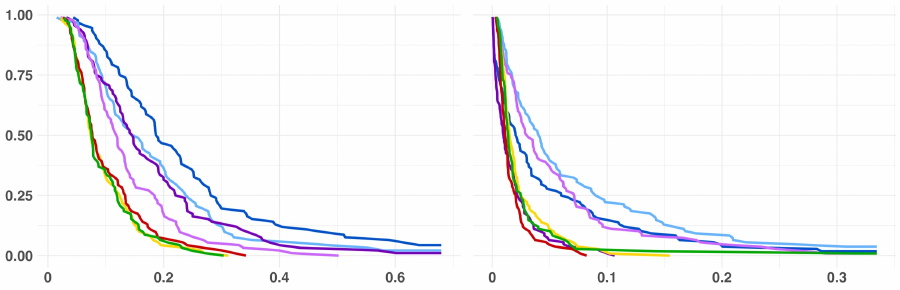}
\caption*{\footnotesize (b) iBias$^2$}

\includegraphics[width=1\textwidth]{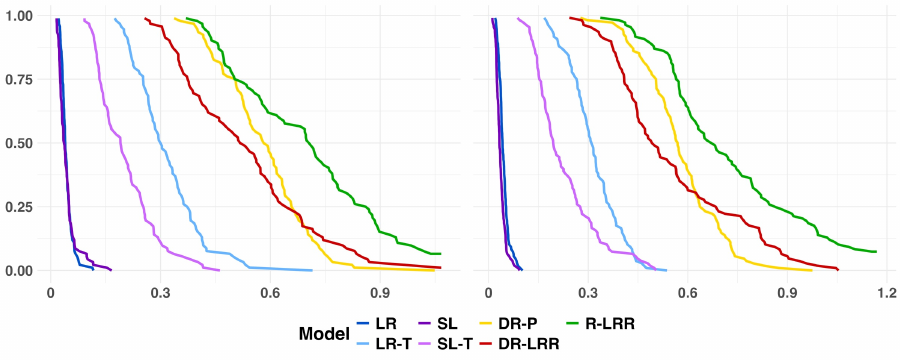}
\caption*{\footnotesize (c) iVariance}

\caption{Conditional RR: interaction order 2, sample size 200}
\label{fig:CRR_inter2_200}
\end{figure}

\begin{figure}[H]
\centering
\includegraphics[width=1\textwidth]{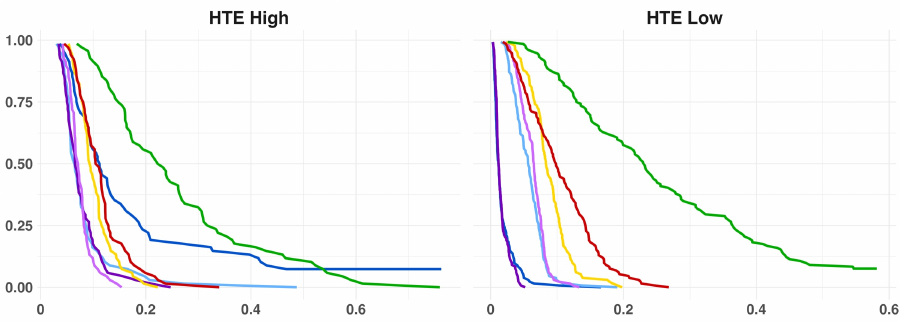}
\caption*{\footnotesize (a) iMSE}

\includegraphics[width=1\textwidth]{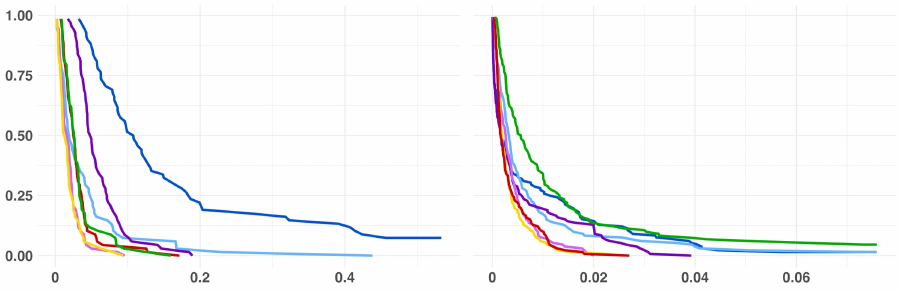}
\caption*{\footnotesize (b) iBias$^2$}

\includegraphics[width=1\textwidth]{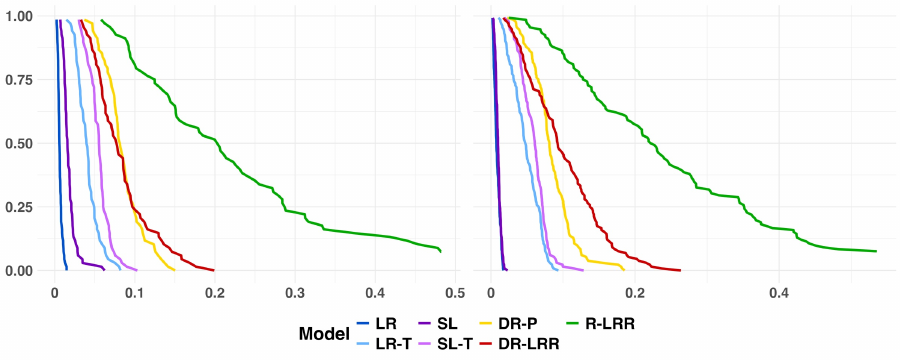}
\caption*{\footnotesize (c) iVariance}

\caption{Conditional RR: interaction order 1, sample size 2000}
\label{fig:CRR_inter1_2000}
\end{figure}

\begin{figure}[H]
\centering
\includegraphics[width=1\textwidth]{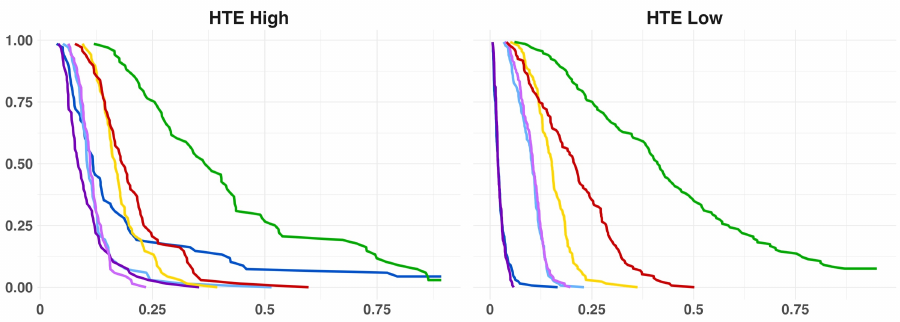}
\caption*{\footnotesize (a) iMSE}

\includegraphics[width=1\textwidth]{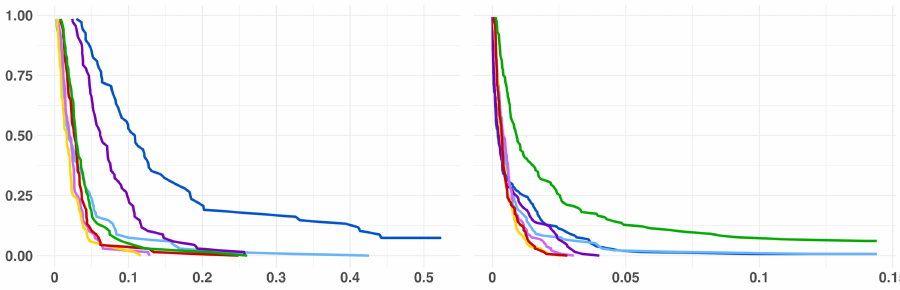}
\caption*{\footnotesize (b) iBias$^2$}

\includegraphics[width=1\textwidth]{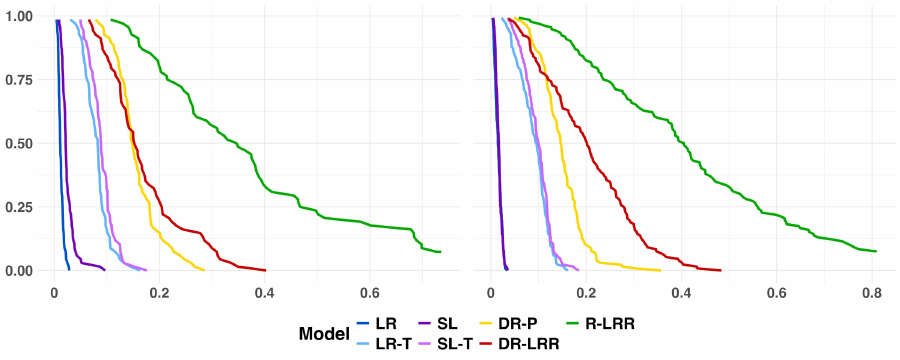}
\caption*{\footnotesize (c) iVariance}

\caption{Conditional RR: interaction order 1, sample size 1000}
\label{fig:CRR_inter1_1000}
\end{figure}

\begin{figure}[H]
\centering
\includegraphics[width=1\textwidth]{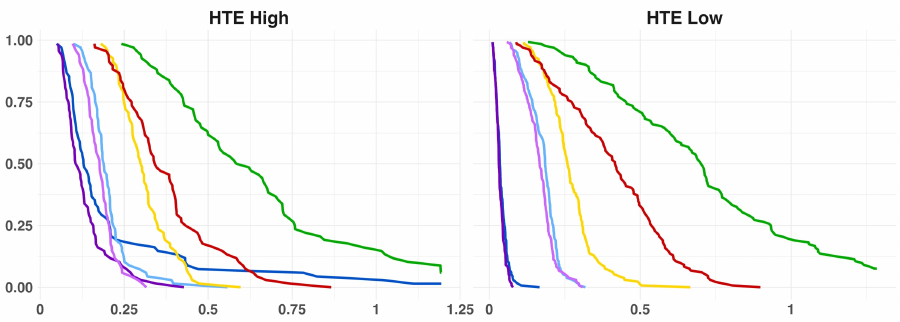}
\caption*{\footnotesize (a) iMSE}

\includegraphics[width=1\textwidth]{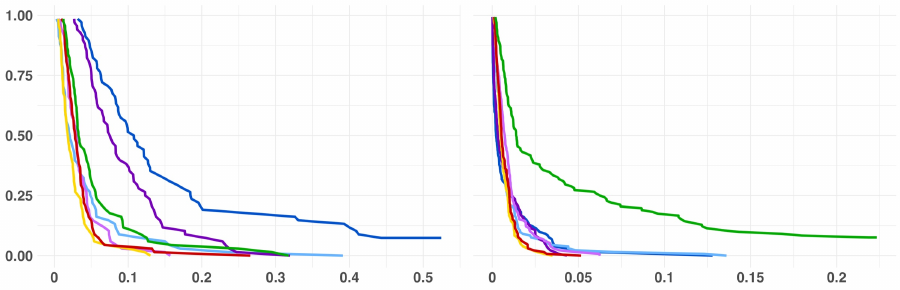}
\caption*{\footnotesize (b) iBias$^2$}

\includegraphics[width=1\textwidth]{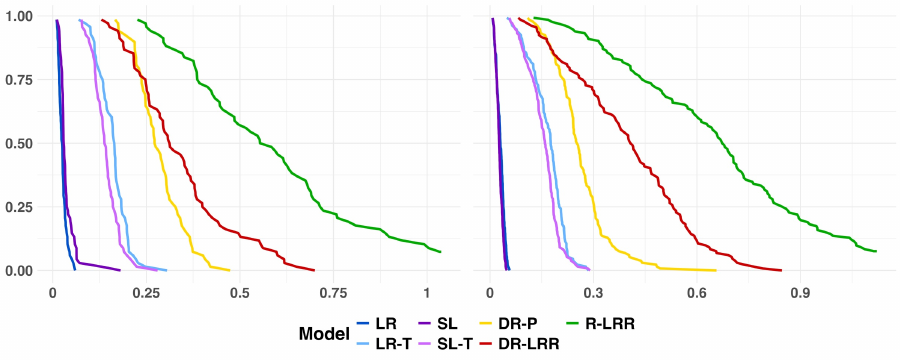}
\caption*{\footnotesize (c) iVariance}

\caption{Conditional RR: interaction order 1, sample size 500}
\label{fig:CRR_inter1_500}
\end{figure}

\begin{figure}[H]
\centering
\includegraphics[width=1\textwidth]{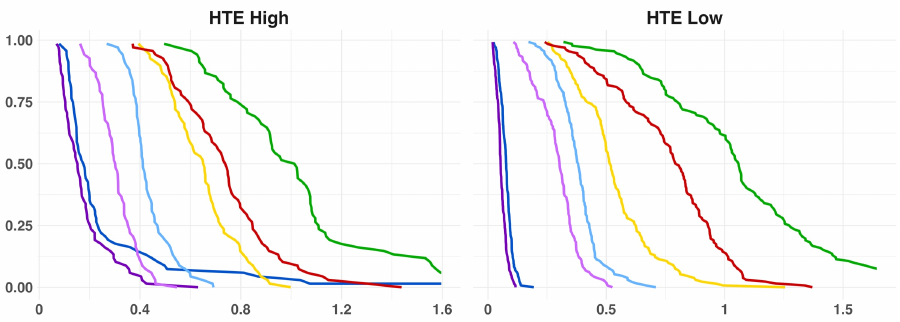}
\caption*{\footnotesize (a) iMSE}

\includegraphics[width=1\textwidth]{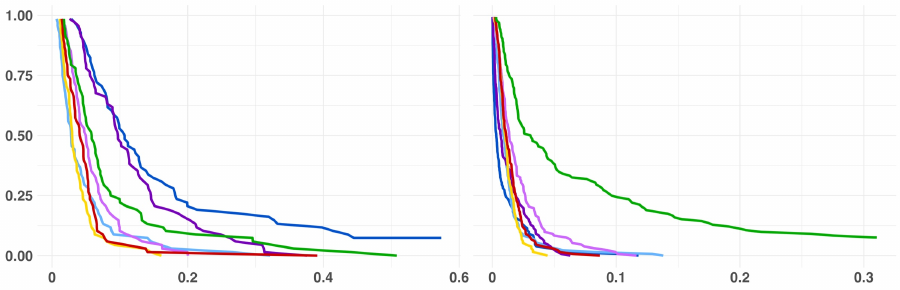}
\caption*{\footnotesize (b) iBias$^2$}

\includegraphics[width=1\textwidth]{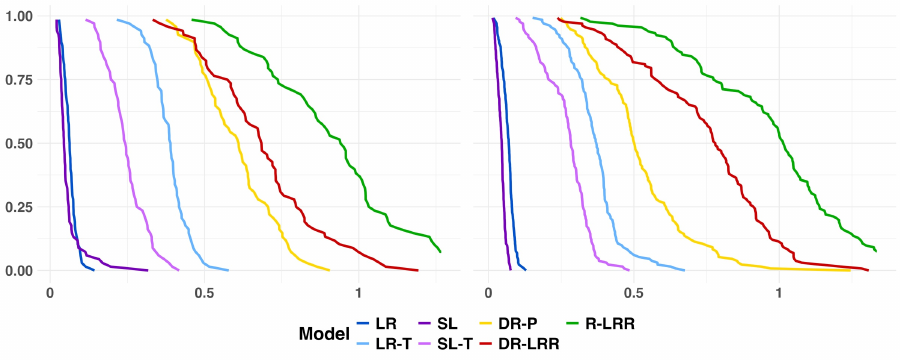}
\caption*{\footnotesize (c) iVariance}

\caption{Conditional RR: interaction order 1, sample size 200}
\label{fig:CRR_inter1_200}
\end{figure}


\begin{figure}[H]
\centering
\includegraphics[width=1\textwidth]{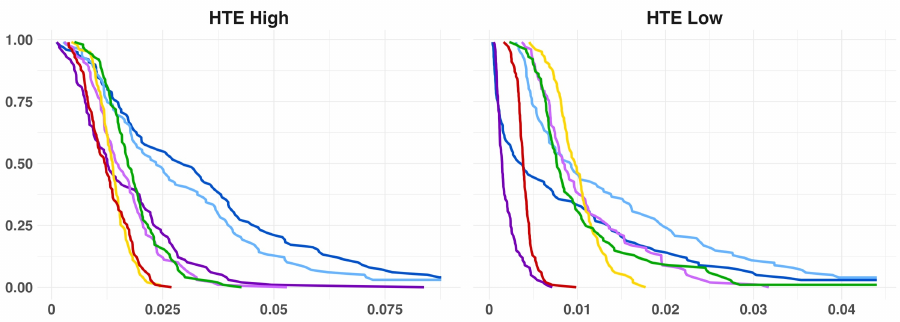}
\caption*{\footnotesize (a) iMSE}

\includegraphics[width=1\textwidth]{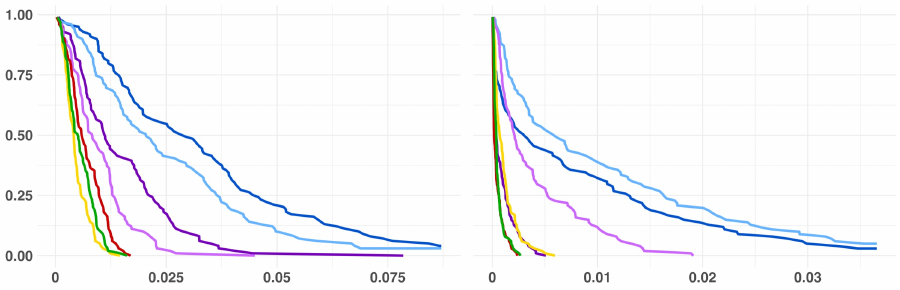}
\caption*{\footnotesize (b) iBias$^2$}

\includegraphics[width=1\textwidth]{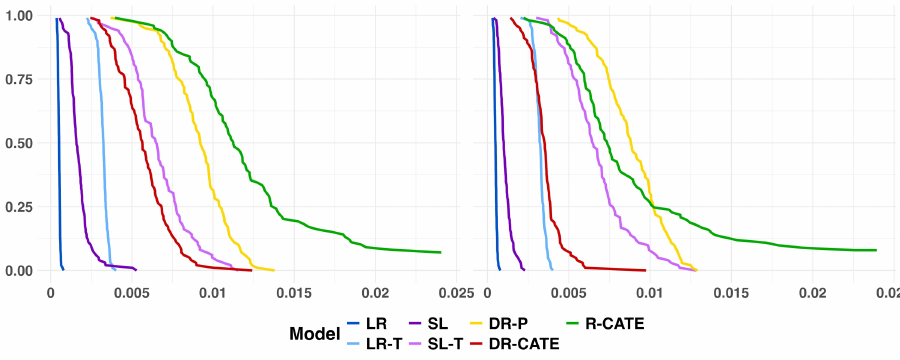}
\caption*{\footnotesize (c) iVariance}

\caption{Conditional ATE: interaction order 3, sample size 2000}
\label{fig:CATE_inter3_2000}
\end{figure}

\begin{figure}[H]
\centering
\includegraphics[width=1\textwidth]{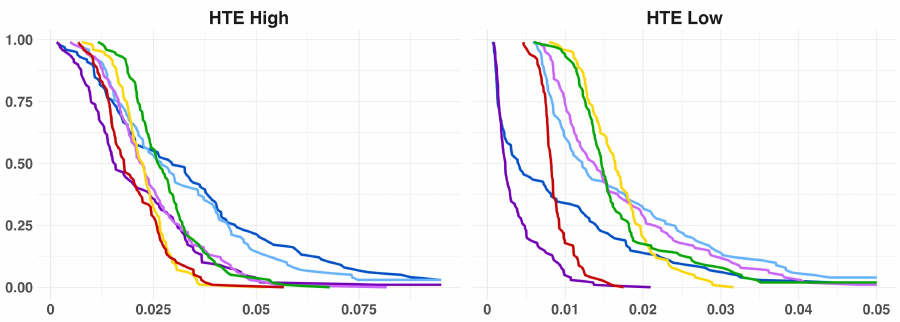}
\caption*{\footnotesize (a) iMSE}

\includegraphics[width=1\textwidth]{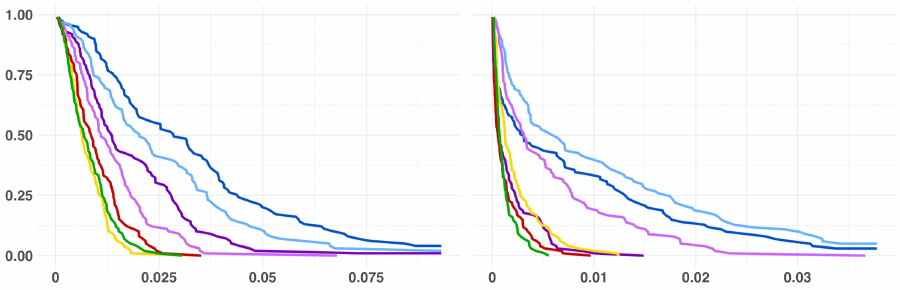}
\caption*{\footnotesize (b) iBias$^2$}

\includegraphics[width=1\textwidth]{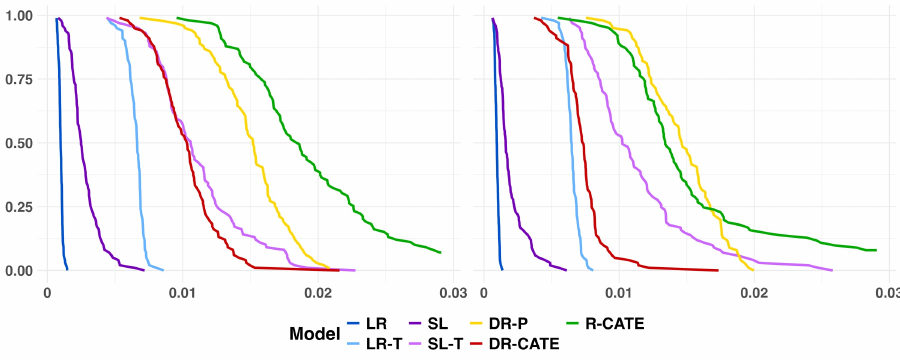}
\caption*{\footnotesize (c) iVariance}

\caption{Conditional ATE: interaction order 3, sample size 1000}
\label{fig:CATE_inter3_1000}
\end{figure}

\begin{figure}[H]
\centering
\includegraphics[width=1\textwidth]{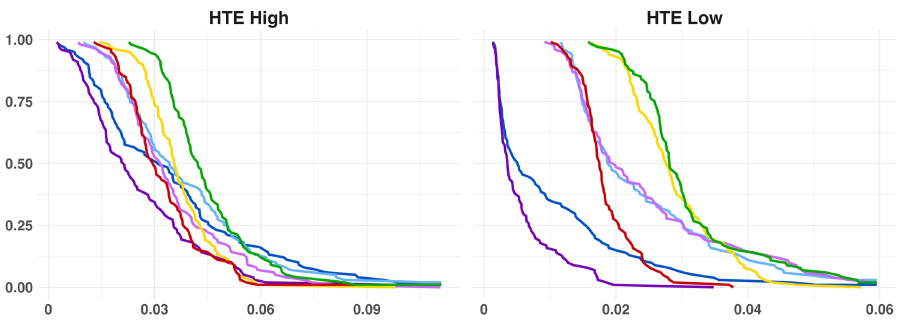}
\caption*{\footnotesize (a) iMSE}

\includegraphics[width=1\textwidth]{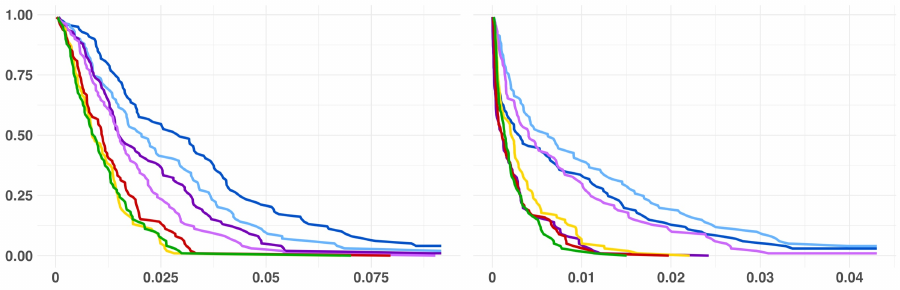}
\caption*{\footnotesize (b) iBias$^2$}

\includegraphics[width=1\textwidth]{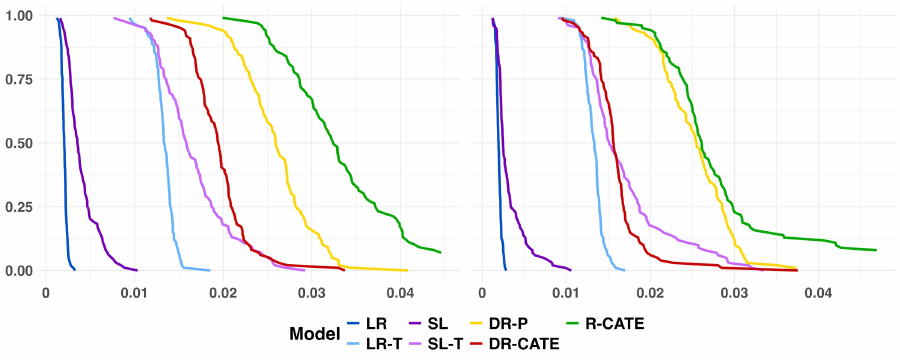}
\caption*{\footnotesize (c) iVariance}

\caption{Conditional ATE: interaction order 3, sample size 500}
\label{fig:CATE_inter3_500}
\end{figure}

\begin{figure}[H]
\centering
\includegraphics[width=1\textwidth]{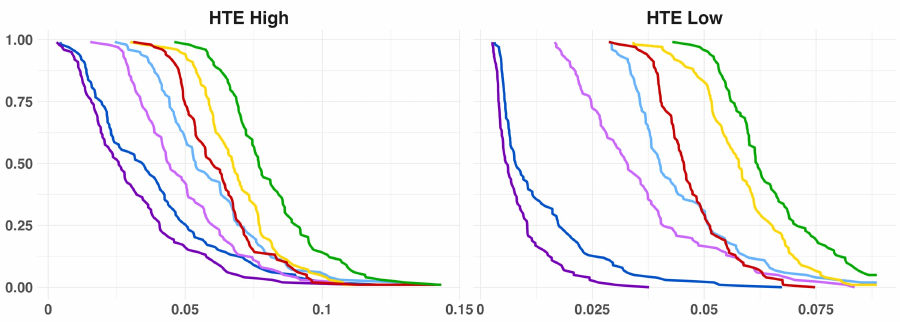}
\caption*{\footnotesize (a) iMSE}

\includegraphics[width=1\textwidth]{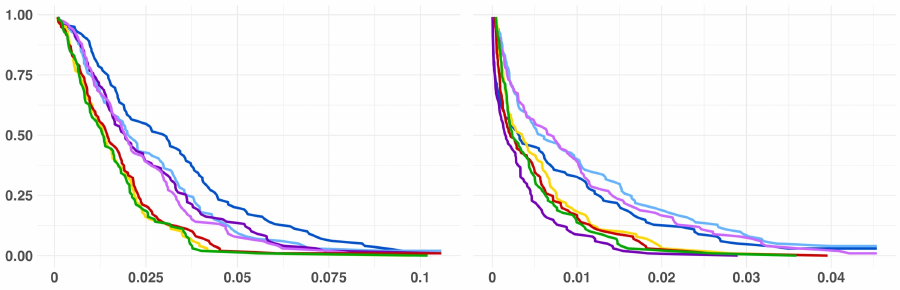}
\caption*{\footnotesize (b) iBias$^2$}

\includegraphics[width=1\textwidth]{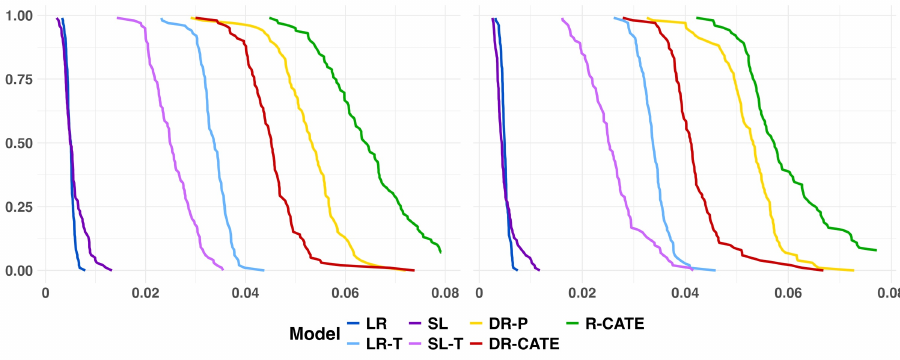}
\caption*{\footnotesize (c) iVariance}

\caption{Conditional ATE: interaction order 3, sample size 200}
\label{fig:CATE_inter3_200}
\end{figure}

\begin{figure}[H]
\centering
\includegraphics[width=1\textwidth]{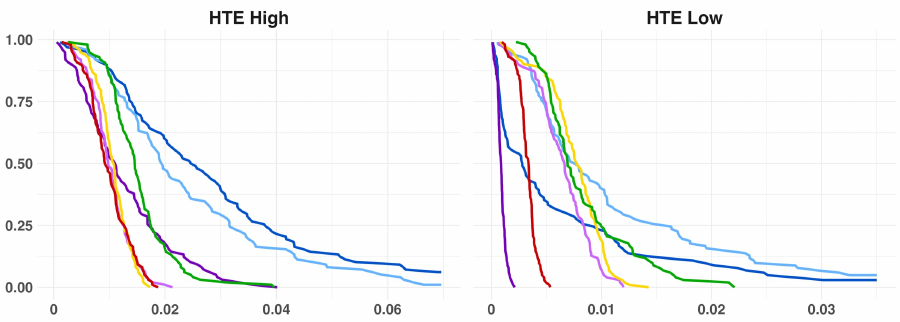}
\caption*{\footnotesize (a) iMSE}

\includegraphics[width=1\textwidth]{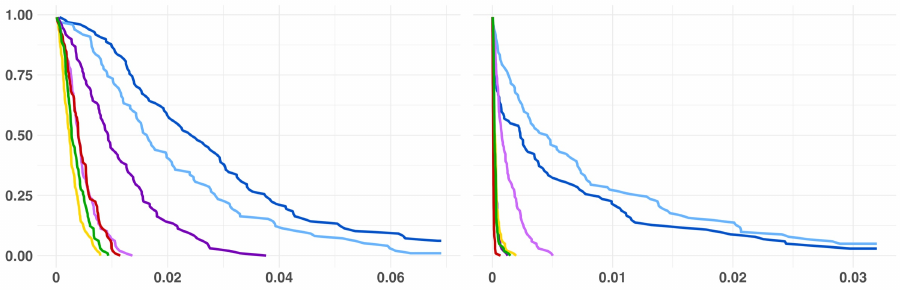}
\caption*{\footnotesize (b) iBias$^2$}

\includegraphics[width=1\textwidth]{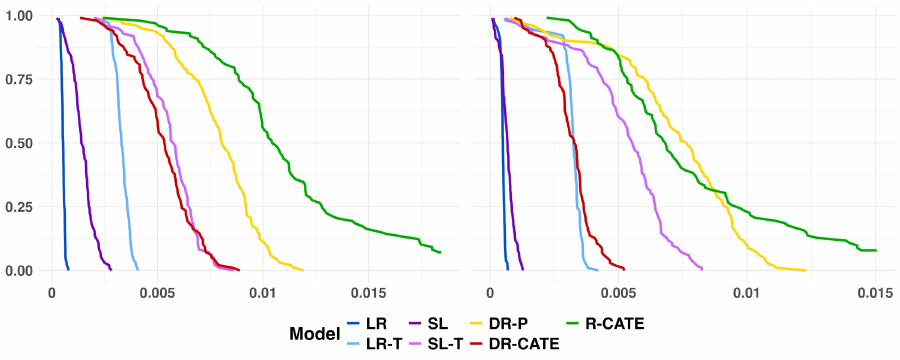}
\caption*{\footnotesize (c) iVariance}

\caption{Conditional ATE: interaction order 2, sample size 2000}
\label{fig:CATE_inter2_2000}
\end{figure}

\begin{figure}[H]
\centering
\includegraphics[width=1\textwidth]{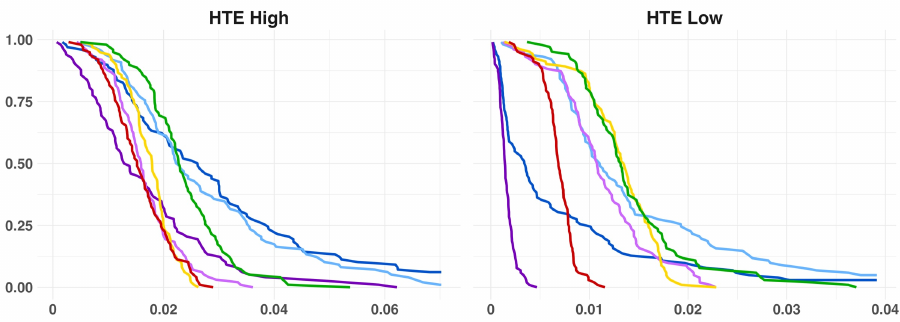}
\caption*{\footnotesize (a) iMSE}

\includegraphics[width=1\textwidth]{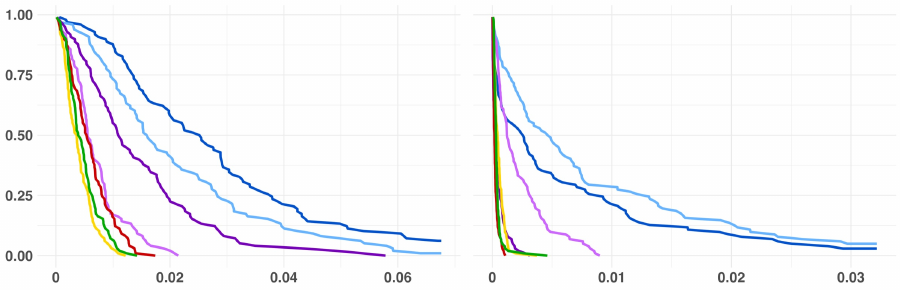}
\caption*{\footnotesize (b) iBias$^2$}

\includegraphics[width=1\textwidth]{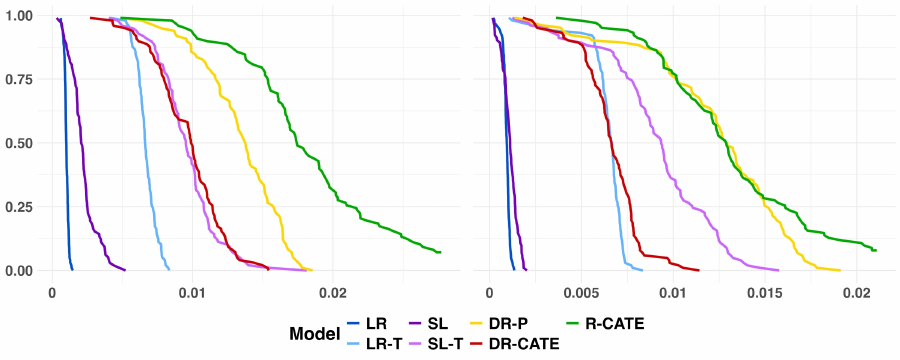}
\caption*{\footnotesize (c) iVariance}

\caption{Conditional ATE: interaction order 2, sample size 1000}
\label{fig:CATE_inter2_1000}
\end{figure}

\begin{figure}[H]
\centering
\includegraphics[width=1\textwidth]{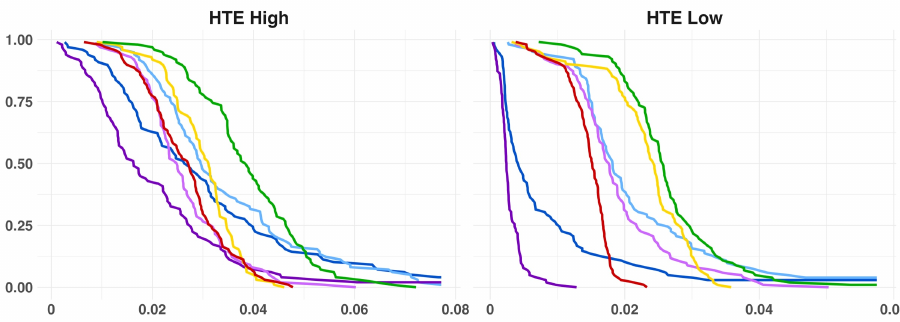}
\caption*{\footnotesize (a) iMSE}

\includegraphics[width=1\textwidth]{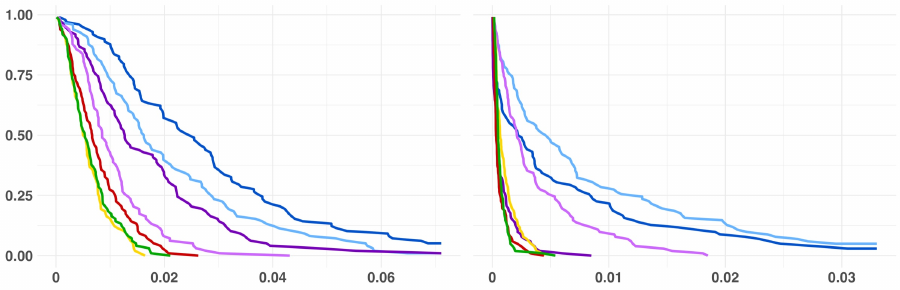}
\caption*{\footnotesize (b) iBias$^2$}

\includegraphics[width=1\textwidth]{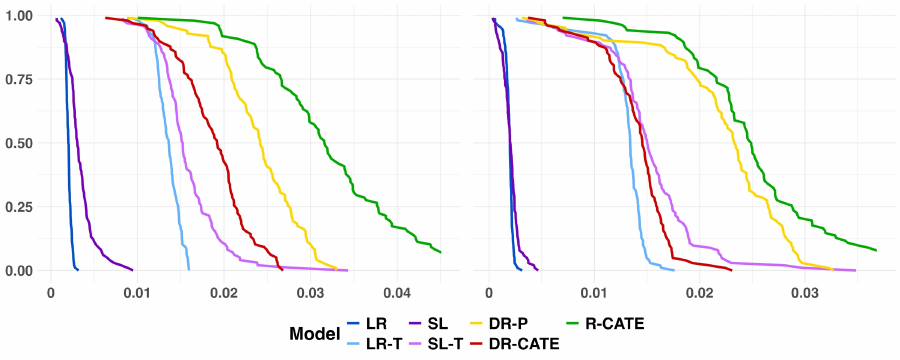}
\caption*{\footnotesize (c) iVariance}

\caption{Conditional ATE: interaction order 2, sample size 500}
\label{fig:CATE_inter2_500}
\end{figure}

\begin{figure}[H]
\centering
\includegraphics[width=1\textwidth]{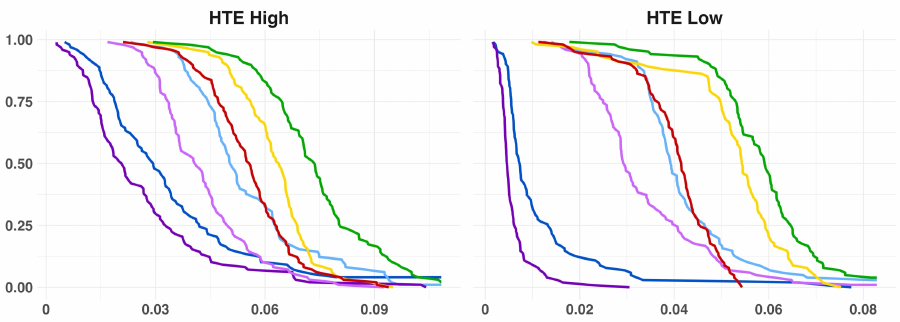}
\caption*{\footnotesize (a) iMSE}

\includegraphics[width=1\textwidth]{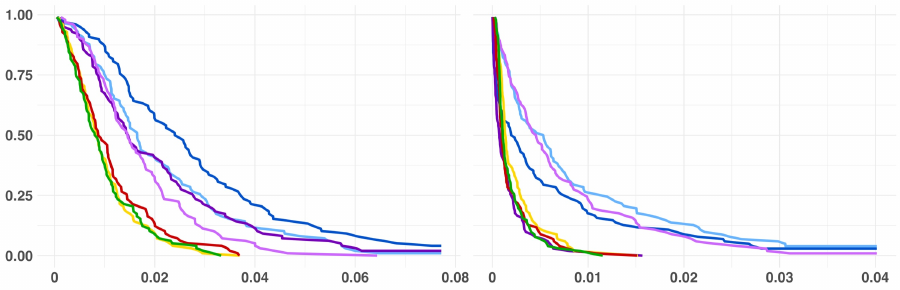}
\caption*{\footnotesize (b) iBias$^2$}

\includegraphics[width=1\textwidth]{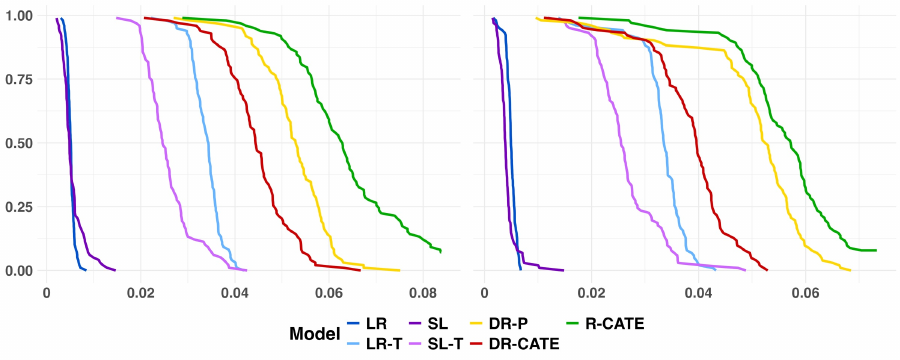}
\caption*{\footnotesize (c) iVariance}

\caption{Conditional ATE: interaction order 2, sample size 200}
\label{fig:CATE_inter2_200}
\end{figure}

\begin{figure}[H]
\centering
\includegraphics[width=1\textwidth]{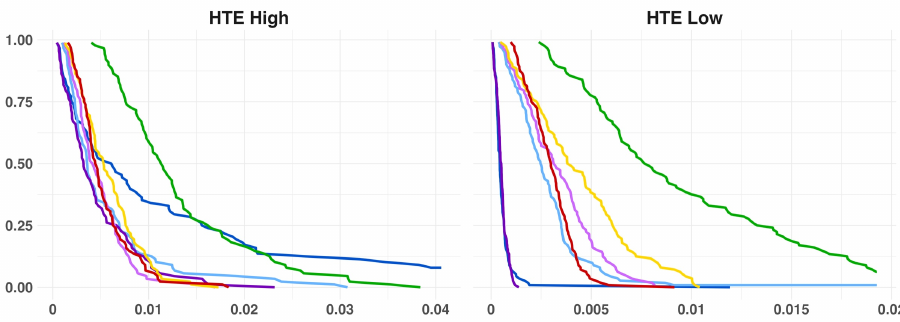}
\caption*{\footnotesize (a) iMSE}

\includegraphics[width=1\textwidth]{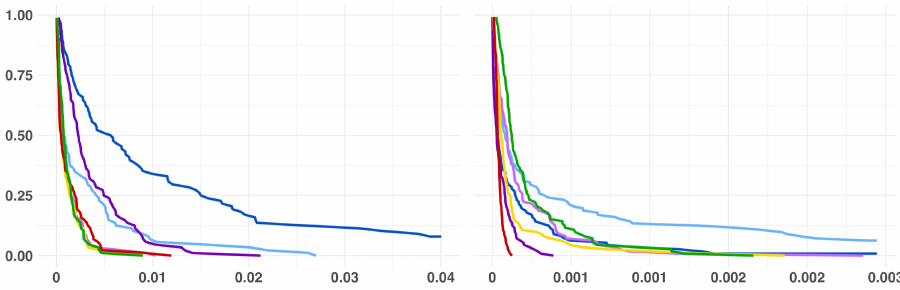}
\caption*{\footnotesize (b) iBias$^2$}

\includegraphics[width=1\textwidth]{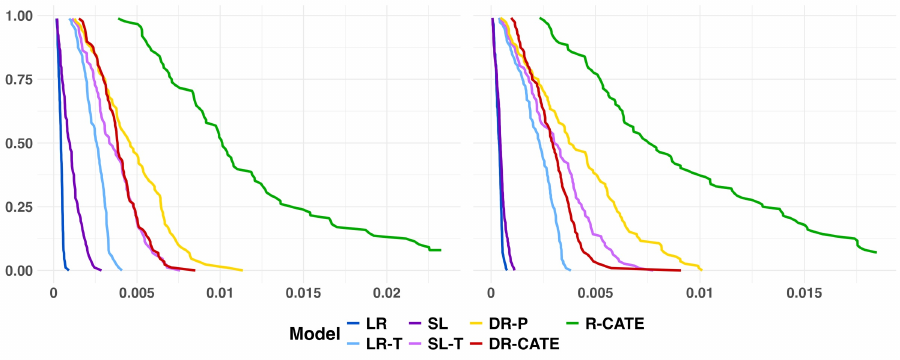}
\caption*{\footnotesize (c) iVariance}

\caption{Conditional ATE: interaction order 1, sample size 2000}
\label{fig:CATE_inter1_2000}
\end{figure}

\begin{figure}[H]
\centering
\includegraphics[width=1\textwidth]{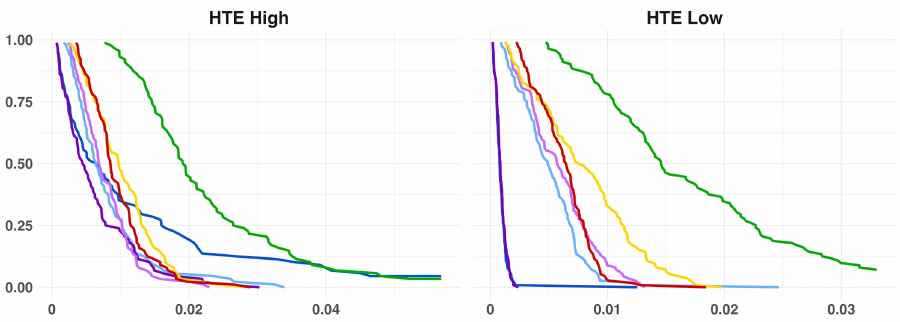}
\caption*{\footnotesize (a) iMSE}

\includegraphics[width=1\textwidth]{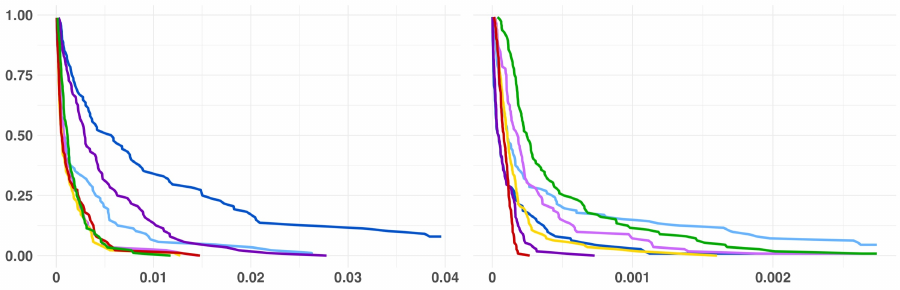}
\caption*{\footnotesize (b) iBias$^2$}

\includegraphics[width=1\textwidth]{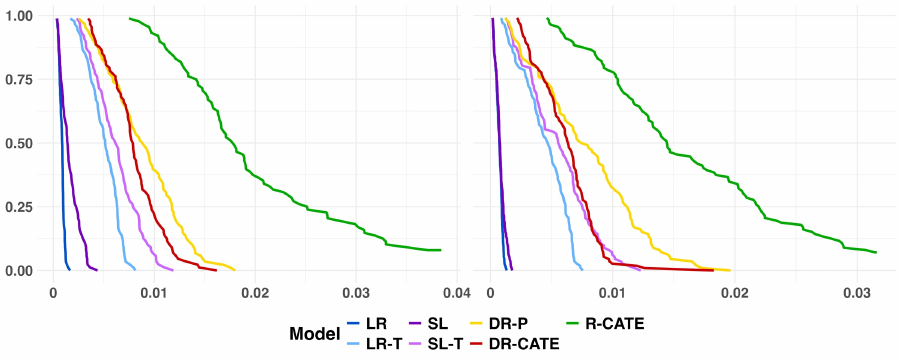}
\caption*{\footnotesize (c) iVariance}

\caption{Conditional ATE: interaction order 1, sample size 1000}
\label{fig:CATE_inter1_1000}
\end{figure}

\begin{figure}[H]
\centering
\includegraphics[width=1\textwidth]{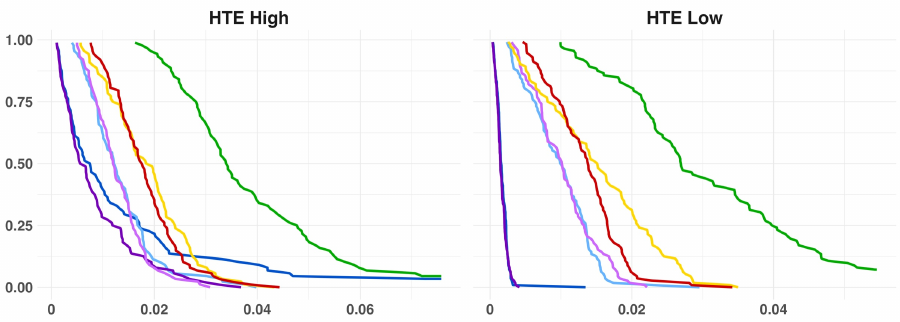}
\caption*{\footnotesize (a) iMSE}

\includegraphics[width=1\textwidth]{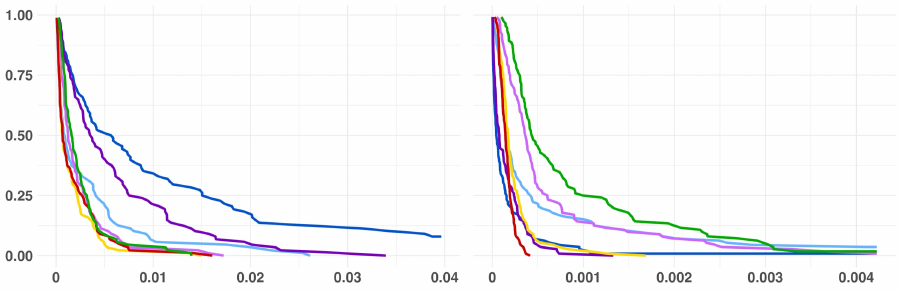}
\caption*{\footnotesize (b) iBias$^2$}

\includegraphics[width=1\textwidth]{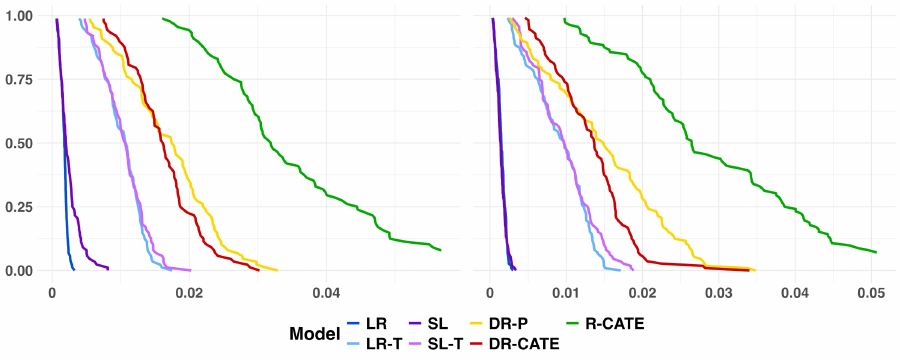}
\caption*{\footnotesize (c) iVariance}

\caption{Conditional ATE: interaction order 1, sample size 500}
\label{fig:CATE_inter1_500}
\end{figure}

\begin{figure}[H]
\centering
\includegraphics[width=1\textwidth]{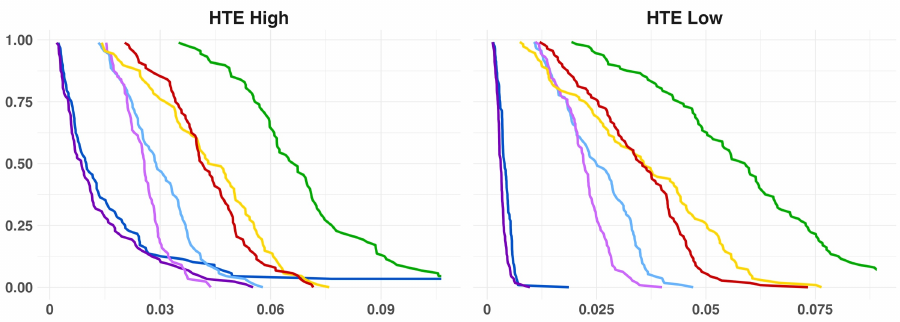}
\caption*{\footnotesize (a) iMSE}

\includegraphics[width=1\textwidth]{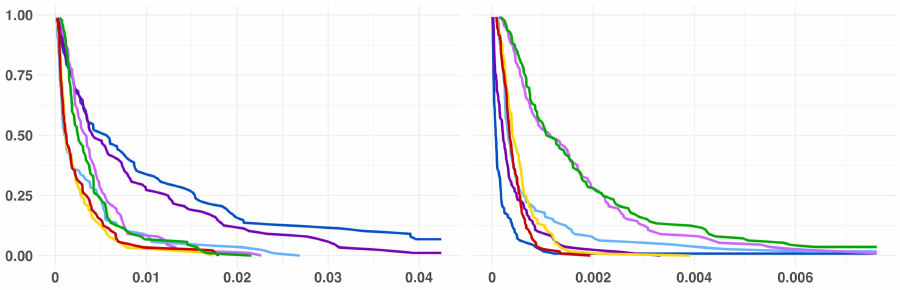}
\caption*{\footnotesize (b) iBias$^2$}

\includegraphics[width=1\textwidth]{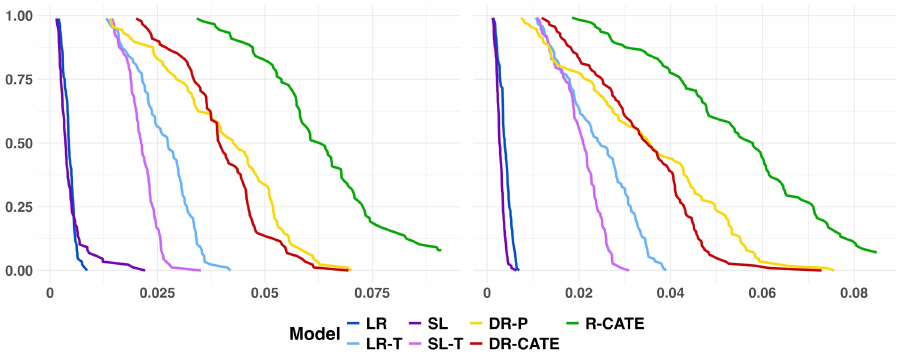}
\caption*{\footnotesize (c) iVariance}

\caption{Conditional ATE: interaction order 1, sample size 200}
\label{fig:CATE_inter1_200}
\end{figure}

\subsection{Summary Tables}
The following tables summarize performance metrics (iMSE, iBias$^2$, iVariance) for CATE, COR, and CRR estimation, stratified by treatment effect heterogeneity level (HTE High vs HTE Low). All numerical values are multiplied by 1000 for readability. Tables report the median and interquartile range (IQR) for each metric.


\begin{table}[H]
    \centering
    \small
\begin{tabular}{l *{3}{p{3cm}}}
        \toprule
        Estimator & iMSE (IQR) & iBias$^2$ (IQR) & iVar (IQR) \\
        \midrule
        \texttt{LR} & 10 (6--21) & 6 (2--15) & 4 (3--5) \\
        \texttt{LR-T} & 29 (22--37) & 1 (1--5) & 27 (22--32) \\
        \texttt{SL} & 9 (5--17) & 4 (2--12) & 4 (3--5) \\
        \texttt{SL-T} & 26 (21--29) & 3 (2--5) & 21 (19--24) \\
        \texttt{R-CATE} & 67 (57--77) & 3 (2--5) & 63 (54--73) \\
        \texttt{DR-P} & 45 (32--56) & 1 (1--3) & 43 (30--52) \\
        \texttt{DR-CATE} & 42 (35--50) & 1 (1--3) & 40 (34--46) \\
        \bottomrule
    \end{tabular}
    \caption{CATE estimation: Interaction Order 1, Sample Size 200, HTE High}
    \label{tab:cate_1_200_hte_h}
\end{table}

\begin{table}[H]
    \centering
    \small
    \begin{tabular}{l *{3}{p{3cm}}}
        \toprule
        Estimator & iMSE (IQR) & iBias$^2$ (IQR) & iVar (IQR) \\
        \midrule
        \texttt{LR} & 4 (3--5) & 0 (0--0) & 4 (3--5) \\
        \texttt{LR-T} & 25 (18--33) & 0 (0--1) & 25 (18--31) \\
        \texttt{SL} & 3 (2--4) & 0 (0--0) & 3 (2--3) \\
        \texttt{SL-T} & 22 (18--26) & 1 (1--2) & 21 (17--24) \\
        \texttt{R-CATE} & 59 (45--73) & 1 (1--2) & 58 (42--71) \\
        \texttt{DR-P} & 36 (23--50) & 0 (0--1) & 35 (22--50) \\
        \texttt{DR-CATE} & 35 (25--44) & 0 (0--1) & 35 (25--43) \\
        \bottomrule
    \end{tabular}
    \caption{CATE estimation: Interaction Order 1, Sample Size 200, HTE Low}
    \label{tab:cate_1_200_hte_l}
\end{table}

\begin{table}[H]
    \centering
    \small
    \begin{tabular}{l *{3}{p{3cm}}}
        \toprule
        Estimator & iMSE (IQR) & iBias$^2$ (IQR) & iVar (IQR) \\
        \midrule
        \texttt{LR} & 8 (3--17) & 6 (2--15) & 2 (1--2) \\
        \texttt{LR-T} & 12 (9--17) & 1 (0--4) & 11 (8--13) \\
        \texttt{SL} & 6 (3--12) & 4 (1--8) & 2 (1--3) \\
        \texttt{SL-T} & 12 (9--16) & 1 (1--3) & 11 (8--13) \\
        \texttt{R-CATE} & 34 (28--48) & 2 (1--3) & 32 (27--44) \\
        \texttt{DR-P} & 19 (13--24) & 1 (0--2) & 18 (12--22) \\
        \texttt{DR-CATE} & 18 (13--22) & 1 (0--3) & 16 (13--19) \\
        \bottomrule
    \end{tabular}
    \caption{CATE estimation: Interaction Order 1, Sample Size 500, HTE High}
    \label{tab:cate_1_500_hte_h}
\end{table}

\begin{table}[H]
    \centering
    \small
    \begin{tabular}{l *{3}{p{3cm}}}
        \toprule
        Estimator & iMSE (IQR) & iBias$^2$ (IQR) & iVar (IQR) \\
        \midrule
        \texttt{LR} & 2 (1--2) & 0 (0--0) & 1 (1--2) \\
        \texttt{LR-T} & 10 (6--14) & 0 (0--0) & 10 (6--12) \\
        \texttt{SL} & 1 (1--2) & 0 (0--0) & 1 (1--2) \\
        \texttt{SL-T} & 10 (7--13) & 0 (0--1) & 10 (7--13) \\
        \texttt{R-CATE} & 27 (22--40) & 0 (0--1) & 26 (21--39) \\
        \texttt{DR-P} & 15 (9--21) & 0 (0--0) & 15 (8--21) \\
        \texttt{DR-CATE} & 14 (10--17) & 0 (0--0) & 14 (10--16) \\
        \bottomrule
    \end{tabular}
    \caption{CATE estimation: Interaction Order 1, Sample Size 500, HTE Low}
    \label{tab:cate_1_500_hte_l}
\end{table}

\begin{table}[H]
    \centering
    \small
    \begin{tabular}{l *{3}{p{3cm}}}
        \toprule
        Estimator & iMSE (IQR) & iBias$^2$ (IQR) & iVar (IQR) \\
        \midrule
        \texttt{LR} & 7 (3--16) & 6 (2--15) & 1 (1--1) \\
        \texttt{LR-T} & 6 (4--10) & 1 (0--4) & 5 (4--6) \\
        \texttt{SL} & 5 (2--8) & 3 (1--7) & 1 (1--2) \\
        \texttt{SL-T} & 7 (5--10) & 1 (0--2) & 6 (4--8) \\
        \texttt{R-CATE} & 19 (15--27) & 1 (1--2) & 18 (14--25) \\
        \texttt{DR-P} & 10 (7--13) & 1 (0--2) & 9 (6--12) \\
        \texttt{DR-CATE} & 9 (7--12) & 1 (0--2) & 8 (6--10) \\
        \bottomrule
    \end{tabular}
    \caption{CATE estimation: Interaction Order 1, Sample Size 1000, HTE High}
    \label{tab:cate_1_1000_hte_h}
\end{table}

\begin{table}[H]
    \centering
    \small
    \begin{tabular}{l *{3}{p{3cm}}}
        \toprule
        Estimator & iMSE (IQR) & iBias$^2$ (IQR) & iVar (IQR) \\
        \midrule
        \texttt{LR} & 1 (1--1) & 0 (0--0) & 1 (1--1) \\
        \texttt{LR-T} & 5 (3--7) & 0 (0--0) & 5 (3--6) \\
        \texttt{SL} & 1 (1--1) & 0 (0--0) & 1 (1--1) \\
        \texttt{SL-T} & 6 (4--8) & 0 (0--0) & 6 (3--8) \\
        \texttt{R-CATE} & 15 (11--23) & 0 (0--1) & 14 (10--22) \\
        \texttt{DR-P} & 8 (5--11) & 0 (0--0) & 8 (4--11) \\
        \texttt{DR-CATE} & 6 (4--8) & 0 (0--0) & 6 (4--8) \\
        \bottomrule
    \end{tabular}
    \caption{CATE estimation: Interaction Order 1, Sample Size 1000, HTE Low}
    \label{tab:cate_1_1000_hte_l}
\end{table}

\begin{table}[H]
    \centering
    \small
    \begin{tabular}{l *{3}{p{3cm}}}
        \toprule
        Estimator & iMSE (IQR) & iBias$^2$ (IQR) & iVar (IQR) \\
        \midrule
        \texttt{LR} & 6 (2--16) & 6 (2--15) & 0 (0--1) \\
        \texttt{LR-T} & 4 (2--7) & 1 (0--4) & 3 (2--3) \\
        \texttt{SL} & 3 (2--7) & 2 (1--5) & 1 (1--1) \\
        \texttt{SL-T} & 4 (3--6) & 1 (0--2) & 3 (2--5) \\
        \texttt{R-CATE} & 11 (8--16) & 1 (0--2) & 10 (7--14) \\
        \texttt{DR-P} & 5 (3--8) & 0 (0--1) & 5 (3--6) \\
        \texttt{DR-CATE} & 4 (3--6) & 0 (0--2) & 4 (3--5) \\
        \bottomrule
    \end{tabular}
    \caption{CATE estimation: Interaction Order 1, Sample Size 2000, HTE High}
    \label{tab:cate_1_2000_hte_h}
\end{table}

\begin{table}[H]
    \centering
    \small
    \begin{tabular}{l *{3}{p{3cm}}}
        \toprule
        Estimator & iMSE (IQR) & iBias$^2$ (IQR) & iVar (IQR) \\
        \midrule
        \texttt{LR} & 0 (0--1) & 0 (0--0) & 0 (0--0) \\
        \texttt{LR-T} & 2 (2--4) & 0 (0--0) & 2 (2--3) \\
        \texttt{SL} & 1 (0--1) & 0 (0--0) & 0 (0--1) \\
        \texttt{SL-T} & 3 (2--5) & 0 (0--0) & 3 (2--4) \\
        \texttt{R-CATE} & 8 (5--13) & 0 (0--0) & 8 (5--13) \\
        \texttt{DR-P} & 4 (2--6) & 0 (0--0) & 4 (2--6) \\
        \texttt{DR-CATE} & 3 (2--4) & 0 (0--0) & 3 (2--4) \\
        \bottomrule
    \end{tabular}
    \caption{CATE estimation: Interaction Order 1, Sample Size 2000, HTE Low}
    \label{tab:cate_1_2000_hte_l}
\end{table}

\begin{table}[H]
    \centering
    \small
    \begin{tabular}{l *{3}{p{3cm}}}
        \toprule
        Estimator & iMSE (IQR) & iBias$^2$ (IQR) & iVar (IQR) \\
        \midrule
        \texttt{LR} & 29 (19--42) & 24 (14--37) & 5 (5--6) \\
        \texttt{LR-T} & 51 (45--63) & 17 (10--29) & 34 (32--36) \\
        \texttt{SL} & 21 (13--33) & 15 (9--28) & 5 (4--6) \\
        \texttt{SL-T} & 41 (34--50) & 15 (9--22) & 25 (22--28) \\
        \texttt{R-CATE} & 74 (65--81) & 8 (4--12) & 63 (56--70) \\
        \texttt{DR-P} & 64 (56--69) & 8 (5--13) & 53 (49--58) \\
        \texttt{DR-CATE} & 56 (47--63) & 9 (5--14) & 44 (41--48) \\
        \bottomrule
    \end{tabular}
    \caption{CATE estimation: Interaction Order 2, Sample Size 200, HTE High}
    \label{tab:cate_2_200_hte_h}
\end{table}

\begin{table}[H]
    \centering
    \small
    \begin{tabular}{l *{3}{p{3cm}}}
        \toprule
        Estimator & iMSE (IQR) & iBias$^2$ (IQR) & iVar (IQR) \\
        \midrule
        \texttt{LR} & 7 (6--13) & 2 (0--8) & 5 (4--6) \\
        \texttt{LR-T} & 39 (35--46) & 5 (2--12) & 34 (32--36) \\
        \texttt{SL} & 5 (4--6) & 1 (0--2) & 4 (3--5) \\
        \texttt{SL-T} & 29 (25--40) & 5 (2--10) & 26 (23--29) \\
        \texttt{R-CATE} & 59 (53--65) & 1 (1--2) & 57 (51--62) \\
        \texttt{DR-P} & 54 (50--59) & 1 (1--3) & 53 (49--57) \\
        \texttt{DR-CATE} & 41 (37--45) & 1 (1--2) & 40 (35--43) \\
        \bottomrule
    \end{tabular}
    \caption{CATE estimation: Interaction Order 2, Sample Size 200, HTE Low}
    \label{tab:cate_2_200_hte_l}
\end{table}

\begin{table}[H]
    \centering
    \small
    \begin{tabular}{l *{3}{p{3cm}}}
        \toprule
        Estimator & iMSE (IQR) & iBias$^2$ (IQR) & iVar (IQR) \\
        \midrule
        \texttt{LR} & 27 (16--40) & 25 (14--38) & 2 (2--2) \\
        \texttt{LR-T} & 30 (24--43) & 16 (10--29) & 14 (13--15) \\
        \texttt{SL} & 16 (10--27) & 13 (8--22) & 3 (2--4) \\
        \texttt{SL-T} & 25 (21--32) & 9 (6--13) & 15 (14--17) \\
        \texttt{R-CATE} & 39 (32--46) & 5 (3--9) & 32 (27--38) \\
        \texttt{DR-P} & 31 (25--35) & 5 (3--8) & 24 (21--28) \\
        \texttt{DR-CATE} & 27 (21--32) & 7 (4--11) & 19 (16--21) \\
        \bottomrule
    \end{tabular}
    \caption{CATE estimation: Interaction Order 2, Sample Size 500, HTE High}
    \label{tab:cate_2_500_hte_h}
\end{table}

\begin{table}[H]
    \centering
    \small
    \begin{tabular}{l *{3}{p{3cm}}}
        \toprule
        Estimator & iMSE (IQR) & iBias$^2$ (IQR) & iVar (IQR) \\
        \midrule
        \texttt{LR} & 4 (2--10) & 2 (0--8) & 2 (2--2) \\
        \texttt{LR-T} & 18 (15--26) & 5 (2--12) & 13 (13--14) \\
        \texttt{SL} & 2 (2--4) & 0 (0--1) & 2 (2--2) \\
        \texttt{SL-T} & 18 (15--22) & 2 (1--5) & 15 (13--17) \\
        \texttt{R-CATE} & 26 (22--30) & 0 (0--1) & 25 (22--29) \\
        \texttt{DR-P} & 24 (20--28) & 1 (0--1) & 23 (20--27) \\
        \texttt{DR-CATE} & 15 (13--17) & 0 (0--1) & 15 (13--16) \\
        \bottomrule
    \end{tabular}
    \caption{CATE estimation: Interaction Order 2, Sample Size 500, HTE Low}
    \label{tab:cate_2_500_hte_l}
\end{table}

\begin{table}[H]
    \centering
    \small
    \begin{tabular}{l *{3}{p{3cm}}}
        \toprule
        Estimator & iMSE (IQR) & iBias$^2$ (IQR) & iVar (IQR) \\
        \midrule
        \texttt{LR} & 26 (15--39) & 25 (14--38) & 1 (1--1) \\
        \texttt{LR-T} & 23 (17--35) & 16 (10--28) & 7 (6--7) \\
        \texttt{SL} & 13 (8--22) & 11 (6--19) & 2 (2--3) \\
        \texttt{SL-T} & 16 (12--20) & 6 (4--9) & 10 (8--11) \\
        \texttt{R-CATE} & 23 (18--28) & 4 (2--6) & 17 (15--21) \\
        \texttt{DR-P} & 18 (15--20) & 4 (2--6) & 14 (12--16) \\
        \texttt{DR-CATE} & 15 (12--20) & 5 (3--8) & 10 (8--11) \\
        \bottomrule
    \end{tabular}
    \caption{CATE estimation: Interaction Order 2, Sample Size 1000, HTE High}
    \label{tab:cate_2_1000_hte_h}
\end{table}

\begin{table}[H]
    \centering
    \small
    \begin{tabular}{l *{3}{p{3cm}}}
        \toprule
        Estimator & iMSE (IQR) & iBias$^2$ (IQR) & iVar (IQR) \\
        \midrule
        \texttt{LR} & 3 (1--10) & 3 (0--9) & 1 (1--1) \\
        \texttt{LR-T} & 11 (8--19) & 5 (1--13) & 7 (6--7) \\
        \texttt{SL} & 1 (1--2) & 0 (0--1) & 1 (1--1) \\
        \texttt{SL-T} & 11 (9--14) & 1 (1--3) & 9 (8--11) \\
        \texttt{R-CATE} & 13 (11--17) & 0 (0--0) & 13 (10--16) \\
        \texttt{DR-P} & 13 (11--16) & 0 (0--1) & 13 (10--15) \\
        \texttt{DR-CATE} & 7 (6--8) & 0 (0--0) & 7 (6--8) \\
        \bottomrule
    \end{tabular}
    \caption{CATE estimation: Interaction Order 2, Sample Size 1000, HTE Low}
    \label{tab:cate_2_1000_hte_l}
\end{table}

\begin{table}[H]
    \centering
    \small
    \begin{tabular}{l *{3}{p{3cm}}}
        \toprule
        Estimator & iMSE (IQR) & iBias$^2$ (IQR) & iVar (IQR) \\
        \midrule
        \texttt{LR} & 25 (14--38) & 24 (14--38) & 1 (0--1) \\
        \texttt{LR-T} & 19 (13--32) & 16 (9--29) & 3 (3--4) \\
        \texttt{SL} & 11 (6--17) & 9 (5--15) & 1 (1--2) \\
        \texttt{SL-T} & 10 (8--13) & 4 (3--6) & 6 (5--7) \\
        \texttt{R-CATE} & 15 (11--17) & 3 (2--5) & 11 (9--13) \\
        \texttt{DR-P} & 11 (9--13) & 2 (1--4) & 8 (7--9) \\
        \texttt{DR-CATE} & 9 (7--13) & 4 (2--6) & 5 (4--6) \\
        \bottomrule
    \end{tabular}
    \caption{CATE estimation: Interaction Order 2, Sample Size 2000, HTE High}
    \label{tab:cate_2_2000_hte_h}
\end{table}

\begin{table}[H]
    \centering
    \small
    \begin{tabular}{l *{3}{p{3cm}}}
        \toprule
        Estimator & iMSE (IQR) & iBias$^2$ (IQR) & iVar (IQR) \\
        \midrule
        \texttt{LR} & 3 (1--9) & 2 (0--9) & 0 (0--1) \\
        \texttt{LR-T} & 8 (5--16) & 5 (1--12) & 3 (3--3) \\
        \texttt{SL} & 1 (1--1) & 0 (0--0) & 1 (0--1) \\
        \texttt{SL-T} & 7 (5--8) & 1 (0--2) & 6 (4--6) \\
        \texttt{R-CATE} & 7 (5--10) & 0 (0--0) & 7 (5--10) \\
        \texttt{DR-P} & 8 (6--10) & 0 (0--0) & 8 (6--9) \\
        \texttt{DR-CATE} & 3 (3--4) & 0 (0--0) & 3 (3--4) \\
        \bottomrule
    \end{tabular}
    \caption{CATE estimation: Interaction Order 2, Sample Size 2000, HTE Low}
    \label{tab:cate_2_2000_hte_l}
\end{table}

\begin{table}[H]
    \centering
    \small
    \begin{tabular}{l *{3}{p{3cm}}}
        \toprule
        Estimator & iMSE (IQR) & iBias$^2$ (IQR) & iVar (IQR) \\
        \midrule
        \texttt{LR} & 35 (20--50) & 31 (15--44) & 5 (4--5) \\
        \texttt{LR-T} & 54 (44--70) & 20 (10--37) & 34 (32--36) \\
        \texttt{SL} & 26 (16--41) & 19 (11--36) & 5 (4--7) \\
        \texttt{SL-T} & 44 (34--60) & 19 (10--33) & 25 (22--29) \\
        \texttt{R-CATE} & 77 (69--91) & 13 (7--21) & 64 (58--71) \\
        \texttt{DR-P} & 67 (59--77) & 14 (7--22) & 53 (49--57) \\
        \texttt{DR-CATE} & 60 (50--70) & 15 (8--23) & 45 (42--48) \\
        \bottomrule
    \end{tabular}
    \caption{CATE estimation: Interaction Order 3, Sample Size 200, HTE High}
    \label{tab:cate_3_200_hte_h}
\end{table}

\begin{table}[H]
    \centering
    \small
    \begin{tabular}{l *{3}{p{3cm}}}
        \toprule
        Estimator & iMSE (IQR) & iBias$^2$ (IQR) & iVar (IQR) \\
        \midrule
        \texttt{LR} & 8 (6--18) & 3 (0--13) & 5 (4--5) \\
        \texttt{LR-T} & 40 (36--51) & 6 (2--16) & 34 (32--36) \\
        \texttt{SL} & 6 (4--10) & 2 (0--5) & 4 (4--6) \\
        \texttt{SL-T} & 33 (25--41) & 7 (2--14) & 25 (22--29) \\
        \texttt{R-CATE} & 62 (57--71) & 2 (1--6) & 57 (53--65) \\
        \texttt{DR-P} & 58 (52--65) & 4 (1--8) & 53 (50--56) \\
        \texttt{DR-CATE} & 45 (41--51) & 2 (1--7) & 41 (38--44) \\
        \bottomrule
    \end{tabular}
    \caption{CATE estimation: Interaction Order 3, Sample Size 200, HTE Low}
    \label{tab:cate_3_200_hte_l}
\end{table}

\begin{table}[H]
    \centering
    \small
    \begin{tabular}{l *{3}{p{3cm}}}
        \toprule
        Estimator & iMSE (IQR) & iBias$^2$ (IQR) & iVar (IQR) \\
        \midrule
        \texttt{LR} & 31 (16--46) & 29 (15--44) & 2 (2--2) \\
        \texttt{LR-T} & 34 (24--50) & 20 (10--37) & 13 (13--14) \\
        \texttt{SL} & 21 (13--35) & 15 (9--32) & 4 (3--5) \\
        \texttt{SL-T} & 32 (24--41) & 15 (8--24) & 16 (14--19) \\
        \texttt{R-CATE} & 42 (35--51) & 9 (5--15) & 32 (29--37) \\
        \texttt{DR-P} & 36 (30--43) & 9 (5--15) & 26 (23--28) \\
        \texttt{DR-CATE} & 30 (25--38) & 11 (6--18) & 19 (18--21) \\
        \bottomrule
    \end{tabular}
    \caption{CATE estimation: Interaction Order 3, Sample Size 500, HTE High}
    \label{tab:cate_3_500_hte_h}
\end{table}

\begin{table}[H]
    \centering
    \small
    \begin{tabular}{l *{3}{p{3cm}}}
        \toprule
        Estimator & iMSE (IQR) & iBias$^2$ (IQR) & iVar (IQR) \\
        \midrule
        \texttt{LR} & 5 (2--15) & 3 (0--13) & 2 (2--2) \\
        \texttt{LR-T} & 19 (15--30) & 6 (2--17) & 13 (12--14) \\
        \texttt{SL} & 4 (2--7) & 1 (0--3) & 2 (2--3) \\
        \texttt{SL-T} & 20 (15--30) & 4 (1--11) & 15 (14--19) \\
        \texttt{R-CATE} & 28 (26--33) & 1 (1--3) & 26 (24--30) \\
        \texttt{DR-P} & 28 (23--33) & 2 (1--5) & 25 (22--29) \\
        \texttt{DR-CATE} & 17 (16--20) & 1 (0--3) & 16 (14--17) \\
        \bottomrule
    \end{tabular}
    \caption{CATE estimation: Interaction Order 3, Sample Size 500, HTE Low}
    \label{tab:cate_3_500_hte_l}
\end{table}

\begin{table}[H]
    \centering
    \small
    \begin{tabular}{l *{3}{p{3cm}}}
        \toprule
        Estimator & iMSE (IQR) & iBias$^2$ (IQR) & iVar (IQR) \\
        \midrule
        \texttt{LR} & 30 (15--45) & 29 (14--44) & 1 (1--1) \\
        \texttt{LR-T} & 27 (17--43) & 20 (10--37) & 7 (6--7) \\
        \texttt{SL} & 16 (10--31) & 14 (8--28) & 2 (2--3) \\
        \texttt{SL-T} & 22 (17--31) & 11 (7--18) & 11 (9--13) \\
        \texttt{R-CATE} & 26 (21--32) & 7 (4--11) & 19 (16--23) \\
        \texttt{DR-P} & 23 (19--27) & 6 (4--11) & 15 (13--17) \\
        \texttt{DR-CATE} & 18 (14--26) & 9 (5--14) & 10 (9--12) \\
        \bottomrule
    \end{tabular}
    \caption{CATE estimation: Interaction Order 3, Sample Size 1000, HTE High}
    \label{tab:cate_3_1000_hte_h}
\end{table}

\begin{table}[H]
    \centering
    \small
    \begin{tabular}{l *{3}{p{3cm}}}
        \toprule
        Estimator & iMSE (IQR) & iBias$^2$ (IQR) & iVar (IQR) \\
        \midrule
        \texttt{LR} & 4 (1--14) & 3 (0--13) & 1 (1--1) \\
        \texttt{LR-T} & 12 (8--24) & 6 (2--17) & 6 (6--7) \\
        \texttt{SL} & 2 (1--5) & 1 (0--2) & 2 (1--2) \\
        \texttt{SL-T} & 14 (10--22) & 3 (1--8) & 10 (9--13) \\
        \texttt{R-CATE} & 15 (13--18) & 1 (0--1) & 13 (12--16) \\
        \texttt{DR-P} & 16 (13--20) & 1 (0--3) & 15 (12--17) \\
        \texttt{DR-CATE} & 8 (8--9) & 1 (0--2) & 7 (7--8) \\
        \bottomrule
    \end{tabular}
    \caption{CATE estimation: Interaction Order 3, Sample Size 1000, HTE Low}
    \label{tab:cate_3_1000_hte_l}
\end{table}

\begin{table}[H]
    \centering
    \small
    \begin{tabular}{l *{3}{p{3cm}}}
        \toprule
        Estimator & iMSE (IQR) & iBias$^2$ (IQR) & iVar (IQR) \\
        \midrule
        \texttt{LR} & 30 (15--46) & 29 (14--45) & 1 (0--1) \\
        \texttt{LR-T} & 24 (13--40) & 21 (10--37) & 3 (3--3) \\
        \texttt{SL} & 13 (8--23) & 11 (6--21) & 2 (1--2) \\
        \texttt{SL-T} & 15 (11--21) & 8 (5--13) & 7 (5--8) \\
        \texttt{R-CATE} & 17 (13--22) & 5 (3--8) & 11 (10--14) \\
        \texttt{DR-P} & 14 (11--16) & 4 (3--6) & 9 (8--10) \\
        \texttt{DR-CATE} & 12 (8--17) & 6 (4--10) & 6 (5--7) \\
        \bottomrule
    \end{tabular}
    \caption{CATE estimation: Interaction Order 3, Sample Size 2000, HTE High}
    \label{tab:cate_3_2000_hte_h}
\end{table}

\begin{table}[H]
    \centering
    \small
    \begin{tabular}{l *{3}{p{3cm}}}
        \toprule
        Estimator & iMSE (IQR) & iBias$^2$ (IQR) & iVar (IQR) \\
        \midrule
        \texttt{LR} & 4 (1--14) & 3 (0--13) & 0 (0--1) \\
        \texttt{LR-T} & 9 (5--20) & 6 (2--16) & 3 (3--3) \\
        \texttt{SL} & 1 (1--2) & 0 (0--1) & 1 (1--1) \\
        \texttt{SL-T} & 9 (7--14) & 2 (1--5) & 6 (6--8) \\
        \texttt{R-CATE} & 8 (6--11) & 0 (0--1) & 7 (6--10) \\
        \texttt{DR-P} & 10 (8--12) & 1 (0--1) & 9 (7--10) \\
        \texttt{DR-CATE} & 4 (3--4) & 0 (0--1) & 4 (3--4) \\
        \bottomrule
    \end{tabular}
    \caption{CATE estimation: Interaction Order 3, Sample Size 2000, HTE Low}
    \label{tab:cate_3_2000_hte_l}
\end{table}


\begin{table}[H]
    \centering
    \small
    \begin{tabular}{l *{3}{p{3cm}}}
        \toprule
        Estimator & iMSE (IQR) & iBias$^2$ (IQR) & iVar (IQR) \\
        \midrule
        \texttt{LR} & 387 (301--554) & 223 (106--395) & 157 (123--197) \\
        \texttt{LR-T} & 1108 (1003--1194) & 59 (34--123) & 1007 (906--1084) \\
        \texttt{SL} & 318 (246--464) & 196 (119--337) & 123 (99--152) \\
        \texttt{SL-T} & 867 (669--1017) & 123 (80--184) & 727 (516--872) \\
        \texttt{R-LOR} & 3017 (2465--3761) & 114 (72--178) & 2878 (2405--3624) \\
        \texttt{DR-P} & 1644 (1475--1858) & 63 (44--98) & 1585 (1408--1775) \\
        \texttt{DR-LOR} & 2077 (1579--2680) & 67 (44--122) & 1975 (1481--2638) \\
        \bottomrule
    \end{tabular}
    \caption{COR estimation: Interaction Order 1, Sample Size 200, HTE High}
    \label{tab:cor_1_200_hte_h}
\end{table}

\begin{table}[H]
    \centering
    \small
    \begin{tabular}{l *{3}{p{3cm}}}
        \toprule
        Estimator & iMSE (IQR) & iBias$^2$ (IQR) & iVar (IQR) \\
        \midrule
        \texttt{LR} & 176 (142--212) & 4 (2--13) & 173 (134--202) \\
        \texttt{LR-T} & 1028 (927--1140) & 17 (12--31) & 1000 (905--1111) \\
        \texttt{SL} & 134 (101--165) & 12 (4--27) & 113 (88--137) \\
        \texttt{SL-T} & 832 (655--969) & 42 (21--79) & 791 (593--890) \\
        \texttt{R-LOR} & 3123 (2348--3899) & 63 (35--148) & 3045 (2312--3724) \\
        \texttt{DR-P} & 1505 (1293--1739) & 20 (15--30) & 1476 (1281--1715) \\
        \texttt{DR-LOR} & 2179 (1584--2710) & 25 (17--37) & 2159 (1570--2684) \\
        \bottomrule
    \end{tabular}
    \caption{COR estimation: Interaction Order 1, Sample Size 200, HTE Low}
    \label{tab:cor_1_200_hte_l}
\end{table}

\begin{table}[H]
    \centering
    \small
    \begin{tabular}{l *{3}{p{3cm}}}
        \toprule
        Estimator & iMSE (IQR) & iBias$^2$ (IQR) & iVar (IQR) \\
        \midrule
        \texttt{LR} & 284 (201--463) & 218 (105--388) & 58 (46--87) \\
        \texttt{LR-T} & 499 (419--563) & 41 (23--108) & 415 (331--494) \\
        \texttt{SL} & 245 (172--353) & 163 (94--277) & 77 (62--99) \\
        \texttt{SL-T} & 477 (402--520) & 64 (37--98) & 387 (309--454) \\
        \texttt{R-LOR} & 1809 (1249--2356) & 71 (44--108) & 1736 (1161--2280) \\
        \texttt{DR-P} & 732 (654--825) & 38 (23--80) & 688 (624--755) \\
        \texttt{DR-LOR} & 894 (653--1222) & 38 (24--84) & 825 (575--1144) \\
        \bottomrule
    \end{tabular}
    \caption{COR estimation: Interaction Order 1, Sample Size 500, HTE High}
    \label{tab:cor_1_500_hte_h}
\end{table}

\begin{table}[H]
    \centering
    \small
    \begin{tabular}{l *{3}{p{3cm}}}
        \toprule
        Estimator & iMSE (IQR) & iBias$^2$ (IQR) & iVar (IQR) \\
        \midrule
        \texttt{LR} & 79 (61--98) & 3 (1--10) & 73 (53--93) \\
        \texttt{LR-T} & 438 (371--519) & 11 (6--18) & 407 (355--501) \\
        \texttt{SL} & 73 (57--93) & 6 (2--13) & 64 (52--79) \\
        \texttt{SL-T} & 428 (360--497) & 16 (10--27) & 402 (352--474) \\
        \texttt{R-LOR} & 1853 (1187--2492) & 28 (15--72) & 1833 (1172--2405) \\
        \texttt{DR-P} & 685 (582--774) & 10 (7--14) & 674 (577--758) \\
        \texttt{DR-LOR} & 959 (618--1326) & 10 (7--17) & 947 (606--1311) \\
        \bottomrule
    \end{tabular}
    \caption{COR estimation: Interaction Order 1, Sample Size 500, HTE Low}
    \label{tab:cor_1_500_hte_l}
\end{table}

\begin{table}[H]
    \centering
    \small
    \begin{tabular}{l *{3}{p{3cm}}}
        \toprule
        Estimator & iMSE (IQR) & iBias$^2$ (IQR) & iVar (IQR) \\
        \midrule
        \texttt{LR} & 251 (153--419) & 220 (104--391) & 30 (22--42) \\
        \texttt{LR-T} & 284 (236--329) & 37 (18--107) & 197 (163--267) \\
        \texttt{SL} & 202 (135--279) & 145 (80--224) & 55 (43--66) \\
        \texttt{SL-T} & 298 (256--332) & 41 (24--81) & 238 (196--281) \\
        \texttt{R-LOR} & 1025 (670--1434) & 52 (29--75) & 960 (605--1386) \\
        \texttt{DR-P} & 415 (367--477) & 28 (18--69) & 371 (342--424) \\
        \texttt{DR-LOR} & 442 (323--611) & 31 (19--73) & 359 (269--553) \\
        \bottomrule
    \end{tabular}
    \caption{COR estimation: Interaction Order 1, Sample Size 1000, HTE High}
    \label{tab:cor_1_1000_hte_h}
\end{table}

\begin{table}[H]
    \centering
    \small
    \begin{tabular}{l *{3}{p{3cm}}}
        \toprule
        Estimator & iMSE (IQR) & iBias$^2$ (IQR) & iVar (IQR) \\
        \midrule
        \texttt{LR} & 41 (31--55) & 3 (1--10) & 34 (24--50) \\
        \texttt{LR-T} & 237 (184--281) & 7 (4--14) & 222 (171--262) \\
        \texttt{SL} & 45 (35--55) & 4 (1--8) & 39 (31--47) \\
        \texttt{SL-T} & 262 (219--297) & 9 (5--16) & 246 (212--285) \\
        \texttt{R-LOR} & 1012 (592--1489) & 17 (9--35) & 997 (584--1470) \\
        \texttt{DR-P} & 374 (313--440) & 7 (4--10) & 370 (309--436) \\
        \texttt{DR-LOR} & 454 (293--635) & 6 (4--9) & 449 (288--619) \\
        \bottomrule
    \end{tabular}
    \caption{COR estimation: Interaction Order 1, Sample Size 1000, HTE Low}
    \label{tab:cor_1_1000_hte_l}
\end{table}

\begin{table}[H]
    \centering
    \small
    \begin{tabular}{l *{3}{p{3cm}}}
        \toprule
        Estimator & iMSE (IQR) & iBias$^2$ (IQR) & iVar (IQR) \\
        \midrule
        \texttt{LR} & 236 (132--408) & 215 (105--396) & 15 (11--21) \\
        \texttt{LR-T} & 174 (140--205) & 39 (18--105) & 98 (79--136) \\
        \texttt{SL} & 152 (103--215) & 109 (72--179) & 37 (29--47) \\
        \texttt{SL-T} & 187 (162--208) & 32 (15--64) & 147 (124--165) \\
        \texttt{R-LOR} & 587 (379--844) & 39 (24--60) & 564 (335--772) \\
        \texttt{DR-P} & 238 (204--284) & 25 (12--57) & 206 (176--236) \\
        \texttt{DR-LOR} & 239 (179--306) & 29 (16--69) & 174 (130--258) \\
        \bottomrule
    \end{tabular}
    \caption{COR estimation: Interaction Order 1, Sample Size 2000, HTE High}
    \label{tab:cor_1_2000_hte_h}
\end{table}

\begin{table}[H]
    \centering
    \small
    \begin{tabular}{l *{3}{p{3cm}}}
        \toprule
        Estimator & iMSE (IQR) & iBias$^2$ (IQR) & iVar (IQR) \\
        \midrule
        \texttt{LR} & 24 (17--34) & 3 (1--7) & 18 (13--27) \\
        \texttt{LR-T} & 128 (98--163) & 6 (3--13) & 111 (83--153) \\
        \texttt{SL} & 27 (20--34) & 3 (1--6) & 24 (18--28) \\
        \texttt{SL-T} & 157 (128--184) & 5 (3--11) & 149 (124--174) \\
        \texttt{R-LOR} & 611 (303--883) & 11 (5--18) & 599 (297--871) \\
        \texttt{DR-P} & 213 (171--250) & 4 (3--6) & 207 (165--243) \\
        \texttt{DR-LOR} & 210 (131--297) & 4 (2--5) & 207 (128--290) \\
        \bottomrule
    \end{tabular}
    \caption{COR estimation: Interaction Order 1, Sample Size 2000, HTE Low}
    \label{tab:cor_1_2000_hte_l}
\end{table}

\begin{table}[H]
    \centering
    \small
    \begin{tabular}{l *{3}{p{3cm}}}
        \toprule
        Estimator & iMSE (IQR) & iBias$^2$ (IQR) & iVar (IQR) \\
        \midrule
        \texttt{LR} & 736 (505--927) & 621 (363--839) & 115 (100--127) \\
        \texttt{LR-T} & 1306 (1142--1656) & 442 (256--714) & 896 (812--962) \\
        \texttt{SL} & 501 (379--773) & 392 (278--677) & 110 (87--141) \\
        \texttt{SL-T} & 943 (780--1302) & 388 (252--550) & 565 (467--700) \\
        \texttt{R-LOR} & 2405 (2060--2872) & 211 (133--324) & 2181 (1771--2627) \\
        \texttt{DR-P} & 1886 (1719--2095) & 222 (146--335) & 1678 (1500--1808) \\
        \texttt{DR-LOR} & 1708 (1521--2195) & 223 (147--373) & 1502 (1231--1811) \\
        \bottomrule
    \end{tabular}
    \caption{COR estimation: Interaction Order 2, Sample Size 200, HTE High}
    \label{tab:cor_2_200_hte_h}
\end{table}

\begin{table}[H]
    \centering
    \small
    \begin{tabular}{l *{3}{p{3cm}}}
        \toprule
        Estimator & iMSE (IQR) & iBias$^2$ (IQR) & iVar (IQR) \\
        \midrule
        \texttt{LR} & 180 (135--298) & 51 (10--172) & 118 (104--136) \\
        \texttt{LR-T} & 1033 (918--1213) & 129 (43--284) & 893 (821--951) \\
        \texttt{SL} & 121 (95--169) & 22 (8--54) & 92 (77--124) \\
        \texttt{SL-T} & 722 (564--1004) & 104 (48--225) & 575 (501--813) \\
        \texttt{R-LOR} & 2136 (1724--2681) & 37 (26--58) & 2078 (1682--2622) \\
        \texttt{DR-P} & 1702 (1533--1886) & 41 (27--81) & 1647 (1508--1792) \\
        \texttt{DR-LOR} & 1461 (1210--1887) & 28 (20--45) & 1420 (1178--1852) \\
        \bottomrule
    \end{tabular}
    \caption{COR estimation: Interaction Order 2, Sample Size 200, HTE Low}
    \label{tab:cor_2_200_hte_l}
\end{table}

\begin{table}[H]
    \centering
    \small
    \begin{tabular}{l *{3}{p{3cm}}}
        \toprule
        Estimator & iMSE (IQR) & iBias$^2$ (IQR) & iVar (IQR) \\
        \midrule
        \texttt{LR} & 650 (431--861) & 606 (379--821) & 43 (39--50) \\
        \texttt{LR-T} & 752 (573--1005) & 421 (247--703) & 326 (301--349) \\
        \texttt{SL} & 433 (308--689) & 356 (236--634) & 68 (56--87) \\
        \texttt{SL-T} & 656 (506--806) & 249 (169--367) & 366 (304--465) \\
        \texttt{R-LOR} & 1140 (955--1551) & 163 (107--254) & 984 (759--1277) \\
        \texttt{DR-P} & 895 (800--1031) & 155 (107--234) & 738 (638--844) \\
        \texttt{DR-LOR} & 767 (661--956) & 200 (131--314) & 541 (456--720) \\
        \bottomrule
    \end{tabular}
    \caption{COR estimation: Interaction Order 2, Sample Size 500, HTE High}
    \label{tab:cor_2_500_hte_h}
\end{table}

\begin{table}[H]
    \centering
    \small
    \begin{tabular}{l *{3}{p{3cm}}}
        \toprule
        Estimator & iMSE (IQR) & iBias$^2$ (IQR) & iVar (IQR) \\
        \midrule
        \texttt{LR} & 98 (57--224) & 51 (11--184) & 43 (38--50) \\
        \texttt{LR-T} & 438 (357--602) & 111 (37--282) & 318 (291--352) \\
        \texttt{SL} & 68 (54--95) & 15 (5--31) & 52 (44--61) \\
        \texttt{SL-T} & 443 (368--602) & 56 (27--110) & 366 (316--493) \\
        \texttt{R-LOR} & 884 (692--1307) & 18 (12--24) & 858 (668--1278) \\
        \texttt{DR-P} & 732 (638--887) & 21 (12--37) & 717 (614--841) \\
        \texttt{DR-LOR} & 524 (425--684) & 13 (8--21) & 507 (413--675) \\
        \bottomrule
    \end{tabular}
    \caption{COR estimation: Interaction Order 2, Sample Size 500, HTE Low}
    \label{tab:cor_2_500_hte_l}
\end{table}

\begin{table}[H]
    \centering
    \small
    \begin{tabular}{l *{3}{p{3cm}}}
        \toprule
        Estimator & iMSE (IQR) & iBias$^2$ (IQR) & iVar (IQR) \\
        \midrule
        \texttt{LR} & 628 (401--845) & 606 (379--826) & 21 (19--24) \\
        \texttt{LR-T} & 562 (407--826) & 406 (256--684) & 155 (142--170) \\
        \texttt{SL} & 358 (258--620) & 308 (218--563) & 47 (39--59) \\
        \texttt{SL-T} & 437 (364--502) & 182 (124--253) & 237 (197--293) \\
        \texttt{R-LOR} & 689 (578--968) & 143 (102--216) & 538 (423--734) \\
        \texttt{DR-P} & 543 (476--619) & 111 (78--157) & 421 (354--496) \\
        \texttt{DR-LOR} & 455 (389--566) & 179 (119--279) & 259 (234--352) \\
        \bottomrule
    \end{tabular}
    \caption{COR estimation: Interaction Order 2, Sample Size 1000, HTE High}
    \label{tab:cor_2_1000_hte_h}
\end{table}

\begin{table}[H]
    \centering
    \small
    \begin{tabular}{l *{3}{p{3cm}}}
        \toprule
        Estimator & iMSE (IQR) & iBias$^2$ (IQR) & iVar (IQR) \\
        \midrule
        \texttt{LR} & 74 (35--216) & 52 (12--194) & 21 (18--23) \\
        \texttt{LR-T} & 264 (198--425) & 104 (34--275) & 152 (140--173) \\
        \texttt{SL} & 41 (33--57) & 9 (3--21) & 30 (27--36) \\
        \texttt{SL-T} & 292 (234--369) & 35 (19--78) & 250 (200--308) \\
        \texttt{R-LOR} & 474 (355--713) & 11 (7--15) & 466 (344--700) \\
        \texttt{DR-P} & 415 (347--524) & 14 (8--22) & 401 (341--499) \\
        \texttt{DR-LOR} & 236 (197--313) & 7 (5--11) & 232 (182--291) \\
        \bottomrule
    \end{tabular}
    \caption{COR estimation: Interaction Order 2, Sample Size 1000, HTE Low}
    \label{tab:cor_2_1000_hte_l}
\end{table}

\begin{table}[H]
    \centering
    \small
    \begin{tabular}{l *{3}{p{3cm}}}
        \toprule
        Estimator & iMSE (IQR) & iBias$^2$ (IQR) & iVar (IQR) \\
        \midrule
        \texttt{LR} & 618 (396--852) & 608 (384--843) & 11 (10--12) \\
        \texttt{LR-T} & 485 (321--725) & 411 (247--658) & 75 (69--84) \\
        \texttt{SL} & 309 (222--495) & 272 (192--438) & 36 (29--46) \\
        \texttt{SL-T} & 285 (249--331) & 131 (97--177) & 151 (129--171) \\
        \texttt{R-LOR} & 463 (383--584) & 123 (81--161) & 325 (243--503) \\
        \texttt{DR-P} & 339 (302--382) & 78 (59--122) & 254 (208--300) \\
        \texttt{DR-LOR} & 297 (253--378) & 155 (105--224) & 144 (129--187) \\
        \bottomrule
    \end{tabular}
    \caption{COR estimation: Interaction Order 2, Sample Size 2000, HTE High}
    \label{tab:cor_2_2000_hte_h}
\end{table}

\begin{table}[H]
    \centering
    \small
    \begin{tabular}{l *{3}{p{3cm}}}
        \toprule
        Estimator & iMSE (IQR) & iBias$^2$ (IQR) & iVar (IQR) \\
        \midrule
        \texttt{LR} & 60 (23--188) & 49 (9--179) & 11 (10--12) \\
        \texttt{LR-T} & 184 (119--341) & 96 (29--257) & 74 (68--85) \\
        \texttt{SL} & 27 (22--37) & 4 (2--12) & 21 (17--25) \\
        \texttt{SL-T} & 189 (148--223) & 25 (13--49) & 155 (129--186) \\
        \texttt{R-LOR} & 283 (201--450) & 8 (5--12) & 277 (192--438) \\
        \texttt{DR-P} & 261 (212--312) & 8 (5--12) & 254 (206--294) \\
        \texttt{DR-LOR} & 115 (97--149) & 3 (2--5) & 112 (92--148) \\
        \bottomrule
    \end{tabular}
    \caption{COR estimation: Interaction Order 2, Sample Size 2000, HTE Low}
    \label{tab:cor_2_2000_hte_l}
\end{table}

\begin{table}[H]
    \centering
    \small
    \begin{tabular}{l *{3}{p{3cm}}}
        \toprule
        Estimator & iMSE (IQR) & iBias$^2$ (IQR) & iVar (IQR) \\
        \midrule
        \texttt{LR} & 884 (499--1216) & 773 (401--1094) & 111 (94--126) \\
        \texttt{LR-T} & 1446 (1186--1892) & 518 (285--989) & 881 (822--978) \\
        \texttt{SL} & 691 (442--959) & 550 (328--844) & 107 (90--148) \\
        \texttt{SL-T} & 1060 (831--1446) & 485 (277--805) & 555 (487--689) \\
        \texttt{R-LOR} & 2395 (2097--2970) & 333 (173--525) & 2044 (1804--2378) \\
        \texttt{DR-P} & 2011 (1766--2373) & 356 (194--647) & 1666 (1514--1799) \\
        \texttt{DR-LOR} & 1801 (1545--2199) & 353 (194--553) & 1391 (1246--1712) \\
        \bottomrule
    \end{tabular}
    \caption{COR estimation: Interaction Order 3, Sample Size 200, HTE High}
    \label{tab:cor_3_200_hte_h}
\end{table}

\begin{table}[H]
    \centering
    \small
    \begin{tabular}{l *{3}{p{3cm}}}
        \toprule
        Estimator & iMSE (IQR) & iBias$^2$ (IQR) & iVar (IQR) \\
        \midrule
        \texttt{LR} & 202 (140--379) & 82 (11--255) & 116 (98--132) \\
        \texttt{LR-T} & 1118 (945--1324) & 175 (54--378) & 903 (819--976) \\
        \texttt{SL} & 147 (100--223) & 38 (7--104) & 98 (78--125) \\
        \texttt{SL-T} & 724 (561--976) & 164 (39--285) & 557 (480--689) \\
        \texttt{R-LOR} & 2130 (1733--2601) & 69 (33--129) & 2021 (1639--2432) \\
        \texttt{DR-P} & 1743 (1597--1951) & 87 (34--192) & 1623 (1536--1718) \\
        \texttt{DR-LOR} & 1455 (1257--1777) & 46 (22--112) & 1356 (1172--1652) \\
        \bottomrule
    \end{tabular}
    \caption{COR estimation: Interaction Order 3, Sample Size 200, HTE Low}
    \label{tab:cor_3_200_hte_l}
\end{table}

\begin{table}[H]
    \centering
    \small
    \begin{tabular}{l *{3}{p{3cm}}}
        \toprule
        Estimator & iMSE (IQR) & iBias$^2$ (IQR) & iVar (IQR) \\
        \midrule
        \texttt{LR} & 784 (434--1119) & 742 (395--1072) & 44 (40--50) \\
        \texttt{LR-T} & 821 (599--1260) & 498 (268--938) & 321 (300--357) \\
        \texttt{SL} & 622 (376--852) & 490 (299--789) & 79 (65--119) \\
        \texttt{SL-T} & 750 (576--1102) & 407 (220--646) & 365 (309--509) \\
        \texttt{R-LOR} & 1256 (1056--1646) & 260 (148--432) & 964 (769--1282) \\
        \texttt{DR-P} & 1035 (868--1342) & 275 (147--420) & 783 (661--913) \\
        \texttt{DR-LOR} & 875 (728--1126) & 303 (180--491) & 564 (474--726) \\
        \bottomrule
    \end{tabular}
    \caption{COR estimation: Interaction Order 3, Sample Size 500, HTE High}
    \label{tab:cor_3_500_hte_h}
\end{table}

\begin{table}[H]
    \centering
    \small
    \begin{tabular}{l *{3}{p{3cm}}}
        \toprule
        Estimator & iMSE (IQR) & iBias$^2$ (IQR) & iVar (IQR) \\
        \midrule
        \texttt{LR} & 133 (58--296) & 79 (9--255) & 42 (38--50) \\
        \texttt{LR-T} & 500 (366--664) & 149 (39--338) & 321 (289--352) \\
        \texttt{SL} & 100 (61--164) & 32 (6--85) & 63 (48--90) \\
        \texttt{SL-T} & 504 (353--701) & 102 (30--230) & 366 (304--474) \\
        \texttt{R-LOR} & 978 (733--1336) & 35 (15--82) & 894 (671--1242) \\
        \texttt{DR-P} & 830 (695--1002) & 52 (17--120) & 741 (661--903) \\
        \texttt{DR-LOR} & 563 (447--824) & 23 (9--60) & 504 (409--702) \\
        \bottomrule
    \end{tabular}
    \caption{COR estimation: Interaction Order 3, Sample Size 500, HTE Low}
    \label{tab:cor_3_500_hte_l}
\end{table}

\begin{table}[H]
    \centering
    \small
    \begin{tabular}{l *{3}{p{3cm}}}
        \toprule
        Estimator & iMSE (IQR) & iBias$^2$ (IQR) & iVar (IQR) \\
        \midrule
        \texttt{LR} & 743 (410--1074) & 726 (389--1054) & 21 (18--23) \\
        \texttt{LR-T} & 647 (417--1075) & 490 (265--921) & 154 (143--170) \\
        \texttt{SL} & 511 (323--778) & 428 (261--710) & 60 (47--94) \\
        \texttt{SL-T} & 576 (423--835) & 329 (182--512) & 253 (206--357) \\
        \texttt{R-LOR} & 807 (642--1125) & 218 (127--363) & 563 (428--824) \\
        \texttt{DR-P} & 677 (526--905) & 201 (124--303) & 467 (364--580) \\
        \texttt{DR-LOR} & 574 (433--754) & 266 (160--435) & 286 (246--371) \\
        \bottomrule
    \end{tabular}
    \caption{COR estimation: Interaction Order 3, Sample Size 1000, HTE High}
    \label{tab:cor_3_1000_hte_h}
\end{table}

\begin{table}[H]
    \centering
    \small
    \begin{tabular}{l *{3}{p{3cm}}}
        \toprule
        Estimator & iMSE (IQR) & iBias$^2$ (IQR) & iVar (IQR) \\
        \midrule
        \texttt{LR} & 96 (33--272) & 72 (10--252) & 21 (18--23) \\
        \texttt{LR-T} & 315 (201--486) & 145 (40--335) & 154 (139--170) \\
        \texttt{SL} & 66 (40--123) & 20 (5--59) & 43 (34--60) \\
        \texttt{SL-T} & 345 (243--519) & 81 (26--206) & 251 (203--329) \\
        \texttt{R-LOR} & 531 (379--763) & 20 (12--47) & 494 (345--710) \\
        \texttt{DR-P} & 482 (403--687) & 34 (13--93) & 453 (381--584) \\
        \texttt{DR-LOR} & 281 (209--393) & 16 (7--39) & 254 (188--340) \\
        \bottomrule
    \end{tabular}
    \caption{COR estimation: Interaction Order 3, Sample Size 1000, HTE Low}
    \label{tab:cor_3_1000_hte_l}
\end{table}

\begin{table}[H]
    \centering
    \small
    \begin{tabular}{l *{3}{p{3cm}}}
        \toprule
        Estimator & iMSE (IQR) & iBias$^2$ (IQR) & iVar (IQR) \\
        \midrule
        \texttt{LR} & 744 (403--1091) & 735 (393--1081) & 11 (10--12) \\
        \texttt{LR-T} & 567 (329--1004) & 485 (260--928) & 74 (69--84) \\
        \texttt{SL} & 383 (260--675) & 329 (228--630) & 44 (33--58) \\
        \texttt{SL-T} & 430 (301--610) & 250 (139--400) & 166 (135--222) \\
        \texttt{R-LOR} & 558 (456--830) & 182 (117--287) & 371 (270--587) \\
        \texttt{DR-P} & 451 (342--577) & 140 (87--206) & 307 (226--390) \\
        \texttt{DR-LOR} & 393 (300--549) & 216 (145--357) & 162 (136--205) \\
        \bottomrule
    \end{tabular}
    \caption{COR estimation: Interaction Order 3, Sample Size 2000, HTE High}
    \label{tab:cor_3_2000_hte_h}
\end{table}

\begin{table}[H]
    \centering
    \small
    \begin{tabular}{l *{3}{p{3cm}}}
        \toprule
        Estimator & iMSE (IQR) & iBias$^2$ (IQR) & iVar (IQR) \\
        \midrule
        \texttt{LR} & 85 (23--246) & 75 (11--238) & 10 (9--12) \\
        \texttt{LR-T} & 230 (120--414) & 150 (38--340) & 74 (69--81) \\
        \texttt{SL} & 46 (30--81) & 13 (5--40) & 30 (24--40) \\
        \texttt{SL-T} & 232 (169--351) & 61 (23--141) & 171 (140--217) \\
        \texttt{R-LOR} & 302 (227--498) & 12 (8--25) & 288 (218--479) \\
        \texttt{DR-P} & 314 (253--443) & 25 (10--51) & 298 (238--383) \\
        \texttt{DR-LOR} & 139 (110--192) & 7 (4--17) & 121 (103--167) \\
        \bottomrule
    \end{tabular}
    \caption{COR estimation: Interaction Order 3, Sample Size 2000, HTE Low}
    \label{tab:cor_3_2000_hte_l}
\end{table}


\begin{table}[H]
    \centering
    \small
    \begin{tabular}{l *{3}{p{3cm}}}
        \toprule
        Estimator & iMSE (IQR) & iBias$^2$ (IQR) & iVar (IQR) \\
        \midrule
        \texttt{LR} & 148 (118--218) & 82 (45--152) & 64 (50--76) \\
        \texttt{LR-T} & 409 (374--458) & 23 (15--47) & 386 (344--419) \\
        \texttt{SL} & 118 (87--185) & 81 (45--131) & 45 (36--57) \\
        \texttt{SL-T} & 312 (255--358) & 41 (28--66) & 255 (209--309) \\
        \texttt{R-LRR} & 1026 (819--1163) & 60 (35--113) & 961 (755--1095) \\
        \texttt{DR-P} & 603 (519--698) & 28 (17--41) & 564 (489--662) \\
        \texttt{DR-LRR} & 752 (605--870) & 35 (22--53) & 718 (585--827) \\
        \bottomrule
    \end{tabular}
    \caption{CRR estimation: Interaction Order 1, Sample Size 200, HTE High}
    \label{tab:crr_1_200_hte_h}
\end{table}

\begin{table}[H]
    \centering
    \small
    \begin{tabular}{l *{3}{p{3cm}}}
        \toprule
        Estimator & iMSE (IQR) & iBias$^2$ (IQR) & iVar (IQR) \\
        \midrule
        \texttt{LR} & 74 (56--87) & 2 (1--6) & 70 (52--81) \\
        \texttt{LR-T} & 381 (327--427) & 8 (5--13) & 371 (319--408) \\
        \texttt{SL} & 51 (38--65) & 5 (2--13) & 45 (32--53) \\
        \texttt{SL-T} & 297 (229--344) & 13 (7--25) & 280 (206--324) \\
        \texttt{R-LRR} & 1056 (797--1253) & 25 (12--90) & 1010 (773--1129) \\
        \texttt{DR-P} & 517 (416--628) & 9 (7--13) & 504 (409--617) \\
        \texttt{DR-LRR} & 785 (583--924) & 10 (6--16) & 776 (573--904) \\
        \bottomrule
    \end{tabular}
    \caption{CRR estimation: Interaction Order 1, Sample Size 200, HTE Low}
    \label{tab:crr_1_200_hte_l}
\end{table}

\begin{table}[H]
    \centering
    \small
    \begin{tabular}{l *{3}{p{3cm}}}
        \toprule
        Estimator & iMSE (IQR) & iBias$^2$ (IQR) & iVar (IQR) \\
        \midrule
        \texttt{LR} & 107 (77--172) & 85 (43--150) & 26 (18--34) \\
        \texttt{LR-T} & 197 (170--223) & 16 (10--41) & 167 (135--194) \\
        \texttt{SL} & 92 (67--142) & 58 (36--107) & 29 (25--40) \\
        \texttt{SL-T} & 179 (148--205) & 24 (15--39) & 145 (118--171) \\
        \texttt{R-LRR} & 660 (448--836) & 34 (24--62) & 622 (410--770) \\
        \texttt{DR-P} & 294 (247--341) & 16 (10--27) & 271 (237--312) \\
        \texttt{DR-LRR} & 384 (278--479) & 24 (15--36) & 349 (250--463) \\
        \bottomrule
    \end{tabular}
    \caption{CRR estimation: Interaction Order 1, Sample Size 500, HTE High}
    \label{tab:crr_1_500_hte_h}
\end{table}

\begin{table}[H]
    \centering
    \small
    \begin{tabular}{l *{3}{p{3cm}}}
        \toprule
        Estimator & iMSE (IQR) & iBias$^2$ (IQR) & iVar (IQR) \\
        \midrule
        \texttt{LR} & 33 (23--43) & 2 (0--5) & 30 (20--39) \\
        \texttt{LR-T} & 183 (141--209) & 5 (2--8) & 175 (133--199) \\
        \texttt{SL} & 30 (22--39) & 3 (1--7) & 26 (20--34) \\
        \texttt{SL-T} & 165 (127--192) & 6 (3--11) & 157 (123--184) \\
        \texttt{R-LRR} & 694 (451--912) & 12 (6--50) & 676 (440--826) \\
        \texttt{DR-P} & 258 (215--310) & 4 (3--7) & 250 (212--302) \\
        \texttt{DR-LRR} & 407 (270--532) & 5 (3--8) & 403 (265--523) \\
        \bottomrule
    \end{tabular}
    \caption{CRR estimation: Interaction Order 1, Sample Size 500, HTE Low}
    \label{tab:crr_1_500_hte_l}
\end{table}

\begin{table}[H]
    \centering
    \small
    \begin{tabular}{l *{3}{p{3cm}}}
        \toprule
        Estimator & iMSE (IQR) & iBias$^2$ (IQR) & iVar (IQR) \\
        \midrule
        \texttt{LR} & 99 (62--161) & 84 (42--151) & 13 (9--17) \\
        \texttt{LR-T} & 111 (92--133) & 16 (8--40) & 85 (66--102) \\
        \texttt{SL} & 73 (52--112) & 51 (31--82) & 21 (16--27) \\
        \texttt{SL-T} & 115 (96--135) & 16 (10--26) & 92 (78--110) \\
        \texttt{R-LRR} & 387 (263--534) & 27 (17--43) & 373 (240--504) \\
        \texttt{DR-P} & 167 (144--195) & 13 (8--23) & 150 (134--181) \\
        \texttt{DR-LRR} & 194 (145--262) & 22 (13--34) & 162 (125--233) \\
        \bottomrule
    \end{tabular}
    \caption{CRR estimation: Interaction Order 1, Sample Size 1000, HTE High}
    \label{tab:crr_1_1000_hte_h}
\end{table}

\begin{table}[H]
    \centering
    \small
    \begin{tabular}{l *{3}{p{3cm}}}
        \toprule
        Estimator & iMSE (IQR) & iBias$^2$ (IQR) & iVar (IQR) \\
        \midrule
        \texttt{LR} & 17 (11--24) & 1 (0--4) & 14 (9--21) \\
        \texttt{LR-T} & 100 (71--117) & 4 (2--6) & 90 (66--109) \\
        \texttt{SL} & 18 (13--24) & 2 (0--4) & 16 (11--19) \\
        \texttt{SL-T} & 104 (76--121) & 3 (2--6) & 98 (75--117) \\
        \texttt{R-LRR} & 411 (237--574) & 8 (4--22) & 405 (233--558) \\
        \texttt{DR-P} & 146 (118--184) & 3 (2--5) & 144 (116--177) \\
        \texttt{DR-LRR} & 203 (121--283) & 3 (2--4) & 199 (120--279) \\
        \bottomrule
    \end{tabular}
    \caption{CRR estimation: Interaction Order 1, Sample Size 1000, HTE Low}
    \label{tab:crr_1_1000_hte_l}
\end{table}

\begin{table}[H]
    \centering
    \small
    \begin{tabular}{l *{3}{p{3cm}}}
        \toprule
        Estimator & iMSE (IQR) & iBias$^2$ (IQR) & iVar (IQR) \\
        \midrule
        \texttt{LR} & 92 (55--159) & 85 (43--152) & 6 (4--9) \\
        \texttt{LR-T} & 67 (54--85) & 15 (7--41) & 42 (31--55) \\
        \texttt{SL} & 56 (42--80) & 42 (29--65) & 15 (12--19) \\
        \texttt{SL-T} & 73 (62--84) & 12 (6--22) & 57 (49--68) \\
        \texttt{R-LRR} & 236 (154--341) & 24 (12--33) & 210 (133--312) \\
        \texttt{DR-P} & 95 (82--115) & 10 (6--19) & 85 (72--101) \\
        \texttt{DR-LRR} & 113 (82--133) & 20 (12--31) & 85 (58--112) \\
        \bottomrule
    \end{tabular}
    \caption{CRR estimation: Interaction Order 1, Sample Size 2000, HTE High}
    \label{tab:crr_1_2000_hte_h}
\end{table}

\begin{table}[H]
    \centering
    \small
    \begin{tabular}{l *{3}{p{3cm}}}
        \toprule
        Estimator & iMSE (IQR) & iBias$^2$ (IQR) & iVar (IQR) \\
        \midrule
        \texttt{LR} & 10 (7--15) & 1 (0--4) & 7 (5--12) \\
        \texttt{LR-T} & 53 (38--69) & 3 (1--6) & 45 (31--62) \\
        \texttt{SL} & 11 (8--15) & 1 (0--4) & 10 (7--12) \\
        \texttt{SL-T} & 64 (46--74) & 2 (1--4) & 61 (44--71) \\
        \texttt{R-LRR} & 228 (128--363) & 5 (2--10) & 220 (127--352) \\
        \texttt{DR-P} & 81 (65--103) & 2 (1--3) & 78 (63--100) \\
        \texttt{DR-LRR} & 94 (56--142) & 1 (1--3) & 92 (55--137) \\
        \bottomrule
    \end{tabular}
    \caption{CRR estimation: Interaction Order 1, Sample Size 2000, HTE Low}
    \label{tab:crr_1_2000_hte_l}
\end{table}

\begin{table}[H]
    \centering
    \small
    \begin{tabular}{l *{3}{p{3cm}}}
        \toprule
        Estimator & iMSE (IQR) & iBias$^2$ (IQR) & iVar (IQR) \\
        \midrule
        \texttt{LR} & 227 (171--321) & 186 (123--280) & 40 (32--51) \\
        \texttt{LR-T} & 462 (400--565) & 151 (91--238) & 302 (253--377) \\
        \texttt{SL} & 177 (113--260) & 139 (80--219) & 37 (26--50) \\
        \texttt{SL-T} & 316 (259--412) & 115 (83--183) & 193 (136--257) \\
        \texttt{R-LRR} & 789 (622--958) & 74 (50--117) & 714 (535--871) \\
        \texttt{DR-P} & 664 (591--761) & 73 (52--120) & 587 (506--661) \\
        \texttt{DR-LRR} & 611 (479--760) & 76 (56--130) & 534 (379--665) \\
        \bottomrule
    \end{tabular}
    \caption{CRR estimation: Interaction Order 2, Sample Size 200, HTE High}
    \label{tab:crr_2_200_hte_h}
\end{table}

\begin{table}[H]
    \centering
    \small
    \begin{tabular}{l *{3}{p{3cm}}}
        \toprule
        Estimator & iMSE (IQR) & iBias$^2$ (IQR) & iVar (IQR) \\
        \midrule
        \texttt{LR} & 70 (50--102) & 20 (5--59) & 42 (36--52) \\
        \texttt{LR-T} & 378 (311--435) & 39 (17--87) & 314 (266--363) \\
        \texttt{SL} & 45 (31--60) & 9 (3--20) & 33 (25--42) \\
        \texttt{SL-T} & 234 (182--327) & 28 (14--71) & 193 (154--266) \\
        \texttt{R-LRR} & 689 (581--857) & 13 (9--21) & 655 (564--844) \\
        \texttt{DR-P} & 592 (515--687) & 15 (10--23) & 567 (500--636) \\
        \texttt{DR-LRR} & 508 (418--678) & 11 (7--17) & 490 (408--665) \\
        \bottomrule
    \end{tabular}
    \caption{CRR estimation: Interaction Order 2, Sample Size 200, HTE Low}
    \label{tab:crr_2_200_hte_l}
\end{table}

\begin{table}[H]
    \centering
    \small
    \begin{tabular}{l *{3}{p{3cm}}}
        \toprule
        Estimator & iMSE (IQR) & iBias$^2$ (IQR) & iVar (IQR) \\
        \midrule
        \texttt{LR} & 207 (145--290) & 187 (123--278) & 15 (12--20) \\
        \texttt{LR-T} & 268 (208--341) & 142 (88--233) & 108 (88--140) \\
        \texttt{SL} & 154 (96--214) & 125 (75--191) & 24 (18--31) \\
        \texttt{SL-T} & 222 (175--273) & 80 (58--120) & 127 (97--175) \\
        \texttt{R-LRR} & 432 (321--572) & 59 (40--93) & 363 (242--470) \\
        \texttt{DR-P} & 322 (275--389) & 57 (36--74) & 270 (223--322) \\
        \texttt{DR-LRR} & 287 (224--383) & 65 (44--112) & 218 (147--304) \\
        \bottomrule
    \end{tabular}
    \caption{CRR estimation: Interaction Order 2, Sample Size 500, HTE High}
    \label{tab:crr_2_500_hte_h}
\end{table}

\begin{table}[H]
    \centering
    \small
    \begin{tabular}{l *{3}{p{3cm}}}
        \toprule
        Estimator & iMSE (IQR) & iBias$^2$ (IQR) & iVar (IQR) \\
        \midrule
        \texttt{LR} & 36 (25--71) & 18 (5--58) & 15 (12--20) \\
        \texttt{LR-T} & 157 (130--204) & 33 (11--83) & 112 (92--133) \\
        \texttt{SL} & 27 (20--35) & 5 (2--12) & 19 (15--24) \\
        \texttt{SL-T} & 159 (118--203) & 16 (9--41) & 131 (98--177) \\
        \texttt{R-LRR} & 346 (253--500) & 7 (5--11) & 339 (246--489) \\
        \texttt{DR-P} & 279 (214--327) & 8 (5--13) & 269 (208--312) \\
        \texttt{DR-LRR} & 203 (162--310) & 5 (3--8) & 198 (149--302) \\
        \bottomrule
    \end{tabular}
    \caption{CRR estimation: Interaction Order 2, Sample Size 500, HTE Low}
    \label{tab:crr_2_500_hte_l}
\end{table}

\begin{table}[H]
    \centering
    \small
    \begin{tabular}{l *{3}{p{3cm}}}
        \toprule
        Estimator & iMSE (IQR) & iBias$^2$ (IQR) & iVar (IQR) \\
        \midrule
        \texttt{LR} & 196 (135--284) & 188 (123--278) & 7 (6--9) \\
        \texttt{LR-T} & 201 (144--274) & 142 (89--226) & 53 (42--70) \\
        \texttt{SL} & 133 (84--193) & 115 (69--171) & 16 (13--23) \\
        \texttt{SL-T} & 153 (126--181) & 62 (44--90) & 80 (63--118) \\
        \texttt{R-LRR} & 261 (204--355) & 55 (39--84) & 207 (141--290) \\
        \texttt{DR-P} & 197 (164--239) & 39 (27--59) & 154 (125--194) \\
        \texttt{DR-LRR} & 177 (145--214) & 61 (41--98) & 108 (77--153) \\
        \bottomrule
    \end{tabular}
    \caption{CRR estimation: Interaction Order 2, Sample Size 1000, HTE High}
    \label{tab:crr_2_1000_hte_h}
\end{table}

\begin{table}[H]
    \centering
    \small
    \begin{tabular}{l *{3}{p{3cm}}}
        \toprule
        Estimator & iMSE (IQR) & iBias$^2$ (IQR) & iVar (IQR) \\
        \midrule
        \texttt{LR} & 27 (14--70) & 17 (5--63) & 7 (6--9) \\
        \texttt{LR-T} & 93 (70--134) & 32 (11--83) & 53 (44--64) \\
        \texttt{SL} & 16 (12--22) & 3 (1--9) & 12 (9--14) \\
        \texttt{SL-T} & 108 (81--135) & 13 (7--25) & 91 (68--111) \\
        \texttt{R-LRR} & 200 (146--280) & 5 (3--8) & 196 (137--272) \\
        \texttt{DR-P} & 163 (126--195) & 5 (3--8) & 160 (122--187) \\
        \texttt{DR-LRR} & 100 (80--149) & 3 (2--5) & 94 (75--143) \\
        \bottomrule
    \end{tabular}
    \caption{CRR estimation: Interaction Order 2, Sample Size 1000, HTE Low}
    \label{tab:crr_2_1000_hte_l}
\end{table}

\begin{table}[H]
    \centering
    \small
    \begin{tabular}{l *{3}{p{3cm}}}
        \toprule
        Estimator & iMSE (IQR) & iBias$^2$ (IQR) & iVar (IQR) \\
        \midrule
        \texttt{LR} & 192 (129--279) & 188 (126--274) & 4 (3--5) \\
        \texttt{LR-T} & 168 (112--244) & 139 (89--223) & 25 (20--35) \\
        \texttt{SL} & 111 (71--157) & 99 (57--142) & 13 (9--19) \\
        \texttt{SL-T} & 100 (88--122) & 47 (32--66) & 52 (41--70) \\
        \texttt{R-LRR} & 176 (132--239) & 49 (34--71) & 129 (86--185) \\
        \texttt{DR-P} & 123 (106--154) & 29 (19--43) & 91 (75--121) \\
        \texttt{DR-LRR} & 119 (98--145) & 57 (38--87) & 56 (43--84) \\
        \bottomrule
    \end{tabular}%
    \caption{CRR estimation: Interaction Order 3, Sample Size 2000, HTE Low}
    \label{tab:crr_3_2000_hte_l}
\end{table}

\end{document}